\documentclass[onefignum,onetabnum]{siamonline190516}
\usepackage{lipsum}
\usepackage{amsfonts}
\usepackage{graphicx}
\usepackage{epstopdf}
\usepackage{algorithmic}
\ifpdf
  \DeclareGraphicsExtensions{.eps,.pdf,.png,.jpg}
\else
  \DeclareGraphicsExtensions{.eps}
\fi

\usepackage{enumitem}
\setlist[enumerate]{leftmargin=.5in}
\setlist[itemize]{leftmargin=.5in}

\newsiamremark{remark}{Remark}
\newsiamremark{hypothesis}{Hypothesis}
\crefname{hypothesis}{Hypothesis}{Hypotheses}
\newsiamthm{claim}{Claim}

\headers{Solving Inverse Problems by Joint Posterior Max with Autoencoding Prior}{M. González, A. Almansa 
 and P. Tan}

\title{Solving Inverse Problems \\ by Joint Posterior Maximization with Autoencoding Prior\thanks{
Submitted to the editors on March 24th 2021. Revised version February 3rd 2022.
The authors would like to sincerely thank Mauricio Delbracio, José Lezama and Pablo Musé for their help, their insightful comments, and their continuous support throughout this project.
\funding{This work was funded by ECOS Sud Project U17E04, by the French-Uruguayan Institute of Mathematics and Interactions (\href{http://ifumi.cmat.edu.uy/}{IFUMI}), by CSIC I+D 2018-256 (Uruguay) and by ANII (Uruguay) under Grant
11 FCE\_1\_2017\_1\_135458 and by the French Research Agency through the PostProdLEAP project (ANR-19-CE23-0027-01). Computer experiments for this work ran on a Titan Xp GPU donated by NVIDIA, as well as on HPC resources from GENCI-IDRIS (Grants 2020-AD011011641 and 2021-AD011011641R1). }}}

\author{Mario González\thanks{DMEL, CenUR RN, Universidad de la República, Salto, Uruguay
  (\email{mgonzalez@unorte.edu.uy}, \url{http://dmel.interior.edu.uy/mario-gonzalez/}).}
\and Andrés Almansa \thanks{MAP5, CNRS \& Université Paris Cité, France 
  (\email{andres.almansa@parisdescartes.fr}).}
\and Pauline Tan \thanks{LJLL, Sorbonne Université, Paris, France (\email{pauline.tan@sorbonne-universite.fr})}
}

\usepackage{amsopn}

\usepackage{amsmath,amsfonts,bm}

\def\1{\bm{1}}

\def\vmu{{\bm{\mu}}}

\def\vu{{\bm{u}}}

\def\vw{{\bm{w}}}
\def\vx{{\bm{x}}}
\def\vy{{\bm{y}}}
\def\vz{{\bm{z}}}

\def\mA{{\bm{A}}}

\def\mSigma{{\bm{\Sigma}}}

\DeclareMathAlphabet{\mathsfit}{\encodingdefault}{\sfdefault}{m}{sl}
\SetMathAlphabet{\mathsfit}{bold}{\encodingdefault}{\sfdefault}{bx}{n}

\newcommand{\R}{\mathbb{R}}

\DeclareMathOperator*{\argmax}{arg\,max}
\DeclareMathOperator*{\argmin}{arg\,min}

\usepackage{subfigure}

\usepackage{algorithm}  %

\theoremstyle{plain}\newtheorem{property}{Property}
\theoremstyle{plain}
\theoremstyle{plain}\newtheorem{assumption}{Assumption}

\usepackage{color}

\usepackage[numbers]{natbib}

\newcommand{\MAP}{\textsc{map}}
\newcommand{\x}{\ensuremath{\vx}} %
\newcommand{\X}{X} %
\newcommand{\y}{\ensuremath{\vy}} %
\newcommand{\Y}{Y} %
\newcommand{\z}{\ensuremath{\vz}} %

\newcommand{\Z}{Z} %

\newcommand{\decoder}{\ensuremath{\mathsf{G}}}
\newcommand{\generator}{\decoder}
\newcommand{\encoderParams}{\ensuremath{\phi}}
\newcommand{\decoderParams}{\ensuremath{\theta}}
\newcommand{\muEncoder}{\vmu_{\encoderParams}}
\newcommand{\muDecoder}{\vmu_{\decoderParams}}
\newcommand{\SigmaEncoder}{\mSigma_{\encoderParams}}
\newcommand{\SigmaDecoder}{\mSigma_{\decoderParams}}

\newcommand{\sigmaDVAE}{\sigma_{\operatorname{DVAE}}}
\newcommand{\Fdata}{\ensuremath{F}}
\newcommand{\Gprior}{\ensuremath{U}}
\newcommand{\Hcoupling}{\ensuremath{H}}
\newcommand{\Htheta}{\ensuremath{\Hcoupling_{\decoderParams}}}
\newcommand{\Kcoupling}{\ensuremath{K}}
\newcommand{\Kphi}{\ensuremath{\Kcoupling_{\encoderParams}}}

\newcommand{\pDecoder}{p_\decoderParams}
\newcommand{\ptheta}{\pDecoder}
\newcommand{\qEncoder}{q_\encoderParams}
\newcommand{\qphi}{\qEncoder}
\newcommand{\Normal}{\mathcal{N}} %

\newcommand{\xdim}{\ensuremath{d}} %
\newcommand{\ydim}{\ensuremath{m}}
\newcommand{\zdim}{\ensuremath{l}}

\newcommand{\NN}{\mathbb{N}} %
\newcommand{\RR}{\mathbb{R}} %
\newcommand{\pdf}{p} %

\newcommand{\PDF}[2]{\pdf_{#1} \left({#2}\right)} %

\newcommand{\ConditionalPDF}[4]{\pdf_{{#1} \vert {#2}} \left( {#3} \,\middle\vert\, {#4} \right)} %

\DeclareMathOperator*{\gd}{\textsc{gd}}
\newcommand{\GD}[3]{\gd_{#1} {#2},\,\text{starting from $#1 = #3$}}

\newcommand{\prox}{\operatorname{prox}}

\usepackage[normalem]{ulem} %
\usepackage{color}
\definecolor{myred}{rgb}{0.8, 0.0, 0.0}

\ifpdf
\hypersetup{
  pdftitle={Solving Inverse Problems by Joint Posterior Maximization with Autoencoding Prior},
  pdfauthor={M. González, A. Almansa, %
  and P. Tan}
}
\fi

\begin{document}

\maketitle

\begin{abstract}
In this work we address the problem of solving ill-posed inverse problems in imaging where the prior is a variational autoencoder (VAE).
Specifically we consider the decoupled case where the prior is trained once and can be reused for many different log-concave degradation models without retraining. Whereas previous MAP-based approaches to this problem lead to highly non-convex optimization algorithms, our approach computes the joint (space-latent) MAP that naturally leads to alternate optimization algorithms and to the use of a stochastic encoder to accelerate computations.
The resulting technique (JPMAP) performs Joint Posterior Maximization using an Autoencoding Prior.
We show theoretical and experimental evidence that the proposed objective function is quite close to bi-convex. Indeed it satisfies a weak bi-convexity property which is sufficient to guarantee that our optimization scheme converges to a stationary point.
We also highlight the importance of correctly training the VAE using a denoising criterion, in order to ensure that the encoder generalizes well to out-of-distribution images, without affecting the quality of the generative model. This simple modification is key to providing robustness to the whole procedure.
Finally we show how our joint MAP methodology relates to more common MAP approaches, and we propose a continuation scheme that makes use of our JPMAP algorithm to provide more robust MAP estimates.
Experimental results also show the higher quality of the solutions obtained by our JPMAP approach with respect to other non-convex MAP approaches which more often get stuck in spurious local optima.
\end{abstract}



\section{Introduction}
General inverse problems in imaging consist in estimating a clean image $\vx\in\R^\xdim$ from noisy, degraded measurements $\vy\in\R^\ydim$. In many cases the degradation model is known and its conditional density
$$ \ConditionalPDF{\Y}{\X}{\vy}{\vx} \propto e^{-\Fdata(\vx,\vy)} $$
is log-concave with respect to $\vx$. To illustrate this, let us consider the case where the negative log-conditional is quadratic with respect to $\vx$ 
\begin{equation} 
\label{eq:data-term}
\Fdata(\vx,\vy) = \frac{1}{2\sigma^2} \| \mA \vx - \vy \|^2. 
\end{equation}
This boils down to a linear degradation model that takes into account degradations such as, white Gaussian noise, blur, and missing pixels. When the degradation operator $\mA$ is non-invertible or ill-conditioned, or when the noise level $\sigma$ is high, obtaining a good estimate of $\vx$ requires prior knowledge on the image, given by $\PDF{X}{\vx} \propto e^{-\lambda \Gprior(\vx)}$. Variational and Bayesian methods in imaging are extensively used to derive MMSE or MAP estimators, 
\begin{equation} 
\label{eq:MAP}
\hat{\vx}_{\textsc{map}}\! =\! \argmax_\vx \ConditionalPDF{\X}{\Y}{\vx}{\vy} \! =\! \argmin_\vx \left \{\Fdata(\vx,\vy) + \lambda \Gprior(\vx)\right \}
\end{equation}
based on \emph{(a)} explicit hand-crafted priors like Tikhonov regularization \cite{Tikhonov1943}, total variation \cite{Rudin1992,Chambolle04,Louchet2013,Pereyra2016} and its higher order \cite{Bredies2010} and non-local \cite{Gilboa2008} variants, sparsity in a transformed domain \cite{Donoho1995}, or in redundant representations like wavelet frames or patch dictionaries \cite{Elad2010}, or \emph{(b)} learning-based priors like patch-based Gaussian mixture models \cite{Zoran2011,yu2011solving, Teodoro2018scene}.
\paragraph{Neural network regression}
Since deep neural networks (NN) showed their superiority in image classification tasks~\cite{Krizhevsky2012} researchers started to look for ways to use this tool to solve inverse problems too. The most straightforward attempts employed neural networks as \emph{regressors} to learn a risk minimizing mapping $\vy \mapsto \vx$ from many examples $(\vx_i,\vy_i)$ either agnostically~\cite{dong2014learning,zhang2017beyond,zhang2018ffdnet, gharbi2016deep,schwartz2018deepisp,gao2019dynamic} 
or including the degradation model in the network architecture via unrolled optimization techniques~\cite{gregor2010learning,Chen2017,diamond2017unrolled,gilton2019neumann}.
\paragraph{Implicitly decoupled priors}
The main drawback of neural networks regression is that they require to retrain the neural network each time a single parameter of the degradation model changes. To avoid the need for retraining, another family of approaches seek to \emph{decouple} the NN-based learned image prior from the degradation model.
A popular approach within this methodology are \emph{Plug \& Play} (or PnP) methods. Instead of directly learning the log-prior $-\log \PDF{X}{\vx} = \Gprior(\vx) + C$, these methods seek to learn an approximation of its gradient $\nabla\Gprior$ \cite{Bigdeli2017,Bigdeli2017a} or proximal operator $\prox_\Gprior$ \cite{venkatakrishnan2013plug,meinhardt2017learning,Zhang2017,chan2017plug,kamilov2017plug,ryu2019plug}, by replacing it by a denoising NN. Then, these approximations are used in an iterative optimization algorithm to find the corresponding MAP estimator in Equation~\eqref{eq:MAP} or more generally some sort of consensus equilibrium among the data fitting term and the priors~\cite{buzzard2018plug}.

Taking an apparently different approach Romano \emph{et al.} introduced the regularization by denoising (RED) algorithm \citep{romano2017little} which uses a denoiser $D_\sigma$ to construct an explicit regularizer $\Gprior(\x) = \frac{1}{2}\x^T (\x - D_\sigma(\x))$. %
Under certain conditions (see below) its gradient $\nabla \Gprior = I - D_\sigma$ can be conveniently computed in terms of the denoiser, leading to a gradient descent scheme for the associated MAP estimator, which is very easy to implement.

\paragraph{Explicitly decoupled generative priors}
In another series of works pioneered by Bora \emph{et al.} \cite{bora2017compressed} and followed by \cite{shah2018solving,Raj2019,Menon2020,Huang2018,Hand2020} the Plug \& Play prior is explicitly provided by a generative model, most often a generative adversarial network $\generator$ that maps a latent variable $\z \sim \mathcal{N}(0,I)$ to an image $\x = \generator(\z)$ with the desired distribution $p_{X}$ as represented by the learning dataset. More precisely these methods solve an optimization problem on the latent variable $\z$
\begin{equation}
\label{eq:Bora-optim}
\hat{\z}_{\textsc{map}}\! =\! \argmin_\z \left \{\Fdata(\generator(\z),\vy) + \frac{1}{2} \alpha \|\z\|^{2} \right \}
\end{equation}
and the reconstructed image is provided by $\hat{\x}_{\textsc{map}} = \generator(\hat{\z}_{\textsc{map}})$.
As we show in the following sub-section and in appendix~\ref{sec:mapZ}, this corresponds (when $\alpha=1$) to the Maximum A Posteriori (MAP) estimator with respect to the $\z$ variable. In this work we adopt this framework with some extensions that help avoid getting trapped in spurious critical points of the non-convex objective function. 

\paragraph{Empirical success of Plug \& Play and RED}
Plug \& Play and RED approaches became very popular because they allow to repurpose very powerful state of the art denoisers as regularizers of a large family of inverse problems in a quite straightforward manner. They have been successfully applied to many different problems in imaging and they have thus empirically proven their superiority (in terms of achievable reconstruction quality with respect to more classical regularization techniques), and opened the way for the solution of more difficult inverse problems in imaging.
\paragraph{Theoretical questions}
The success of Plug \& Play and RED approaches largely outpaced our understanding of why and when these techniques lead to algorithms that provably converge to well-posed statistical estimators with well known properties.
This is not surprising because obtaining convergence guarantees for non-convex optimization problems under realistic conditions is quite challenging.

A notable exception where strong convergence results have been obtained is the particular case of compressed sensing, where the rows of the degradation operator (or sensing matrix) $\mA$ are independent realizations of a zero-mean Gaussian distribution. For this problem, Hand \emph{et al.} \cite{Hand2020,Huang2018} show that the optimization objective~\eqref{eq:Bora-optim} has almost no spurious stationary points when the generator is assumed to be a random ReLU network with Gaussian weights. As a consequence, a minor modification to the gradient descent algorithm in~\cite{bora2017compressed} converges with high probability to the global optimum.

In this paper we are interested in more general inverse problems, where the sensing matrix $\mA$ is not necessarily random but deterministic and highly structured most often dictated by our modeling of the acquisition device.
In this more general setting the hypotheses of the CS results are not necessarily satisfied, and the kind of convergence guarantees that could be established for PnP algorithms with non-convex priors are much weaker (typically only convergence to a stationary point or fixed point is provided, not necessarily a global optimum), and most works concentrate in the implicit case, where the prior is not explicitly provided by a generative model, but implicit in a denoising algorithm.

In such a case the actual prior is unknown, the existence of a density whose gradient or proximal operator is well approximated by a neural denoiser is most often not guaranteed~\citep{reehorst2018regularization}, and the convergence of the algorithm is not guaranteed either unless the denoiser is retrained with specific constraints like idempotence  \citep{gupta2018cnn,shah2018solving}, contractive residual \citep{ryu2019plug} or exact, invertible, smooth MMSE denoisers \citep{Xu2020}.

The effect of such training constraints on the quality of the denoisers and the associated priors is yet to be explored in detail. But even when these constraints are satisfied, convergence conditions can be quite restrictive, either \emph{(a)} requiring  the data-fitting term $\Fdata$ to be strongly convex \citep{ryu2019plug} (thus excluding many important problems in computational imaging where $\mA$ is not full rank like interpolation, super-resolution, deconvolution with a non-invertible kernel or compressive sensing), and/or \emph{(b)} constraining the regularization parameter $\lambda$ outside of its useful range \cite{ryu2019plug,Xu2020}.\footnote{
In~\cite{laumont2021pnpsgd} the PnP-ADMM and PnP-FBS algorithms introduced in \cite{ryu2019plug,Xu2020} are reported to converge in practice quite far beyond the conditions of the theorem, but require (to obtain optimal performance) the regularization parameter $\lambda$ to be tuned to values that are far outside the region where convergence is guaranteed. Also, the performance is significantly degraded if $\lambda$ is constrained to the range where convergence is guaranteed.
}

Similarly, an early analysis of the RED approach~\citep{reehorst2018regularization} provides a convergence proof, but only under quite restrictive conditions (locally homogeneous denoisers with symmetric Jacobian) which exclude most state of the art denoisers like DnCNN, BM3D, NLMeans. A more recent analysis of a stochastic variant of the RED algorithm \citep{laumont2021pnpsgd} (called PnP-SGD) significantly expands the family of denoisers that provide convergence guarantees, including DnCNN and the doubly-stochastic variant of NLM \citep{Sreehari2015}. These guarantees come, however, at the expense of a very small descent step which leads to a very computationally expensive algorithm with slow convergence. In addition, the experiments with PnP-SGD show that this algorithm is extremely sensitive to the initial condition, and it can be easily get stuck on spurious local minima if not initialized very carefully.
\paragraph{Focus of this work}
Very recent works focused on developing MAP estimation algorithms with convergence guarantees under more realistic conditions.
The convergence analysis of the RED framework, and its RED-PRO variant was further refined under a demicontractive condition for the denoiser \citep{Cohen2020}. This condition is, however, difficult to verify according to \citet{Pesquet2020} who provides an alternative convergence analysis framework based on firmly non-expansive denoisers for which explicit training procedures exist \citep{Terris2020}.
In this work we explore alternative new ways to bring theory and practice closer together, by proposing novel Plug \& Play algorithms to compute the MAP estimator of an inverse problem with a neural regularizer.
Unlike previous approaches which were based on implicit priors, or on GAN-based explicit priors, our approach is based on an explicit generative prior that has been trained as a Variational AutoEncoder (VAE). As we shall see later, the %
additional VAE structure provides: \emph{(i)} powerful mechanisms to avoid getting stuck in spurious local minima of the associated non-convex functional, and \emph{(ii)} convergence guarantees under much less restrictive conditions on the inverse problem $\Fdata$ and regularization parameter $\lambda$.

The next Section~\ref{sec:map-gen} reviews previous work on similar approaches to compute a MAP estimator from a generative prior that was trained either as a VAE or a GAN. Section~\ref{sec:jpmap-intro} briefly introduces our approach and how it relates to previous work. The section finishes with an overview of the rest of the paper.

\subsection{Maximum a Posteriori meets Generative Models} \label{sec:map-gen} %

Our approach focuses on PnP algorithms where the prior is provided by a generative model.
For instance one could use a generative adversarial network (GAN) to learn a generative model for $X=\generator(\Z)$ with $\Z\sim N(0,I)$ a latent variable.
The generative model induces a prior on $X$ via the push-forward measure $p_\X = \generator \sharp p_\Z$, which following \citep[section 5]{Papamakarios2019} can be developed as
$$
\PDF{\X}{\x} = 
\frac{\PDF{\Z}{\generator^{-1}(\x)}}%
{\sqrt{\det S(\generator^{-1}(\x))}}
\delta_\mathcal{M}(\x)
$$
where 
$S = \left(\frac{\partial \generator}{\partial \vz}\right)^T\left(\frac{\partial \generator}{\partial \vz}\right)$ is the squared Jacobian and 
the manifold $\mathcal{M} = \lbrace \vx \,:\, \exists \vz,\, \vx = \generator(\vz)\rbrace$ represents the image of the generator \generator. With such a prior $p_\X$, the \x-optimization \eqref{eq:MAP} required to obtain $\hat{\x}_\MAP$ becomes intractable (in general), for various reasons:
\begin{itemize}
    \item the computation of $\det S$,
    \item the inversion of $\generator$, and
    \item the hard constraint $\x\in\mathcal{M}$.
\end{itemize}
These operations are all memory and/or computationally intensive, except when they are partially addressed by the use of a normalizing flow like in \citep{Helminger2020,Whang2020}.\\

Current attempts to use such a generative model as a prior, like the one proposed by~\citet{bora2017compressed} for GANs, circumvent these difficulties by performing an optimization on \z\ (in the latent domain) instead of \x. 
Instead of solving Equation~\eqref{eq:MAP}, they solve
\begin{equation}\label{eq:MAPz}
\begin{split}
\hat{\z}_\MAP & = \argmax_\z \left\{ \ConditionalPDF{\Y}{\X}{\vy}{\generator(\vz)}\PDF{\Z}{\vz} \right\} \\
& = \argmin_\z \left \{\Fdata(\generator(\vz),\vy) + \frac{1}{2}\|\vz\|^2\right \},
\end{split}
\end{equation}
by assuming a standard Gaussian prior. This problem is much more tractable, and the corresponding \x-estimate is obtained as \begin{equation}
\hat{\x}_{\MAP-\z} = \generator(\hat{\z}_\MAP).
\end{equation}

As we show in appendix~\ref{sec:mapZ}, this new estimator does not necessarily coincide with $\hat{\x}_\MAP$ but it does correspond to the MAP-estimator of \x\ after the change of variable $\x=\generator(\z)$, namely
$$
\hat{\x}_{\MAP-\z} = \generator\left(
\argmax_\z \left \{\ConditionalPDF{\Z}{\Y}{\z}{\y} \right \}
\right).
$$

Since \generator\ is non-linear, this problem (or its equivalent formulation \eqref{eq:MAPz}) is highly non-convex and difficult to solve with global optimality guarantees. Nevertheless, in the particular case where $\mA$ is a random Gaussian matrix (compressed sensing case) or when $\Fdata$ is strongly convex, recent work shows that the global optimum can be reached with linear convergence rates by a small modification of a gradient descent algorithm \cite{Huang2018,Hand2020}, or by an ADMM algorithm with non-linear constraints \cite{latorre2019fastADMM,Benning2016,Valkonen2014}.
To the best of our knowledge, these results do not extend, however, to the more general case we are interested in here, where $\mA$ is deterministic and rank-deficient, and $\Fdata$ is consequently not strongly convex. In this more general setting, convergence guarantees for this optimization problem remain extremely difficult to establish, as confirmed by experimental results presented in Section~\ref{sec:experiments}.

A common technique to solve difficult optimization problems like the one in Equation~\eqref{eq:MAPz} is to use (Half Quadratic) splitting methods
 \begin{equation}\label{eq:MAPz-splitting}
 \hat{\x}_{\beta} = \argmin_\x \min_\z 
\underbrace{\left \{\Fdata(\vx,\vy) + \frac{\beta}{2} \|\vx - \generator(\vz)\|^2 + \frac{1}{2}\|\vz\|^2\right \}}_{J_{1,\beta}(\x,\z)}
\end{equation}
combined with a continuation scheme, namely:
\begin{equation}\label{eq:MAPz-splitting-continuation-scheme}
\hat{\x}_{\MAP-\z} = \lim_{\beta \to \infty}  \hat{\x}_{\beta}.
\end{equation}
The convergence of the continuation scheme in the last line is a standard result in $\Gamma$-convergence (see \cite{DalMaso1993} and appendix~\ref{sec:continuation}).
The corresponding splitting algorithm is presented in Algorithm \ref{alg:MAPz-splitting}.

\renewcommand{\algorithmiccomment}[1]{\hfill // {#1}}
\begin{algorithm}
\caption{\MAP-\z\ splitting}
\label{alg:MAPz-splitting}
\begin{algorithmic}[1]
\REQUIRE Measurements $\vy$,  Initial condition $\vx_0$, maxiter, $k_{\max}$, $\left\lbrace \beta_{0}, \dots, \beta_{k_{\max}}\right\rbrace$
\ENSURE $\hat{\vx} = \generator\left( \argmax_{\vz} \ConditionalPDF{\Z}{\Y}{\vz}{\vy} \right)$

\FOR{$k:= 0$ \TO $k_{\max}$}
\STATE $\beta := \beta_k$
\FOR{$n:=0$ \TO maxiter}
\STATE $\vz_{n+1} := \argmin_\vz J_{1,\beta}(\vx_n,\vz)$ %
\COMMENT{Nonconvex}
\STATE $\vx_{n+1} := \argmin_\vx J_{1,\beta}(\vx,\vz_{n+1})$ \COMMENT{Quadratic} 
\ENDFOR 
\STATE $\x_{0} := \vx_{n+1}$
\ENDFOR
\RETURN $\vx_{n+1}$

\end{algorithmic}
\end{algorithm}

However, unlike most cases of HQS which include a linear constraint between the two variables, this splitting algorithm still contains (line 4) a difficult non-convex optimization problem\footnote{In another context a primal-dual optimization algorithm was proposed to solve a similar optimization problem~\cite{Benning2016}, but this approach was not explored in the context where $\generator$ is a generative model.}.

\subsection{Proposed method: Joint $\textsc{MAP}_{x,z}$}\label{sec:jpmap-intro}
In this work we propose to address this challenge by substituting the difficult non-convex sub-problem by a local quadratic approximation provided by the encoder of a variational autoencoder.

Indeed, as we show in Section~\ref{sec:JPMAP_framework}, a variational autoencoder allows to interpret the splitting Equation~\eqref{eq:MAPz-splitting} as the negative logarithm of the joint posterior density $\ConditionalPDF{\X,\Z}{\Y}{\vx,\vz}{\vy}$. Therefore, solving Equation~\eqref{eq:MAPz-splitting} amounts to compute a joint $\textsc{map}_{x,z}$ estimator that we denote by $\hat{\x}_{\MAP_{\x,\z}}^{\beta}$. 
Moreover if the same joint conditional density $\ConditionalPDF{\X,\Z}{\Y}{\vx,\vz}{\vy}$ is decomposed in a different manner, it leads to an approximate expression that makes use of the encoder, and is quadratic in $\z$. If this approximation is good enough then the maximization of the joint log-posterior becomes a bi-concave optimization problem or approximately so. And in that case, an extension of standard bi-convex optimization results~\cite{Gorski2007} shows that the algorithm converges to a stationary point.

We also highlight the importance of correctly training the VAE in such a way that the encoder generalizes well to noisy values of $\x$ outside of the support of $\PDF{\X}{\x}$. This can be achieved by training the VAE to reconstruct their clean inputs with noise injected at the input level, as proposed by~\citet{Im2017}. We observe that this modified training does not degrade the quality of the generative model, but makes our quasi-bi-convex optimization procedure  much more robust.

Finally we show that a continuation scheme allows to obtain the $\MAP_\z$ estimator as the limit of a series of joint $\MAP_{\x,\z}$ optimizations. This continuation scheme, in addition to the quasi-bi-convex optimization, and the initialisation heuristic provided by the denoising encoder leads to a much more robust non-convex optimization scheme which more often converges to the right critical point than a straightforward gradient descent of the $\MAP_{\z}$ model.

The remainder of this paper is organized as follows. In Section~\ref{sec:JPMAP_framework} we derive a model for the joint conditional posterior distribution of space and latent variables $\vx$ and $\vz$, given the observation $\vy$. This model makes use of a generative model, more precisely a VAE with Gaussian decoder. We then propose an alternate optimization scheme to maximize the joint posterior model, and state convergence guarantees. Section~\ref{sec:experiments} presents first a set of experiments that illustrates the convergence properties of the optimization scheme. We then test our approach on classical image inverse problems, and compare its performance with state-of-the-art methods. Concluding remarks are presented in Section~\ref{sect:future_work}.

\section{From Variational Autoencoders to Joint Posterior Maximization}
\label{sec:JPMAP_framework}

Recently, some generative models based on neural networks have shown their capability to approximate the complex image distribution in a data-driven fashion. In particular, \emph{Variational Autoencoders (VAE)}~\cite{Kingma2014} combine variational inference to approximate unknown posterior distributions of latent variable models with the ability of neural networks to learn such approximations.

Consider a graphical model $\vz\to \vx$ in which we assume that a latent variable $\vz$ is responsible of the observed image $\vx$. For example, in an image of a handwritten digit we can imagine which digit is represented in the image, width, angle (and so on) as latent variables.
We choose a generative model
$$ \ptheta(\vx,\vz) = \ptheta(\vx|\vz)\PDF{\Z}{\vz} $$
 where $\PDF{\Z}{\vz}$ is some simple distribution (which we can easily sample from) and $\ptheta(\vx|\vz)$ is the approximation of the probability distribution of $\vx$ given $\vz$ parameterized by a neural network (with weights $\theta$) known as \emph{stochastic decoder}.
 
The intractability of $\ptheta(\vx) = \int \ptheta(\vx|\vz)\PDF{\Z}{\vz}\,d\vz$ is related to the posterior distribution $\ptheta(\vz|\vx)$ by
\begin{equation}\label{eq:intractable-posterior}
    \ptheta(\vz|\vx)=\frac{\ptheta(\vx|\vz)\PDF{\Z}{\vz}}{\ptheta(\vx)}.
\end{equation}
The \emph{variational inference} approach consists in approximating this posterior with another model 
 $\qphi(\vz|\vx)$ which, in our case, is another neural network with parameters $\phi$, called a \emph{stochastic encoder}.
 
Following \cite{Kingma2014}, we consider the \emph{Evidence Lower BOund (ELBO)} as
\begin{equation}
    \mathcal{L}_{\theta,\phi}(\vx) := \log \ptheta(\vx) - KL(\qphi(\vz|\vx)\;||\; \ptheta(\vz|\vx)) \le \log \ptheta(\vx)
\end{equation}
 where KL is the Kullback-Leibler divergence.
Thus, given a dataset $\mathcal{D}=\{\vx_1,\ldots,\vx_N\}$ of image samples, maximizing the averaged ELBO on $\mathcal{D}$ means maximizing $\log \ptheta(\mathcal{D})$ which is the maximum likelihood estimator of weights $\theta$ \emph{and} minimizing $KL(\qphi(\vz|\vx)\;||\; \ptheta(\vz|\vx))$ which enforces the approximated posterior $\qphi(\vz|\vx)$ to be similar to the true posterior $\ptheta(\vz|\vx)$.
 
It can be shown \cite{Kingma2014} that the ELBO can be rewritten as
\begin{equation}
 \mathcal{L}_{\theta,\phi}(\vx) = \mathbb{E}_{\qphi(\vz|\vx)} [\log \ptheta(\vx|\vz)] - KL(\qphi(\vz|\vx)\;||\; \PDF{\Z}{\vz}).
 \label{eq:vae_loss}
\end{equation}
 The first term in \eqref{eq:vae_loss} is a \emph{reconstruction loss} similar to the one of plain autoencoders: it enforces that the code $\vz\sim \qphi(\cdot|\vx)$ generated by the encoder $\qphi$ can be used by the decoder $\ptheta$ to reconstruct the original input $\vx$. The second term is a \emph{regularization term} that enforces the distribution $\qphi(\vz|\vx)$ of the latent code $\vz$ (given $\vx$) to be close to the prior distribution $\PDF{\Z}{\vz}$.
It is common to choose an isotropic Gaussian as the prior distribution of the latent code: $$ \PDF{\Z}{\vz} = \mathcal{N}(\z\,|\,0,I) \propto e^{-\|\vz\|^2/2}$$ and a Gaussian encoder $\qphi(\vz|\vx)=\mathcal{N}(\vz\,|\,\mu_\phi(\vx),\Sigma_\phi(\vx))$, so that the KL divergence in \eqref{eq:vae_loss} is straightforward to compute. 
For the decoder $\ptheta(\vx|\vz)$ a Gaussian decoder is the most common choice and as we will see we benefit from that.

\subsection{Learning approximations vs. encoder approximations}

 In this work we construct an image prior using a Variational Autoencoder (VAE). Like any machine learning tool VAEs make different kinds of approximations. Let's distinguish two types of approximations that shall be important in the sequel:
 
 \begin{description}
 \item[Learning approximation:] The ideal prior $p_X^*$ can only be approximated by our VAE due to its architectural constraints, finite complexity, truncated optimization algorithms, finite amount of data and possible biases in the data. Due to all these approximations, after learning we have only access to an approximate prior $p_X \approx p^*_X$.
 VAEs give access to this approximate prior $p_X$ via a generative model: taking samples of a latent variable $\Z$ with known distribution $\mathcal{N}(0,I)$ in $\R^\zdim$ (with $\zdim\ll\xdim)$, and feeding these samples through a learned decoder network, we obtain samples of $X\sim p_X$. The approximate prior itself
        \begin{equation}\label{eq:pX}
             \PDF{\X}{\vx} = \int \ptheta(\vx|\vz)\, \PDF{\Z}{\vz}\, d\vz
        \end{equation}
 is intractable because it requires computing an integral over all possible latent codes $\vz$.
 However the approximate joint distribution is readily accessible
      $$ \PDF{\X,\Z}{\vx,\vz} = \ptheta(\vx|\vz)\, \PDF{\Z}{\vz}$$
 thanks to $\ConditionalPDF{\X}{\Z}{\vx}{\vz} = \ptheta(\vx|\vz)$ which is provided by the decoder network.
 
 \item[Encoder approximation:] In the previous item we considered the VAE as a generative model without making use of the encoder network. The encoder network
 $$\tilde{p}_{\Z|\X}(\z\,|\,\x):=\qphi(\vz|\vx) \approx \ConditionalPDF{\Z}{\X}{\vz}{\vx} $$
 is introduced as an approximate way to solve the intractability of $\ConditionalPDF{Z}{X}{\z}{\x} = \ptheta(\z|\x)$ (which is related to the intractability of $\ptheta(\x)$ as observed in equation~\eqref{eq:intractable-posterior}).
 
 Using the encoder network we can provide an alternative approximation for the joint distribution
 $$ \tilde{p}_{X,Z}(\x,\z) := \qphi(\vz|\vx)\, \PDF{\X}{\vx} \approx \PDF{\X,\Z}{\vx,\vz} $$
 which shall be useful in the sequel.
 
\end{description}
 
 Put another way, the ideal joint distribution $p_{\X,\Z}^*$ is inaccessible, but can be approximated in two different ways:
 
 The first expression denoted $\PDF{\X,\Z}{\vx,\vz}$ only uses the decoder and is only affected by the \emph{learning approximation}
     $$ p_{X,Z}^*(\x,\z) \approx \PDF{\X,\Z}{\vx,\vz} := \ptheta(\vx|\vz)\, \PDF{\Z}{\vz}.$$

 The second expression denoted $\tilde{p}_{X,Z}(\x,\z)$ uses both encoder and decoder and is affected both by the \emph{learning approximation} and by the \emph{encoder approximation}
 $$ \PDF{\X,\Z}{\vx,\vz} \approx \tilde{p}_{X,Z}(\x,\z) := \qphi(\vz|\vx)\, \PDF{\X}{\vx} $$

 In the following subsection we shall forget about the ideal prior $p^*_X$ and joint distribution $p_{X,Z}^*$ which are both inaccessible. Instead we accept $p_X$ (with its learning approximations) as our prior model which shall guide all our estimations.
 The approximation symbol shall be reserved to expressions that are affected by the encoder approximation \emph{in addition to} the learning approximation.

\subsection{Variational Autoencoders as Image Priors}
\label{subsec:VAEs_priors}

To obtain the Maximum a Posteriori estimator (MAP), we could plug in the approximate prior $p_X$ in equation~\eqref{eq:MAP}, but this leads to a numerically difficult problem to solve due to the intractability of $p_X$. Instead, we propose to maximize the joint posterior $\ConditionalPDF{\X,\Z}{\Y}{\vx,\vz}{\vy}$ over $(\vx,\vz)$ which is equivalent to minimizing
\begin{equation} \label{eq:JPMAP}
\begin{aligned}
        J_1(\vx,\vz) & := -\log \ConditionalPDF{\X,\Z}{\Y}{\vx,\vz}{\vy} \\
        & = 
        -\log \ConditionalPDF{\Y}{\X,\Z}{\vy}{\vx,\vz} \ptheta(\vx\,|\,\vz)\PDF{\Z}{\vz} \\
        &= \Fdata(\vx,\vy) + \Htheta(\vx,\vz) + \frac{1}{2}\|\vz\|^2.
\end{aligned}
\end{equation}
Note that the first term is quadratic in $\vx$ (assuming~\eqref{eq:data-term}), the third term is quadratic in $\vz$ and all the difficulty lies in the coupling term $\Htheta(\vx,\vz)=-\log \ptheta(\x\,|\,\z)$. 
For Gaussian decoders~\cite{Kingma2014}, the latter can be written as
\begin{equation}\label{eq:coupling_term}
\begin{split}
 \Htheta(\vx,\vz) & =  %
    \frac{1}{2}\Big (\xdim\log(2\pi) + \log\det\SigmaDecoder(\vz)  \\ 
     & \quad + \,  \|\SigmaDecoder^{-1/2}(\vz)(\vx-\muDecoder(\vz))\|^2 \Big ).
    \end{split}
\end{equation}
which is also convex in $\vx$. Hence, minimization with respect to $\vx$ takes the convenient closed form:
\begin{equation}\label{eq:xmin}
\begin{split}
    \argmin_\vx J_1(\vx,\vz) & = 
    \left(\mA^T\mA + \sigma^2\SigmaDecoder^{-1}(\vz)\right)^{-1} \\
    & \quad \times \left(\mA^T\vy + \sigma^2\SigmaDecoder^{-1}(\vz)\muDecoder(\vz)\right).
\end{split}
\end{equation}

Unfortunately the coupling term $\Hcoupling$ and hence $J_1$ is a priori non-convex in $\vz$. 
As a consequence the $\vz$-minimization problem
\begin{equation}\label{eq:zmin-exact}
    \argmin_\vz J_1(\vx,\vz) 
 \end{equation}
 is a priori more difficult.
However, for Gaussian encoders, VAEs provide an approximate expression for this coupling term which is quadratic in $\vz$. Indeed, given the equivalence  
\begin{equation}
    \begin{split}
        \ptheta(\vx\,|\,\vz) \, \PDF{\Z}{\vz}
        & = \PDF{\X,\Z}{\vx,\vz} \\
        & = \ConditionalPDF{ \Z }{\X}{\vz}{\vx} \, \PDF{\X}{\vx} \\
        & \approx \qphi(\vz\,|\,\vx) \, \PDF{\X}{\vx}
    \end{split}
\label{eq:encoder_approximation}
\end{equation}
we have that 
\begin{equation}
\label{eq:autoencoder-approx}
    \Htheta(\vx,\vz) + \frac{1}{2}\|\vz\|^2 \approx \Kphi(\vx,\vz) - \log \PDF{\X}{\vx}.
\end{equation}
where $\Kphi(\vx,\vz)=-\log \qphi(\vz\,|\,\vx)$. Therefore, this new coupling term becomes
\begin{align*}
 \Kphi(\vx,\vz) & = -\log \mathcal{N}\left(\vz; \muEncoder(\vx),\SigmaEncoder(\vx)\right) \\
              & = \frac{1}{2} \big[ \zdim \log(2\pi) + \log \det \SigmaEncoder(\vx) \nonumber \\
              & \quad +  \|\SigmaEncoder^{-1/2}(\vx)(\vz-\muEncoder(\vx))\|^2 \big],
\end{align*}
which is quadratic in $\vz$. This provides an approximate expression for the energy~\eqref{eq:JPMAP} that we want to minimize, namely 
\begin{equation}\label{eq:JPMAPapprox}
        J_2(\vx,\vz) := \Fdata(\vx,\vy) + \Kphi(\vx,\vz) - \log \PDF{\X}{\vx}
        \approx J_1(\vx,\vz).
\end{equation}
This approximate functional is quadratic in $\vz$, and minimization with respect to this variable yields
\begin{equation}\label{eq:zmin}%
    \argmin_\vz J_2(\vx,\vz) =
    \muEncoder(\vx).
\end{equation}

In the case of linear VAEs,
\begin{align}
    \PDF{\Z}{\vz} &= \Normal(\vz;0,I)\\
    \ConditionalPDF{\X}{\Z}{\vx}{\vz} &= \Normal(\vx;V_\decoderParams\vz+v_\decoderParams,\SigmaDecoder).
    \label{eq:pPCA_decoder}
\end{align}
It is easily shown that the posterior is also Gaussian~\cite{Lucas2019}, namely
\begin{equation}
    \ConditionalPDF{\Z}{\X}{\vz}{\vx} = \Normal(\vz;MV_\decoderParams^T (x-v_\decoderParams),\SigmaDecoder M)\qquad
    \text{where }M=(V_\decoderParams^TV_\decoderParams+\SigmaDecoder)^{-1}.
    \label{eq:linear_posterior}
\end{equation}
Hence, the linear encoder $\qphi(\vz\,|\,\vx)$ that minimizes the ELBO is that of equation~\eqref{eq:linear_posterior} so the approximation~\eqref{eq:encoder_approximation} is exact and then $J_1=J_2$.

\subsection{Alternate Joint Posterior Maximization}
\label{sect:assumptions}

The previous observations suggest to adopt an alternate scheme to minimize $-\log \ConditionalPDF{\X,\Z}{\Y}{\vx,\vz}{\vy}$ in order to solve the inverse problem. We begin our presentation by 
a simple version of the proposed algorithm, which aims at managing the case where
the approximation of $J_1$ by $J_2$ is exact 
(at least in the sense given in Assumption \ref{exact-approximation} below); 
then we propose an adaptation for the more realistic non-exact case and we explore its convergence properties.

When $J_1=J_2$ is a bi-convex function, Algorithm~\ref{alg:JPMAPexact} is known as \emph{Alternate Convex Search}. Its behavior has been studied in \citep{Gorski2007,Aguerrebere2014b}.
Here we shall consider the following (strong) assumption, which includes the strictly bi-convex case ($J_1=J_2$):
\begin{assumption}\label{exact-approximation}
For any $\vx$, if $\vz^*$ is a global minimizer of $J_2(\vx,\cdot)$, then $\vz^*$ is a global minimizer of $J_1(\vx,\cdot)$.
\end{assumption}

The proposed 
alternate minimization takes the simple and fast form depicted in Algorithm~\ref{alg:JPMAPexact}, which can be shown to converge to a 
stationary point of $J_1$
under Assumptions \ref{exact-approximation} and \ref{functioncondition},
as stated in 
Proposition~\ref{thm:convergence-approx}
below. Note that the minimization in step~2 of Algorithm~\ref{alg:JPMAPexact} does not require the knowledge of the unknown term $- \log \PDF{\X}{\vx}$ in Equation~\eqref{eq:JPMAPapprox} since it does not depend on $\vz$.

\renewcommand{\algorithmiccomment}[1]{\hfill // #1}
\begin{algorithm}
\caption{Joint posterior maximization - exact case}
\label{alg:JPMAPexact}
\begin{algorithmic}[1]
\REQUIRE Measurements $\vy$, Autoencoder parameters \decoderParams, \encoderParams, Initial condition $\vx_0$
\ENSURE $\hat{\vx}, \hat{\vz} = \argmax_{\vx,\vz} \ConditionalPDF{\X,\Z}{\Y}{\vx,\vz}{\vy}$

\FOR{$n:=0$ \TO maxiter}
\STATE $\vz_{n+1} := \argmin_\vz J_2(\vx_n,\vz)$ \COMMENT{Quadratic approx}%
\STATE $\vx_{n+1} := \argmin_\vx J_1(\vx,\vz_{n+1})$ \COMMENT{Quadratic}%
\ENDFOR 
\RETURN $\vx_{n+1}, \vz_{n+1}$

\end{algorithmic}
\end{algorithm}

The convergence analysis of the proposed schemes requires some general assumptions on the functions $J_1$ and $J_2$ :
\begin{assumption}\label{functioncondition}
$J_1(\cdot,\vz)$ is convex and admits a unique minimizer for any $\vz$.
Moreover, $J_1$ is coercive and continuously differentiable.
\end{assumption}

The unicity of the minimizers of the partial function $J_1(\cdot,\vz)$ can be dropped. In this case, the proof of the convergence of Algorithm \ref{alg:JPMAPexact} has to be slightly adapted.

The convergence property of Algorithm \ref{alg:JPMAPexact} will be investigated in a wider framework below (Proposition \ref{thm:convergence-approx}).
Note that all the properties required in Assumption \ref{functioncondition} are satisfied if we use a differentiable activation function like the Exponential Linear Unit (ELU)~\citep{Clevert2016} with $\alpha=1$, instead of the more common ReLU activation function. More details can be found in Appendix~\ref{sec:J1prop}.

\subsection{Approximate Alternate Joint Posterior Maximization}
When the autoencoder approximation in~\eqref{eq:JPMAPapprox} is not exact (Assumption~\ref{exact-approximation}),
the energy we want to minimize in Algorithm \ref{alg:JPMAPexact}, namely $J_1$ may not decrease. To ensure the decay, some additional steps can be added.
Noting that
the approximation provided by $J_2$ 
provides a fast and accurate heuristic to initialize the minimization of $J_1$,
an alternative scheme is proposed
in Algorithm~\ref{alg:JPMAP2new}.

\renewcommand{\algorithmiccomment}[1]{\hfill // #1}
\begin{algorithm}
\caption{Joint posterior maximization - approximate case}
\label{alg:JPMAP2new}
\begin{algorithmic}[1]
\REQUIRE Measurements $\vy$, Autoencoder parameters \decoderParams, \encoderParams, Initial conditions $\vx_0, \vz_0$ 
\ENSURE $\hat{\vx}, \hat{\vz} = \argmax_{\vx,\vz} \ConditionalPDF{\X,\Z}{\Y}{\vx,\vz}{\vy}$

\FOR{$n:=0$ \TO maxiter}
\STATE $\vz^{1} := \argmin_\vz J_2(\vx_n,\vz)$ \COMMENT{Equation~\eqref{eq:zmin}}
\STATE $\vz^{2} := \GD{\vz}{J_1(\vx_n,\vz)}{\vz^{1}}$
\STATE $\vz^{3} := \GD{\vz}{ J_1(\vx_n,\vz)}{\vz_{n}}$

\FOR{$i:=1$ \TO 3}
\STATE $\vx^i := \argmin_\vx J_1(\vx,\vz^i)$ \COMMENT{Equation~\eqref{eq:xmin}} 
\ENDFOR 

\STATE $i^* := \argmin_{i \in \{1,2,3\}} J_1(\vx^i,\vz^i) $
\STATE $(\vx_{n+1},\vz_{n+1}) := (\vx^{i^*},\vz^{i^*}) $

\ENDFOR 
\RETURN $\vx_{n+1}, \vz_{n+1}$

\end{algorithmic}
\end{algorithm}
In Algorithm~\ref{alg:JPMAP2new}, $\gd$ is a gradient descent scheme such that for any starting point $\vz_0$, the output $\vz^+$ satisfies
\[
\frac{\partial J_1}{\partial \vz}(\vx,\vz^+) = 0
\quad\text{and}\quad
J_1(\vx,\vz^+) \leq J_1(\vx,\vz_0)
\]
Hence, one can consider for instance a gradient descent scheme which finds a local minimizer of $J_1(\vx,\cdot)$ starting from $\vz_0$.

Our experiments with Algorithm~\ref{alg:JPMAP2new} (Section~\ref{sec:qualitative-JPMAP2}) show that during the first few iterations (where the approximation provided by $J_2$ is good enough) $\vz^1$ and $\vz^2$ reach convergence faster than $\vz^3$. After a critical number of iterations the opposite is true (the initialization provided by the previous iteration is better than the $J_2$ approximation, and $\vz^3$ converges faster).

These observations suggest that a faster execution, with the same convergence properties, can be achieved by the variant in Algorithm~\ref{alg:JPMAP3new},
which avoids the costly computation of $\vz^2$ and $\vz^3$ when unnecessary.
Hence, in practice, we will use Algorithm \ref{alg:JPMAP3new} rather than Algorithm \ref{alg:JPMAP2new}. However, Algorithm~\ref{alg:JPMAP2new} provides a useful tool for diagnostics. Indeed, the comparison of the evaluation of $J_1(\vx^i,\vz^i)$ for $i=1,2,3$ performed in step 8 permits to assess the evolution of the approximation of $J_1$ by $J_2$.

\renewcommand{\algorithmiccomment}[1]{\hfill // #1}
\begin{algorithm}
\caption{Joint posterior maximization - approximate case (faster version)}
\label{alg:JPMAP3new}
\begin{algorithmic}[1]
\REQUIRE Measurements $\vy$, Autoencoder parameters \decoderParams, \encoderParams, Initial condition $\vx_0$, iterations $n_1 \leq n_2 \leq n_{\max}$
\ENSURE $\hat{\vx}, \hat{\vz} = \argmax_{\vx,\vz} \ConditionalPDF{\X,\Z}{\Y}{\vx,\vz}{\vy}$

\FOR{$n:=0$ \TO $n_{\max}$}

\STATE done := FALSE

\IF{$n<n_1$}
\STATE $\vz^1 := \argmin_\vz J_2(\vx_n,\vz)$
\COMMENT{Equation~\eqref{eq:zmin}} %
\STATE $\vx^1 := \argmin_\vx J_1(\vx,\vz^1)$ \COMMENT{Equation~\eqref{eq:xmin}} %

\IF{$J_1(\vx^1,\vz^1) < J_1(\vx_n,\vz_n)$} 
    \STATE $i^* := 1$ \COMMENT{%
                                                        $J_2$ is good enough}
    \STATE done := TRUE
\ENDIF
\ENDIF

\IF{\NOT{done} \AND $n<n_2$}
    \STATE $\vz^1 := \argmin_\vz J_2(\vx_n,\vz)$
    \STATE $\vz^2 := \GD{\vz}{J_1(\vx_n,\vz)}{\vz^{1}}$
    \STATE $\vx^2 := \argmin_\vx J_1(\vx,\vz^2)$ \COMMENT{Equation~\eqref{eq:xmin}} %
    
    \IF{$J_1(\vx^2,\vz^2) < J_1(\vx_n,\vz_n)$}
        \STATE $i^* := 2$ \COMMENT{%
                                                        $J_2$ init is good enough}
        \STATE done := TRUE
    \ENDIF
\ENDIF

\IF{\NOT{done}}
    \STATE $\vz^3 := \GD{\vz}{J_1(\vx_n,\vz)}{\vz_{n}}$
    \STATE $\vx^3 := \argmin_\vx J_1(\vx,\vz^3)$ \COMMENT{Equation~\eqref{eq:xmin}} %
    \STATE $i^* := 3$
\ENDIF

\STATE $(\vx_{n+1}, \vz_{n+1}) := (\vx^{i^*}, \vz^{i^*})$

\ENDFOR 
\RETURN $\vx_{n+1}, \vz_{n+1}$

\end{algorithmic}
\end{algorithm}

Algorithm~\ref{alg:JPMAP3new} is still quite fast when $J_2$ provides a sufficiently good approximation, since in that case the algorithm chooses $i^*=1$, and avoids any call to the iterative gradient descent algorithm. Even if we cannot give a precise definition of what {\em sufficiently good} means, the sample comparison of $\Kphi$ and $\Htheta$ as functions of $\vz$, displayed in Figure~\ref{fig:encoder_approximation}, shows that the approximation is fair enough in the sense that it preserves the global structure of $J_1$. The same behavior was observed for a large number of random tests.

Note that Algorithm \ref{alg:JPMAPexact} is a particular instance of Algorithm \ref{alg:JPMAP3new} in the case where Assumption~\ref{exact-approximation} holds, and  $n_1=n_2=0$ and if grad descent gives a global minimizer of the considered function (in this case, the computation of $\vz^1$, $\vz^2$, are skipped and only $\vz^3$ is computed).

\begin{proposition}[Convergence of Algorithm~\ref{alg:JPMAP3new}]\label{thm:convergence-approx}%
Let $\left\lbrace (\vx_{n},\vz_{n}) \right\rbrace$ be a sequence generated by Algorithm~\ref{alg:JPMAP3new}.
Under Assumption
\ref{functioncondition}
we have that:
\begin{enumerate}
\item The sequence $\left\lbrace J_1(\vx_{n},\vz_{n}) \right\rbrace$
converges monotonically when $n \to \infty$.
\item The sequence $\left\lbrace (\vx_{n},\vz_{n})  \right\rbrace$ has at least one accumulation point.
\item All accumulation points 
of $\left\lbrace (\vx_{n},\vz_{n})  \right\rbrace$
are
stationary points
of $J_1$ and they all have the same function value.
\end{enumerate}
\end{proposition}

\begin{proof}

Since we are interested in the behaviour for $n\to\infty$, we assume $n>n_2$ in Algorithm~\ref{alg:JPMAP3new}.

1. Since $n>n_2$ the algorithm chooses $i^*=3$ and $\vz_{n+1} = \vz^3$. According to the definition of grad descent, one has
\[
J_1(\vx_n,\vz_{n+1}) \leq J_1(\vx_n,\vz_n)
\]
and by optimality one has
\[
J_1(\vx_{n+1},\vz_{n+1}) \leq J_1(\vx_n,\vz_{n+1}).
\]
Hence, since $J_1$ is coercive (thus, lowerbounded), Statement 1 is straightforward.
    
2. Thanks to the coercivity of $J_1$, the sequences $\left\lbrace (\vx_{n},\vz_{n}) \right\rbrace$ and $\left\lbrace (\vx_{n},\vz_{n+1}) \right\rbrace$ are bounded, thus admit an accumulation point.

3.  Using Fermat's rule and the definition of grad descent, one has
    \[\frac{\partial J_1}{\partial \vz}(\vx_n,\vz_{n+1}) = 0
    \quad\text{and}\quad \frac{\partial J_1}{\partial \vx}(\vx_{n+1},\vz_{n+1})=0.\]
    Let $(\vx^*,\vz^*)$ be an accumulation point of $\left\lbrace (\vx_{n},\vz_{n}) \right\rbrace$. By double extraction, one can find two subsequences such that
    \[
(\vx_{n_j+1},\vz_{n_j+1}) \to (\vx^*,\vz^*)    \text{ and }
   (\vx_{n_j},\vz_{n_j+1}) \to (\hat\vx^*,\vz^*)    \]
   By continuity of $\nabla J_1$, one gets that
    \[\frac{\partial J_1}{\partial \vz}(\hat\vx^*,\vz^*) = 0
    \quad\text{and}\quad \frac{\partial J_1}{\partial \vx}(\vx^*,\vz^*)=0\]
    In particular, the convexity of $J_1(\cdot,\vz^*)$ and Assumption \ref{functioncondition} ensure that $\vx^*$ is
    a %
    global minimizer of $J_1(\cdot,\vz^*)$. Besides, the inequalities proved in Point 1  above show that
    \[
    J_1(\vx^*,\vz^*) = J_1(\hat \vx^*,\vz^*) = \lim_{n\to\infty} J_1(\vx_n,\vz_n)
    \]
    that is, $\hat\vx^*$ is also a 
    global minimizer of $J_1(\cdot,\vz^*)$.
    Since $J_1(\cdot,\vz^*)$ has a unique minimizer, one has $\hat\vx^*=\vx^*$, and
\[
\frac{\partial J_1}{\partial \vz}(\vx^*,\vz^*)
= 0\]
namely
$({\vx}^*,\vz^*)$
is a stationary point of $J_1$.
Note that we have also proved that $\vx_{n_j}$ and $\vx_{n_j+1}$ have same limit.
\end{proof}

\begin{remark} Note that if $n_1=n_2=\infty$ we cannot assume that $i^*=3$. In that case statements 1 and 2 are still valid but the third statement is not. The reason is that for $i^*\in\lbrace 1, 2 \rbrace$ we cannot guarantee the chain of inequalities 
\[
J_1(\vx_{n+1},\vz_{n+1}) \leq J_1(\vx_n,\vz_{n+1}) \leq J_1(\vx_n,\vz_n)
\]
but only
\[
J_1(\vx_{n+1},\vz_{n+1}) \leq J_1(\vx_n,\vz_n).
\]
This is consistent with the design of the algorithm where iterations $n<n_2$ serve as an heuristic to guide the algorithm to a sensible critical point. However, convergence to a critical point is only guaranteed by the final iterations $n>n_2$.
\end{remark}

\subsection{MAP-z as the limit case for $\beta\to\infty$}

If one wishes to compute the \MAP-\z\  estimator instead of the joint \MAP-\x-\z\  from the previous section, one has two options:

\begin{enumerate}
    \item Use your favorite gradient descent algorithm to solve equation~\eqref{eq:MAPz}.
    \item Use Algorithm~\ref{alg:JPMAP3new} to solve a series of joint \MAP-\x-\z\  problems with increasing values of $\beta \to \infty$ as suggested in Algorithm~\ref{alg:MAPz-splitting}.
\end{enumerate}

In the experimental section we show that the second approach most often leads to a better optimum.

In practice, in order to provide a stopping criterion for Algorithm~\ref{alg:MAPz-splitting} and to make a sensible choice of $\beta$-values we reformulate Algorithm~\ref{alg:MAPz-splitting} as a constrained optimization problem
\[
\argmin_{\x,\z\,:\, \|\generator(\z)-\x\|^2 \leq \varepsilon} \Fdata(\x,\y) + \frac12\|\z\|^2.
\]
The corresponding Lagrangian form is
\begin{equation}\label{eq:lagrangian_emm}
\max_\beta \min_{\x,\z} \Fdata(\x,\y) + \frac12\|\z\|^2 + \beta \left( \|\generator(\z)-\x\|^2 - \varepsilon \right)^+
\end{equation}
and we use the exponential multiplier method \cite{Tseng1993}
to guide the search for the optimal value of $\beta$ (see Algorithm~\ref{alg:MAPzCS})

\renewcommand{\algorithmiccomment}[1]{\hfill // #1}
\begin{algorithm}
\caption{\MAP-\z\ as the limit of joint \MAP-\x-\z.}
\label{alg:MAPzCS}
\begin{algorithmic}[1]
\REQUIRE Measurements $\vy$, Tolerance $\varepsilon$, Rate $\rho>0$, Initial $\beta_0$, Initial $\x_0$, Iterations $0\leq n_1 \leq n_2 \leq n_{\max}$
\ENSURE $\argmin_{\x,\,\z\,:\, \|\generator(\z)-\x\|^2 \leq \varepsilon} \Fdata(\x,\y) + \frac12\|\z\|^2$.

\STATE $\beta := \beta_0$
\STATE $\x^0, \z^0 := $ Algorithm~\ref{alg:JPMAP3new} starting from $\x=\x_0$ with $\beta, n_1,\,n_2,\,n_{\max}$.
\STATE converged := FALSE
\STATE $k := 0$
\WHILE{\NOT{converged}}
\STATE $\x^{k+1}, \z^{k+1} := $ Algorithm~\ref{alg:JPMAP3new} starting from $\x=\x^k$ with $\beta$ and $n_1=n_2=0$
\STATE $C = \|\generator(\z^{k+1}) - \x^{k+1}\|^2 - \varepsilon$
\STATE $\beta := \beta \exp(\rho C)$
\STATE converged := $(C \leq 0)$
\STATE $k := k+1$
\ENDWHILE
\RETURN $\vx^k, \vz^k$

\end{algorithmic}
\end{algorithm}

\section{Experimental results}
\label{sec:experiments}

\subsection{Baseline algorithms}
\label{sec:baseline_algorithms}
To validate our approach, we perform comparisons on several inverse problems with the following algorithms:
\begin{itemize}
\item CSGM (Bora et al.~\cite{bora2017compressed}) directly computes the $\z-\MAP$ estimator as defined in Equation~\eqref{eq:MAPz} using gradient descent. We run CSGM using the decoder of a VAE as generator $\generator$ starting at random $\z_0$. In addition, as Bora et al. note that random restarts are important for good performance, we also compute the best result (as measured by~\eqref{eq:MAPz}) among $m=10$ different random initializations $\z_0$ and refer to this variant as mCSGM.
\item PULSE \citep{Menon2020} is very similar to CSGM but restricts the search of the latent code $\z$ to the sphere of radius $\sqrt{\zdim}$, arguing that it concentrates most of the probability mass of a Gaussian distribution $\Normal(0,I)$ on a high-dimensional space $\RR^\zdim$. 
\item PGD-GAN \citep{shah2018solving} performs a projected gradient descent of $\Fdata(\vx,\vy)$ wrt $\x$:
\begin{equation}
    \begin{split}
        \left\{\begin{array}{l}
            \vw_k = \x_k - \eta \mA^T(\mA\x_k-\y)\\
            \x_{k+1} = \generator(\argmin_\z \|\vw_k-\generator(\z)\|).
        \end{array}
        \right.
    \end{split}
\end{equation}
\item In addition, we implement the splitting method of Algorithm~\ref{alg:MAPz-splitting} which is a simple continuation scheme for the $\z-\MAP$ estimator of Equation~\eqref{eq:MAPz}.
\end{itemize}
For a fair comparison we run all algorithms on the same prior, \emph{i.e.} the same generator network $\generator=\muDecoder$ where $\muDecoder$ is the decoder mean from the VAE model that we trained for JPMAP.

\subsection{Inverse problems}

Here, we briefly describe the inverse problems $\vy=\mA\vx+\eta$, $\eta\sim\Normal(0,\sigma^2 I)$ to be considered for validating our approach:
\begin{itemize}
    \item \emph{Denoising}: $\mA=I$ and $\sigma$ large.%
    \item \emph{Compressed Sensing}: the sensing matrix $\mA\in\RR^{q\times\xdim}$ has Gaussian random entries $\mA_{ij}\sim\Normal(0,1/q)$, where $q\ll\xdim$ is the number of measurements.%
    \item \emph{Interpolation}: $\mA$ is a diagonal matrix  with random binary entries, so masking a percentage $p$ of the image pixels.%
    \item \emph{(Non-blind) Deblurring}: $\mA\vx = h\ast\vx$ where $h$ is a known convolution kernel.%
    \item \emph{Super-resolution}: $\mA$ is a downsampling/decimation operator of scaling factor $s$.
\end{itemize}

\subsection{AutoEncoder and dataset}
\label{subsec:exp_setup}

In order to test our joint prior maximization model we first train a Variational Autoencoder like in \citep{Kingma2014} on the training data of MNIST handwritten digits \citep{MNIST}.

The \emph{stochastic encoder} takes as input an image $\vx$ of $28\times28 = 784$ pixels and produces as an output the mean and (diagonal) covariance matrix of the Gaussian distribution $q_\phi(\vz|\vx)$, where the latent variable $\vz$ has dimension $8$.
The architecture of the encoder is composed of 3 fully connected layers with ELU activations (to preserve continuous differentiability). The sizes of the layers are as follows:
$784\to 500\to 500 \to (8+8).$
Note that the output is of size $8+8$ in order to encode the mean and diagonal covariance matrix, both of size 8.\\

The \emph{stochastic decoder} takes as an input the latent variable $\vz$ and outputs the mean and covariance matrix of the Gaussian distribution $p_\theta(\x|\z)$. Following \citep{TwoStageVAE} we chose here an isotropic covariance $\SigmaDecoder(\vz)=\gamma^2 I$ where $\gamma>0$ is trained, but independent of $\vz$. This choice simplifies the minimization problem \eqref{eq:zmin-exact}, because the term $\det\Sigma_\theta(\vz)$ (being constant) has no effect on the $\vz$-minimization. The architecture of the decoder is also composed of 3 fully connected layers with ELU activations (to preserve continuous differentiability). The sizes of the layers are as follows:
$8\to 500\to 500 \to 784.$
Note that the covariance matrix is constant, so it does not augment the size of the output layer which is still $784=28\times28$ pixels.\\

We also trained a VAE on CelebA~\cite{liu2015faceattributes} images cropped to $64\times 64\times 3$, with latent dimension ranging from 64 to 512.
We choose a DCGAN-like~\cite{radford2015unsupervised} CNN architecture as encoder and a symmetrical one as decoder with ELU activations, batch normalization and isotropic covariance as before. For more details, see the code\footnote{Code available at \url{https://github.com/mago876/JPMAP}.}.\\

We train these architectures using PyTorch~\cite{paszke2017automatic} with batch size 128 and Adam algorithm for 200 epochs with learning rate 0.0001 and rest of the parameters as default.

\begin{figure*}[htbp]
\begin{center}
  \subfigure[Denoising]{
  \includegraphics[width=0.3\textwidth]{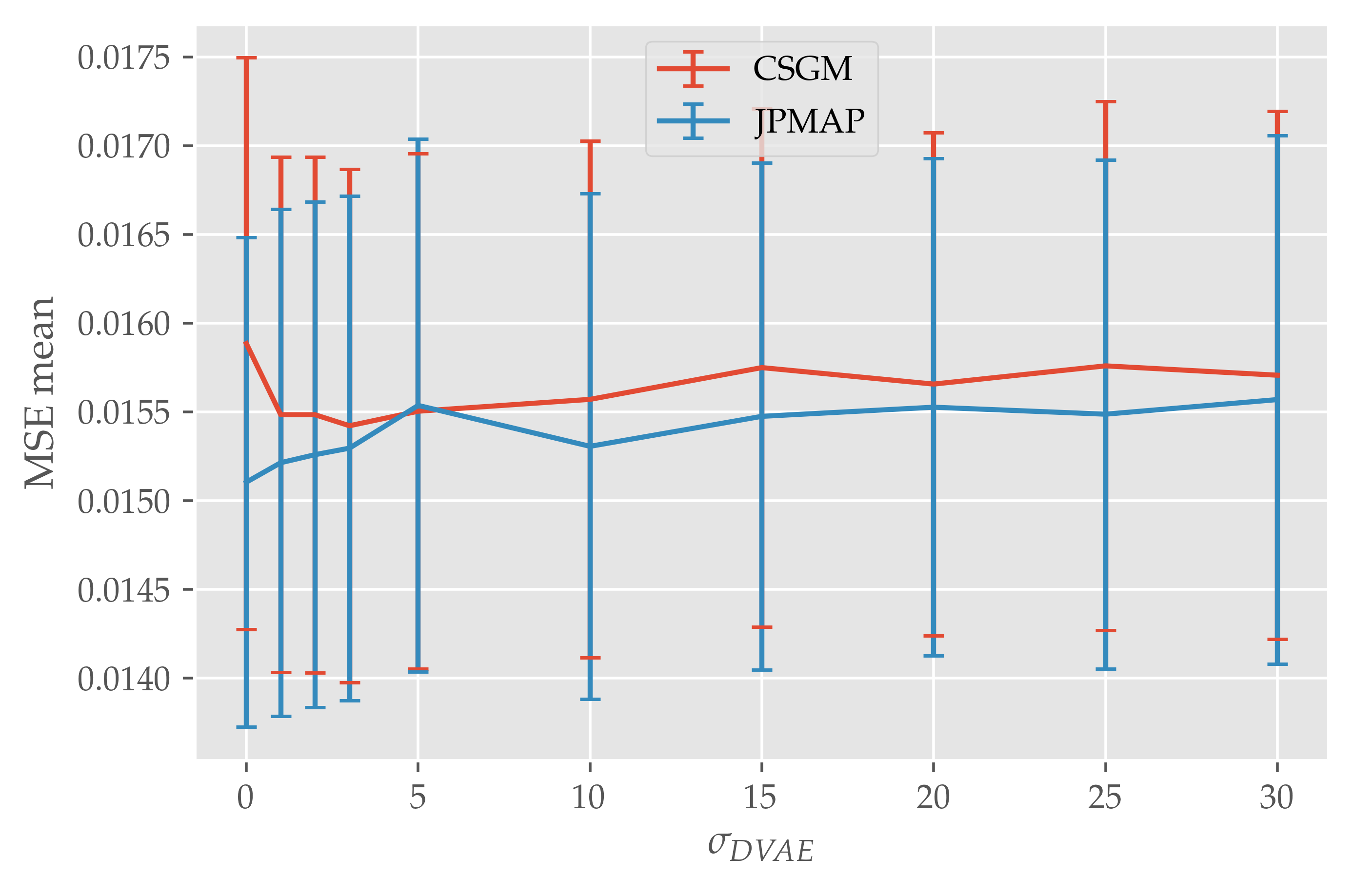}
  }
  \subfigure[Compressed Sensing]{
  \includegraphics[width=0.3\textwidth]{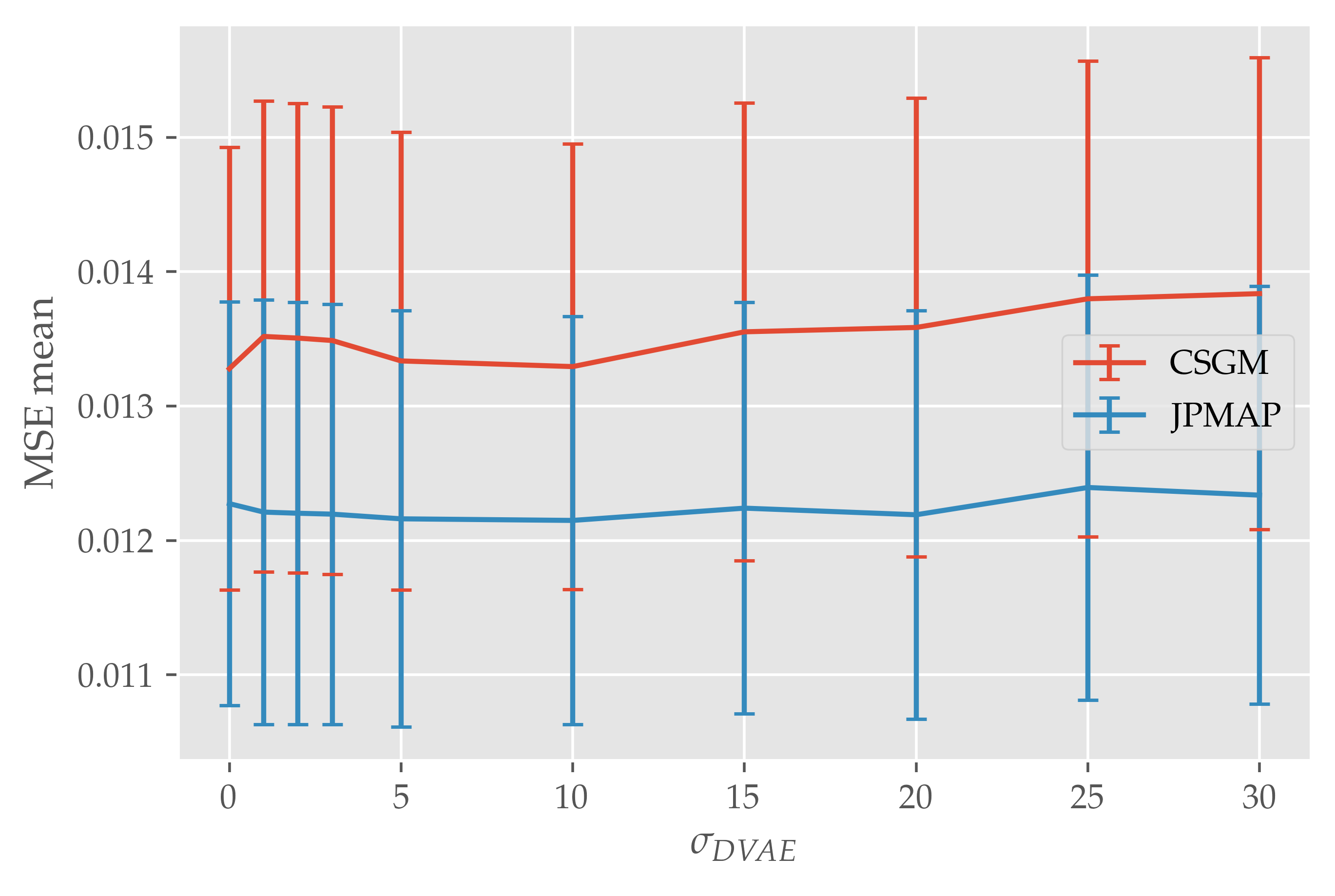}
  }
  \subfigure[Interpolation]{
  \includegraphics[width=0.3\textwidth]{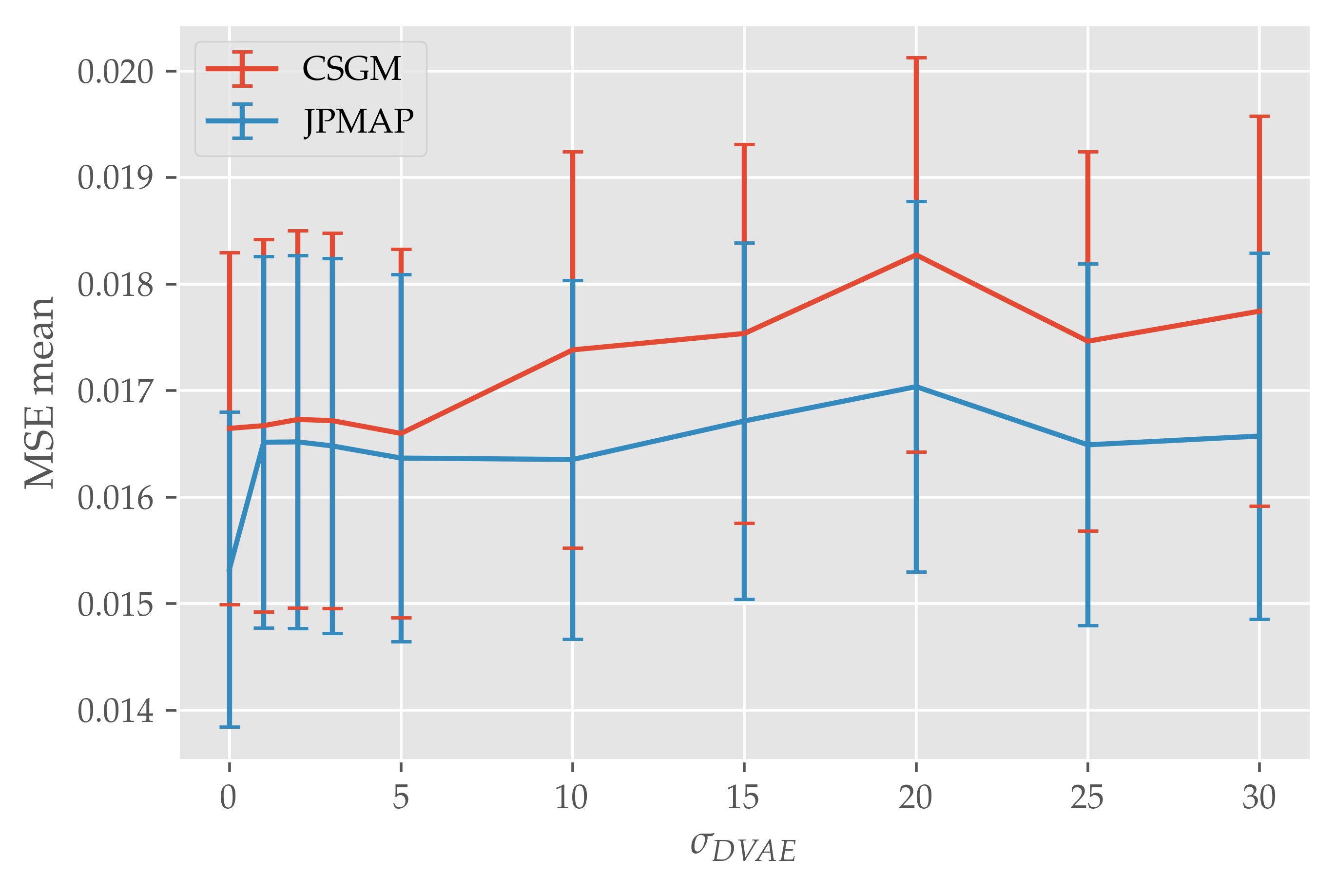}
  }
\caption{Evaluating the quality of the generative model as a function of $\sigmaDVAE$. On (a) Denoising (Gaussian noise $\sigma=150$), (b) Compressed Sensing ($\sim 10.2\%$ measurements, noise $\sigma=10$) and (c) Interpolation ($80\%$ of missing pixels, noise $\sigma=10$). Results of both algorithms are computed on a batch of 50 images and initialising on ground truth $\vx^*$ (for CSGM we use $\vz_0 = \muEncoder(\vx^*)$).}
\label{fig:DVAE-generative-quality}
\end{center}
\end{figure*}

\begin{figure*}[htbp]
\begin{center}
  \subfigure[Denoising]{
  \includegraphics[width=0.3\textwidth]{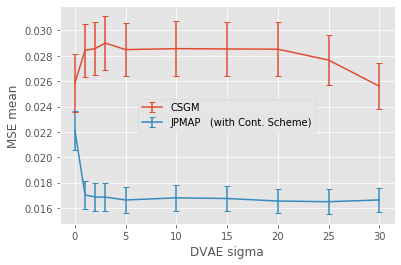}
  }
  \subfigure[Compressed Sensing]{
  \includegraphics[width=0.3\textwidth]{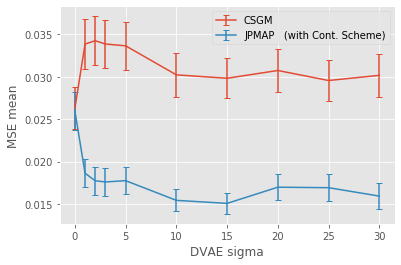}
  }
  \subfigure[Interpolation]{
  \includegraphics[width=0.3\textwidth]{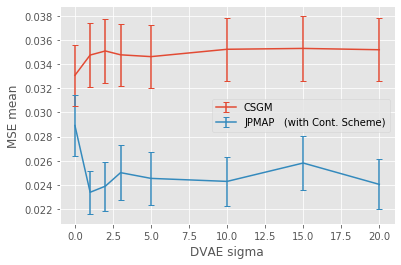}
  }
\caption{Evaluating the effectiveness of JPMAP vs CGSM as a function of $\sigmaDVAE$ (same setup of Figure~\ref{fig:DVAE-generative-quality}). Without a denoising criterion $\sigmaDVAE=0$ the JPMAP algorithm may provide wrong guesses $\z^{1}$ when applying the encoder in step 2 of Algorithm~\ref{alg:JPMAP2new}. For $\sigmaDVAE>0$ however, the alternating minimization algorithm can benefit from the robust initialization heuristics provided by the encoder, and it consistently converges to a better local optimum than the simple gradient descent in CSGM. }
\label{fig:DVAE-restoration}
\end{center}
\end{figure*}

\begin{figure*}
\centering
  \subfigure[Encoder approximation]{
  \label{fig:encoder_approximation}
  \parbox{0.22\textwidth}{\centering
  \includegraphics[height=0.19\textwidth]{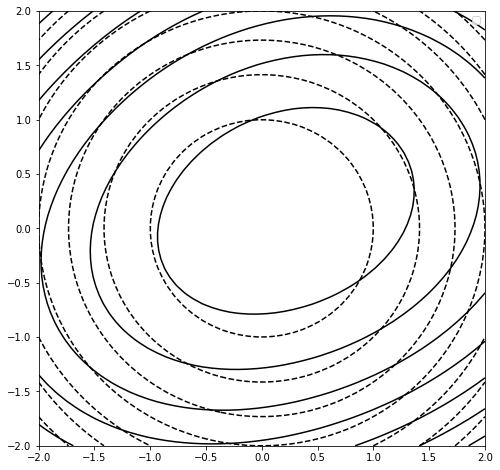}
  \includegraphics[height=0.19\textwidth]{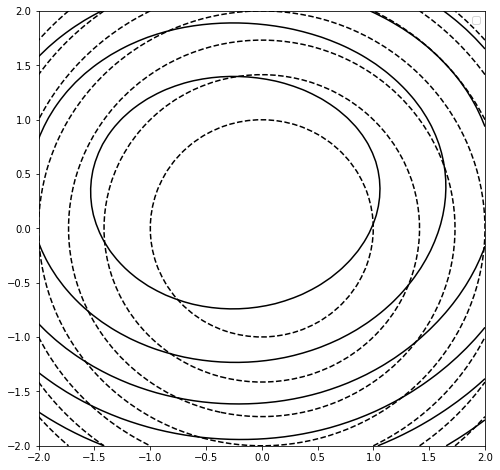}
  }
  }\hfill
  \subfigure[Decoded exact optimum]{
  \label{fig:decoded_exact_optimum}
  \parbox{0.22\textwidth}{\centering
  \includegraphics[height=0.19\textwidth]{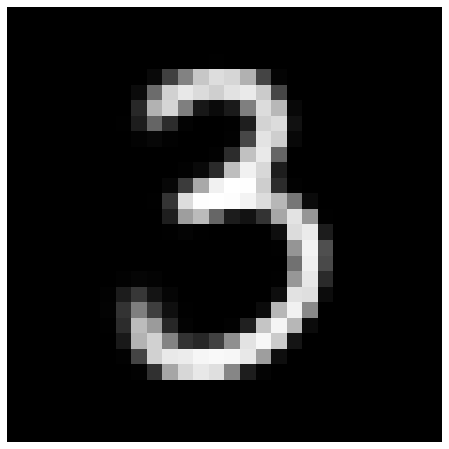}
  \includegraphics[height=0.19\textwidth]{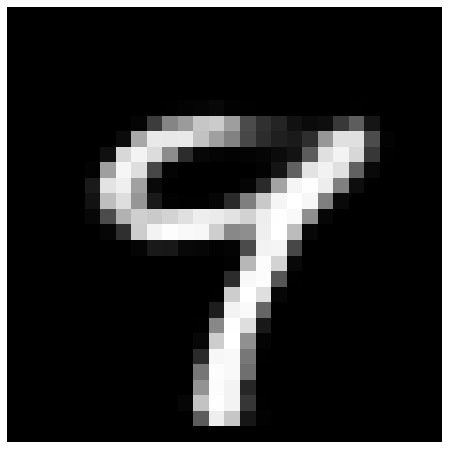}
  }
  }\hfill
  \subfigure[Decoded approx. optimum]{
  \label{fig:decoded_approximate_optimum}
  \parbox{0.22\textwidth}{\centering
  \includegraphics[height=0.19\textwidth]{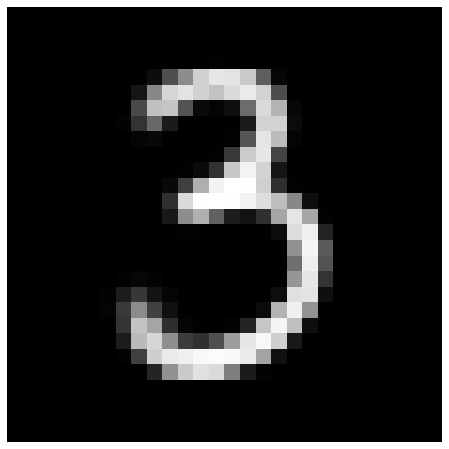}
  \includegraphics[height=0.19\textwidth]{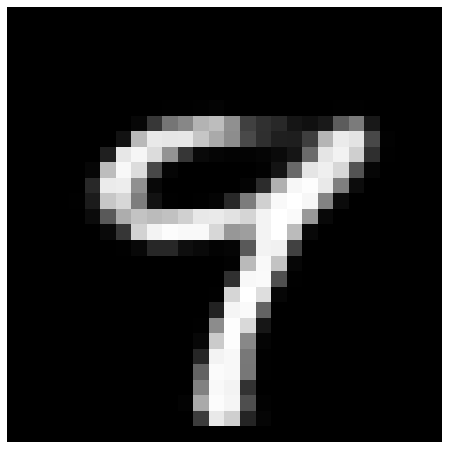}
  }
  }\hfill
  \subfigure[Difference (b)-(c)]{
  \label{fig:decoded_approximate_optimum2}
  \parbox{0.22\textwidth}{\centering
  \includegraphics[height=0.19\textwidth]{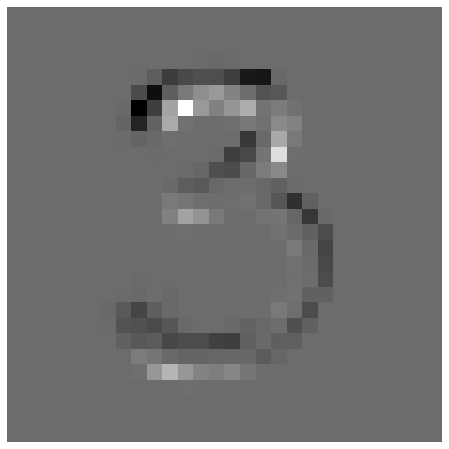}
  \includegraphics[height=0.19\textwidth]{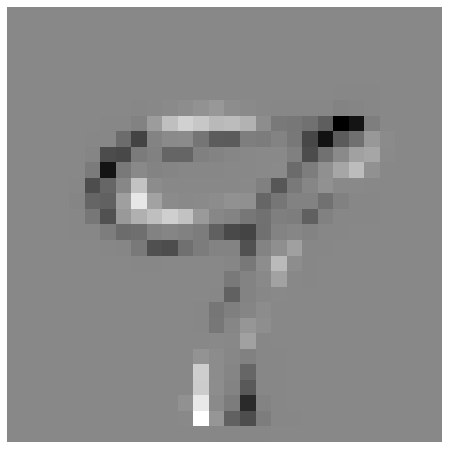}
  }
  }

  \caption{
  Encoder approximation: \emph{(a)}
  Contour plots of $-\log p_\theta(\vx|\vz) +\frac{1}{2}\|\vz\|^2$ and $-\log q_\phi(\vz|\vx)$ for a fixed $\vx$ and for a random 2D subspace in the $\vz$ domain (the plot shows $\pm 2 \SigmaEncoder^{1/2}$ around $\muEncoder$). Observe the relatively small gap between the true posterior $p_\theta(\vz|\vx)$ and its variational approximation $q_\phi(\vz|\vx)$. This figure shows some evidence of partial $\z$-convexity of $J_1$ around the minimum of $J_2$, but it does not show how far is $\z^1$ from $\z^2$.
  \emph{(b)}
  Decoded exact optimum $\vx_1 = \muDecoder\left( \arg\max_\vz p_\theta(\vx|\vz)e^{\frac{1}{2}\|\vz\|^2} \right)$.
  \emph{(c)}
  Decoded approximate optimum $\vx_2 = \muDecoder\left( \arg\max_\vz q_\phi(\vz|\vx) \right)$.
  \emph{(d)} Difference betweeen (b) and (c).
  }
  \label{fig:trained_VAE}
 \end{figure*}

 \begin{figure*}[tbh]
\begin{center}
  \subfigure[Energy evolution, initializing with $\mathcal{N}(0,I)$.]{
  \label{fig:assumption2A}
  \includegraphics[width=0.45\textwidth]{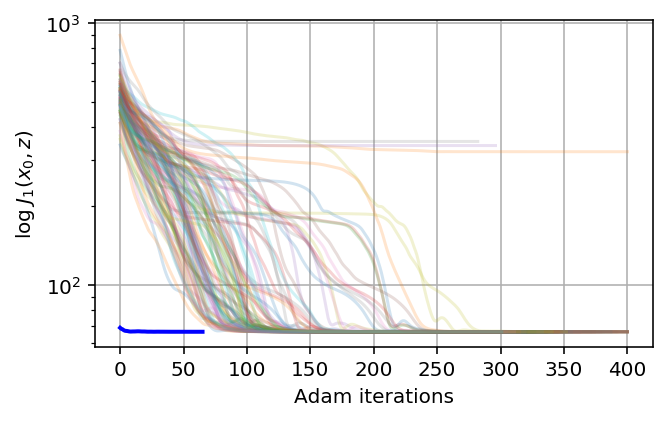}
  }
  \subfigure[Distance to the optimum at each iteration of (a).]{
  \label{fig:assumption2B}
  \includegraphics[width=0.45\textwidth]{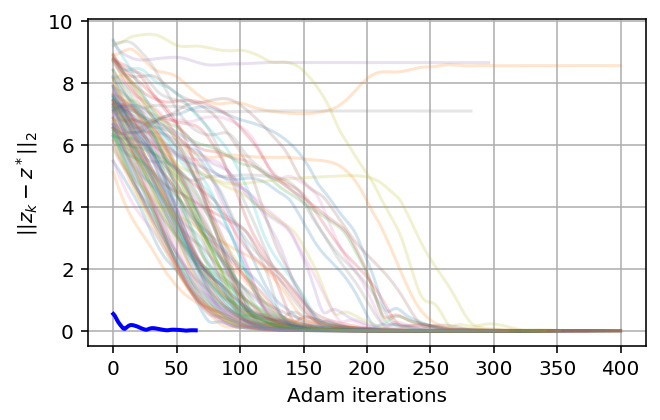}
  }
  \end{center}
  \caption{
  \emph{Effectiveness of the encoder approximation:} We take $\vx_0$ from the \emph{test} set of MNIST and minimize $J_1(\vx_0,\vz)$ with respect to $\vz$ using gradient descent from random Gaussian initializations $\vz_0$. The blue thick curve represents the trajectory if we initialize at the encoder approximation $\vz^1=\argmin_\vz J_2(\vx_0,\vz)=\muEncoder(\vx_0)$. \emph{(a)}: Plots of the energy iterates $J_1(\vx_0,\vz_k)$.
  \emph{(b)}: $\ell^2$ distances of each trajectory with respect to the global optimum $\vz^*$.
  \emph{Conclusion:} Observe that the encoder initialization allows much faster convergence both in energy and in $\vz$, and it avoids the few random initializations that lead to a wrong stationary point different from the unique global minimizer.
  }
  \label{fig:encoder-init}
 \end{figure*}

 \subsection{Need to train the VAE with a denoising criterion}\label{sec:denoising-criterion}

It should be noted that when training our Variational Autoencoder we should be more careful than usual. Indeed in the most widespread applications of VAEs they are only used as a generative model or as a way to interpolate between images that are close to $\mathcal{M}$, \emph{i.e.} the image of the generator $\muDecoder$.
For such applications it is sufficient to train the encoder $\muEncoder,\,\SigmaEncoder$ on a training set that is restricted to $\mathcal{M}$.\\

In our case however, we need the encoder to provide sensible values even when its input $\x$ is quite far away from $\mathcal{M}$: the encoder has to actually fulfill two functions at the same time:
\begin{enumerate}
\item (Approximately) project $\x$ to its closest point in $\mathcal{M}$, and
\item compute the encoding of this projected value (which should be the same as the encoding of the original $\x$.
\end{enumerate}

Traditional VAE training procedures do not ensure that the encoder generalizes well to $\x \not\in \mathcal{M}$. In order to ensure this generalization ability we adopt the training procedure of the DVAE (Denoising VAE) proposed by~\citet{Im2017}, which consists in adding various realizations of zero-mean Gaussian noise of variance $\sigmaDVAE^{2}$ to the samples $\x$ presented to the encoder, while still requiring the decoder to match the noiseless value, \emph{i.e.} we optimize the parameters in such a way that
\begin{equation}
    \muDecoder(\muEncoder(\tilde \x)) \approx \x
    \label{eq:denoising_criterion}
\end{equation}
where $\tilde\x = \x + \sigmaDVAE \varepsilon$ and $\varepsilon \sim \mathcal{N}(0,I)$ for all $\x$ in the training set and for many realizations of $\varepsilon$.\\

More specifically, if we take a corruption model $p(\tilde \vx\,|\,\vx)$ like above, it can be shown~\cite{Im2017} that
\begin{equation}
 \mathcal{\tilde L}_{\theta,\phi}(\vx) = \mathbb{E}_{p(\tilde \vx|\vx)} \left[\mathbb{E}_{\qphi(\vz|\tilde \vx)} [\log \ptheta(\vx|\vz)] - KL(\qphi(\vz|\tilde \vx)\;||\; \PDF{\Z}{\vz})\right]
 \label{eq:dvae_loss}
\end{equation}
is an alternative ELBO of \eqref{eq:vae_loss}. In practice, using Monte Carlo for estimating the expectation $\mathbb{E}_{p(\tilde \vx|\vx)}$ in \eqref{eq:dvae_loss}, we only need to add noise to $\x$ before passing it to the encoder $q_\phi$ during training, as mentioned in \eqref{eq:denoising_criterion}.

Our experiments with this denoising criterion confirm the observation by~\citet{Im2017} that it does not degrade the quality of the generative model, as long as $\sigmaDVAE$ is not too large (see Figure~\ref{fig:DVAE-generative-quality}). As a side benefit, however, we obtain a more robust encoder that generalizes well for values of $\x$ that are not in $\mathcal{M}$ but within a neighbourhood of size $\approx \sigmaDVAE$ around $\mathcal{M}$. This side benefit, which was not the original intention of the DVAE training algorithm in \citep{Im2017} is nevertheless crucial for the success of our algorithm as demonstrated in Figure~\ref{fig:DVAE-restoration}.
The same figure shows that as long as $\sigmaDVAE\geq 5$ its value does not significantly affect the performance. In the sequel we use $\sigmaDVAE=15$.

\begin{figure*}[htbp]
\begin{center}
  \subfigure{
  \includegraphics[width=0.9\textwidth]{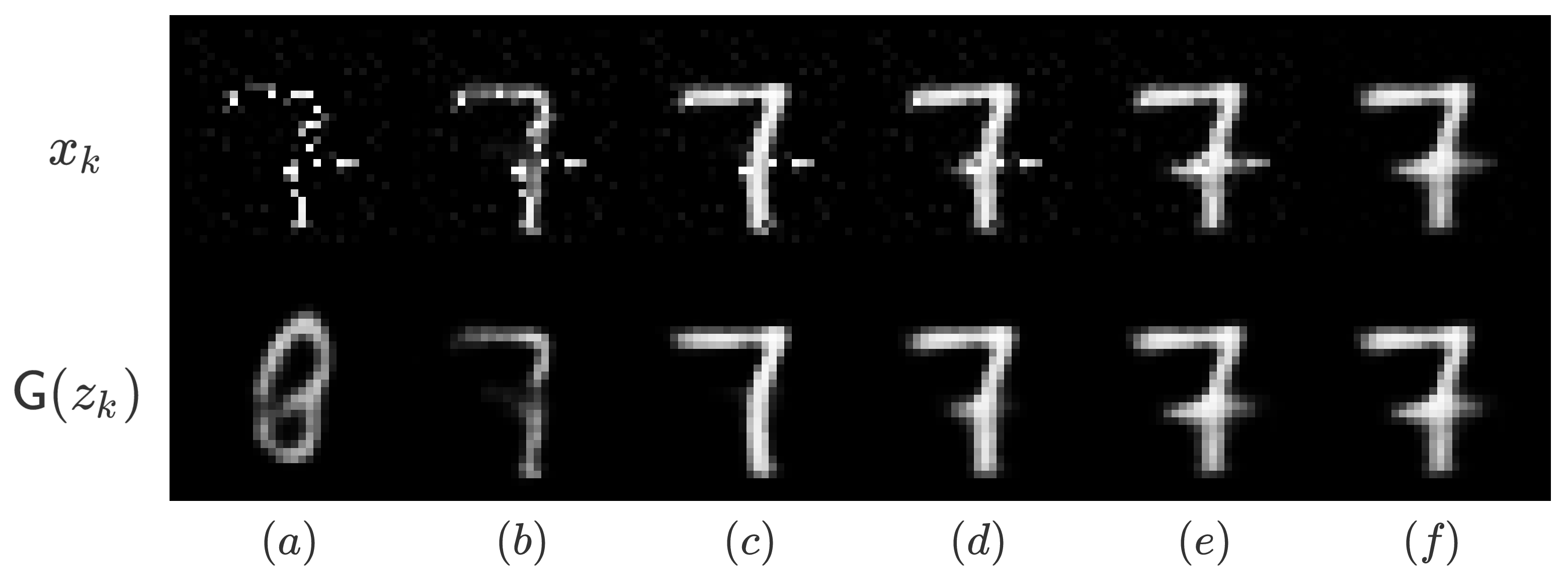}
  }\\
  \subfigure{%
  \includegraphics[width=0.45\textwidth]{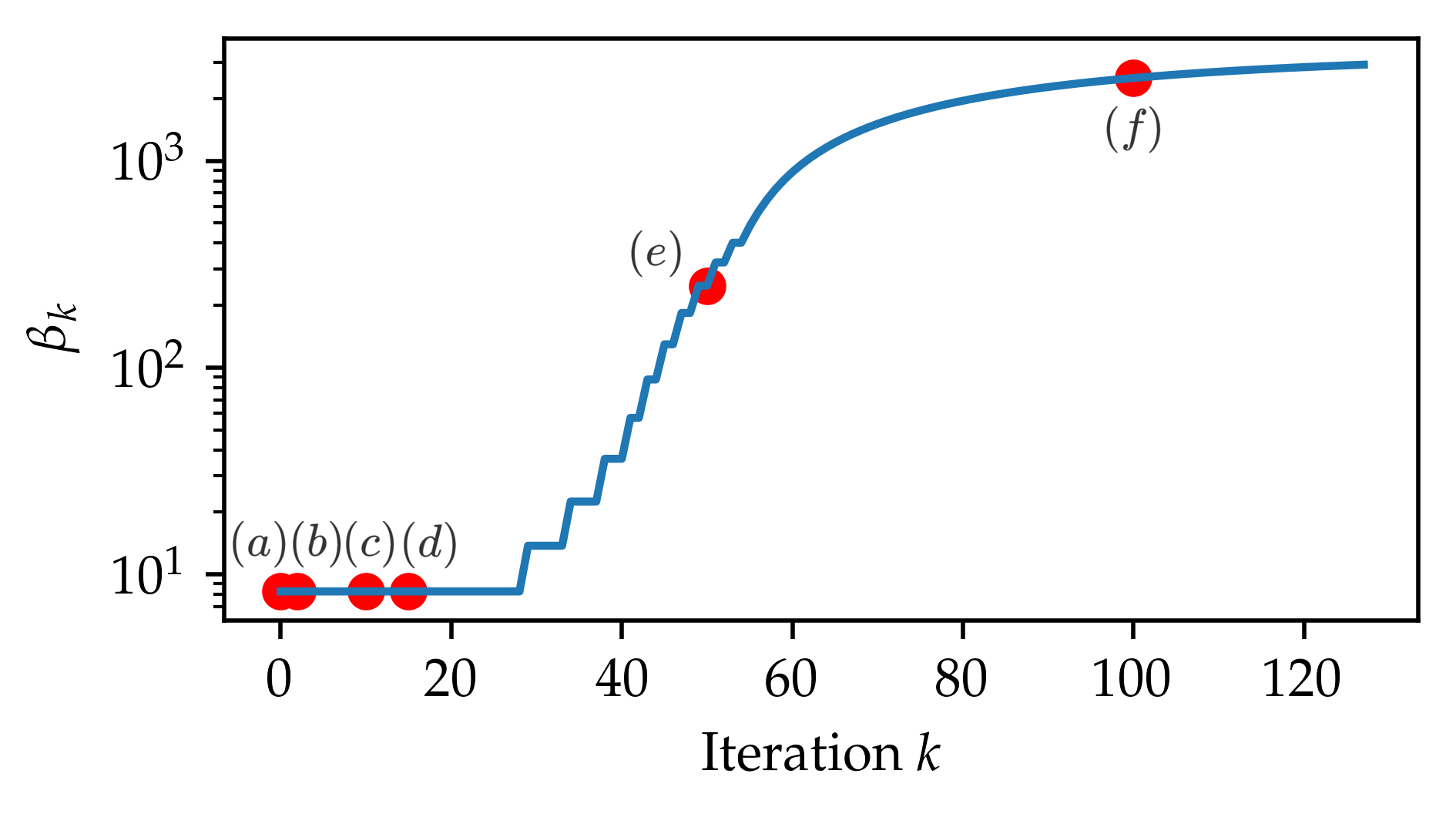}
  }
  \subfigure{%
  \includegraphics[width=0.45\textwidth]{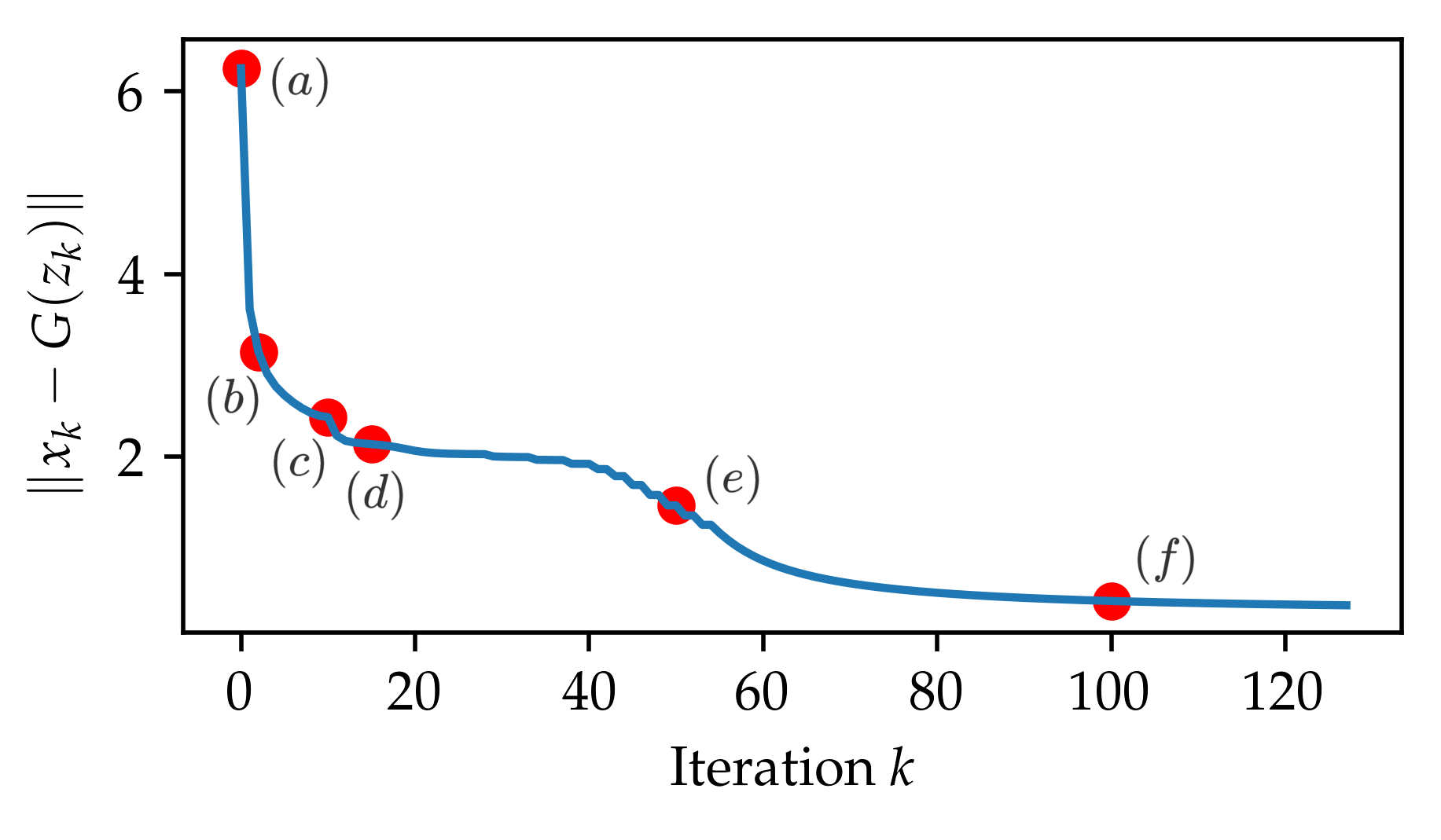}
  }
\caption{\emph{Evolution of Algorithm~\ref{alg:MAPzCS}.} In this interpolation example, JPMAP starts with the initialization in $(a)$. During first iterations $(b)-(d)$ where $\beta_k$ is small, $\vx_k$ and $\generator(\vz_k)$ start loosely approaching each other at a coarse scale, and $\vx_k$ only fills missing pixels with the ones of $\generator(\vz_k)$ (in particular the noise of $\y$ is still present). By increasing $\beta_k$ in $(e)-(f)$ we enforce $\|\generator(\z_k)-\x_k\|^2 \leq \varepsilon$. Here we set $\epsilon=\left(\frac{3}{255}\right)^2 d$, that is, MSE of 3 gray levels.
}
\label{fig:evolution_JPMAP}
\end{center}
\end{figure*}

\begin{figure*}[htbp]
\begin{center}
  \subfigure[Denoising (PSNR)]{
  \includegraphics[width=0.45\textwidth]{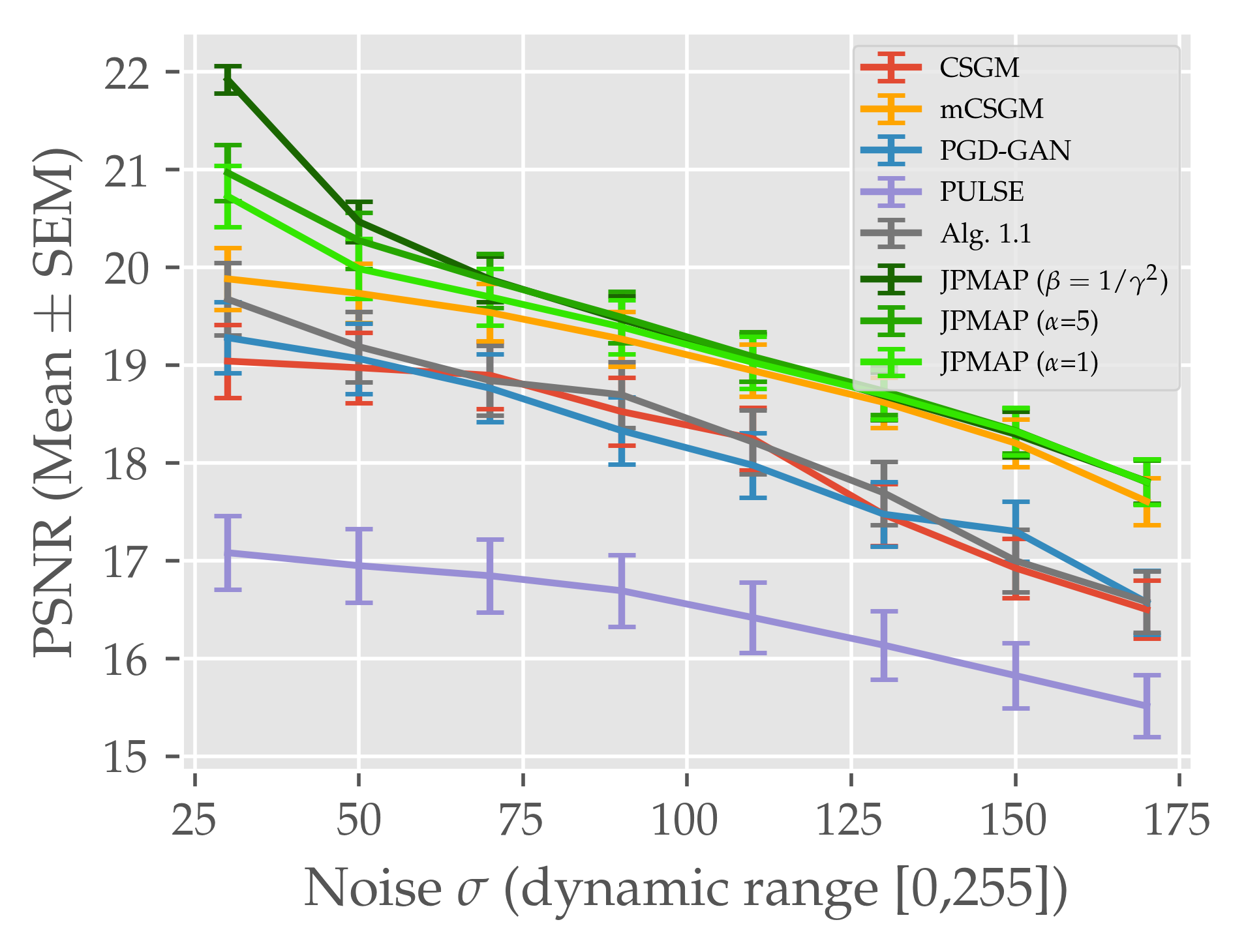}
  }
  \subfigure[Denoising (LPIPS)]{
  \includegraphics[width=0.45\textwidth]{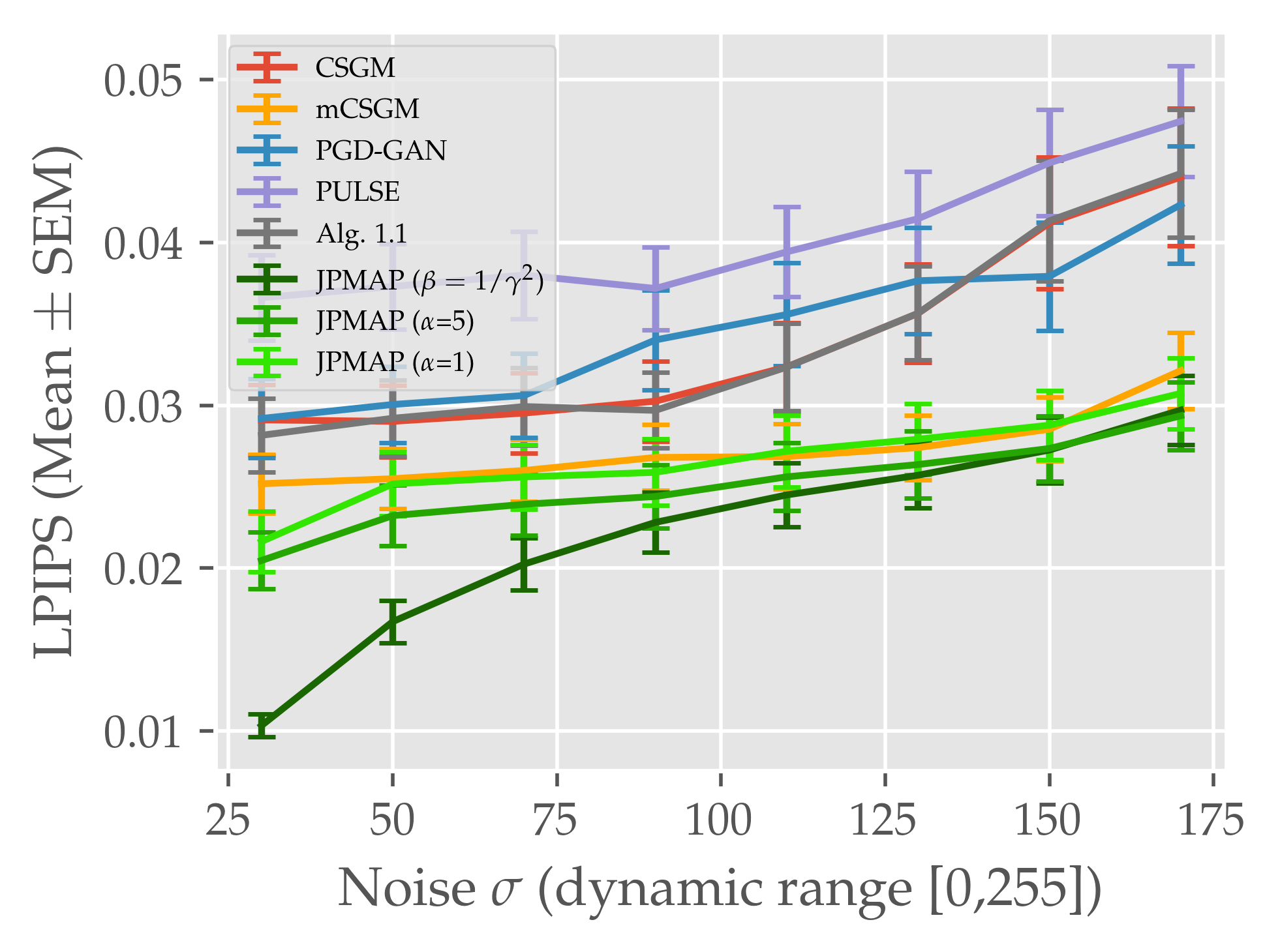}
  }\\
  \subfigure[Compressed Sensing (PSNR)]{
  \includegraphics[width=0.45\textwidth]{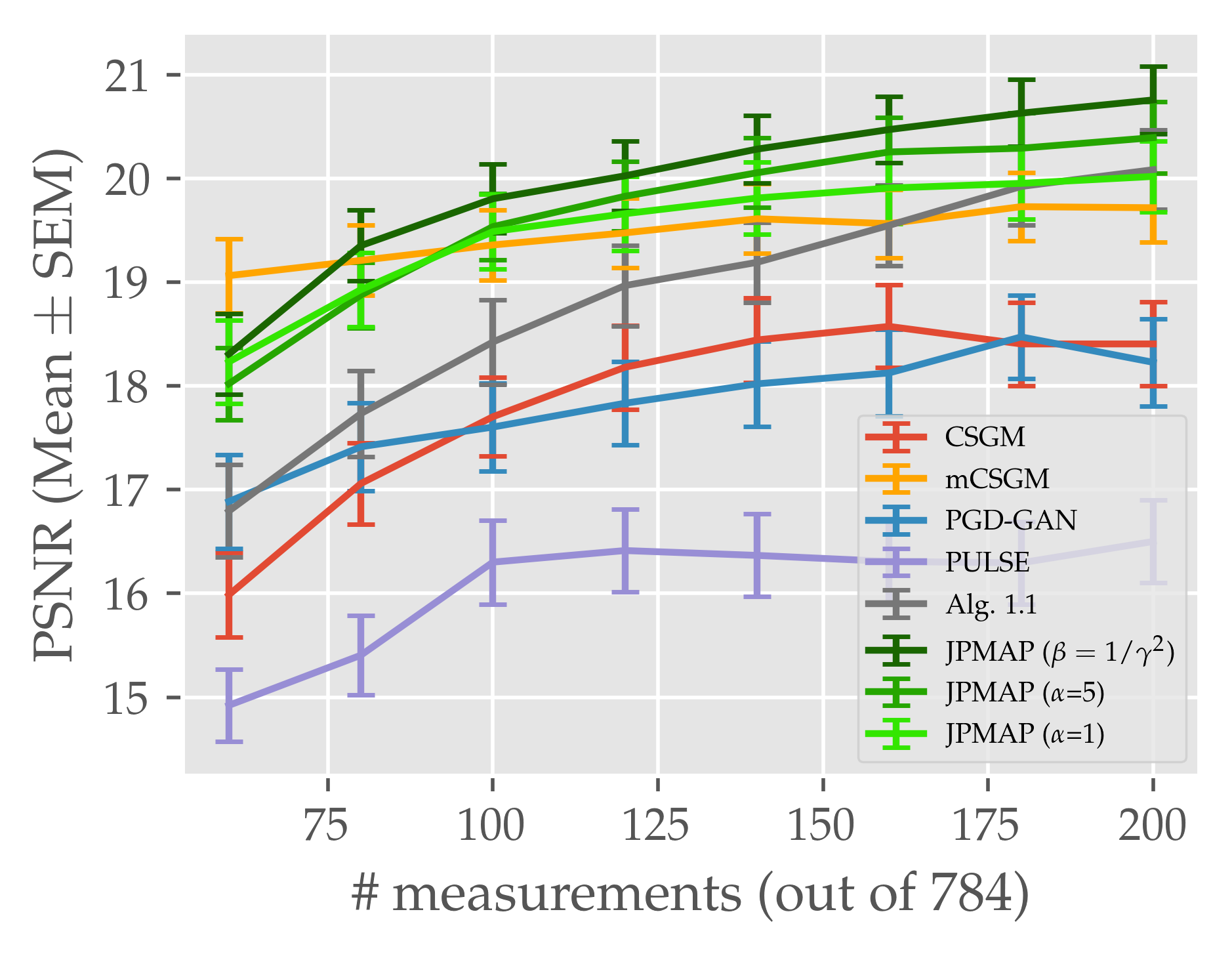}
  }
  \subfigure[Compressed Sensing (LPIPS)]{
  \includegraphics[width=0.45\textwidth]{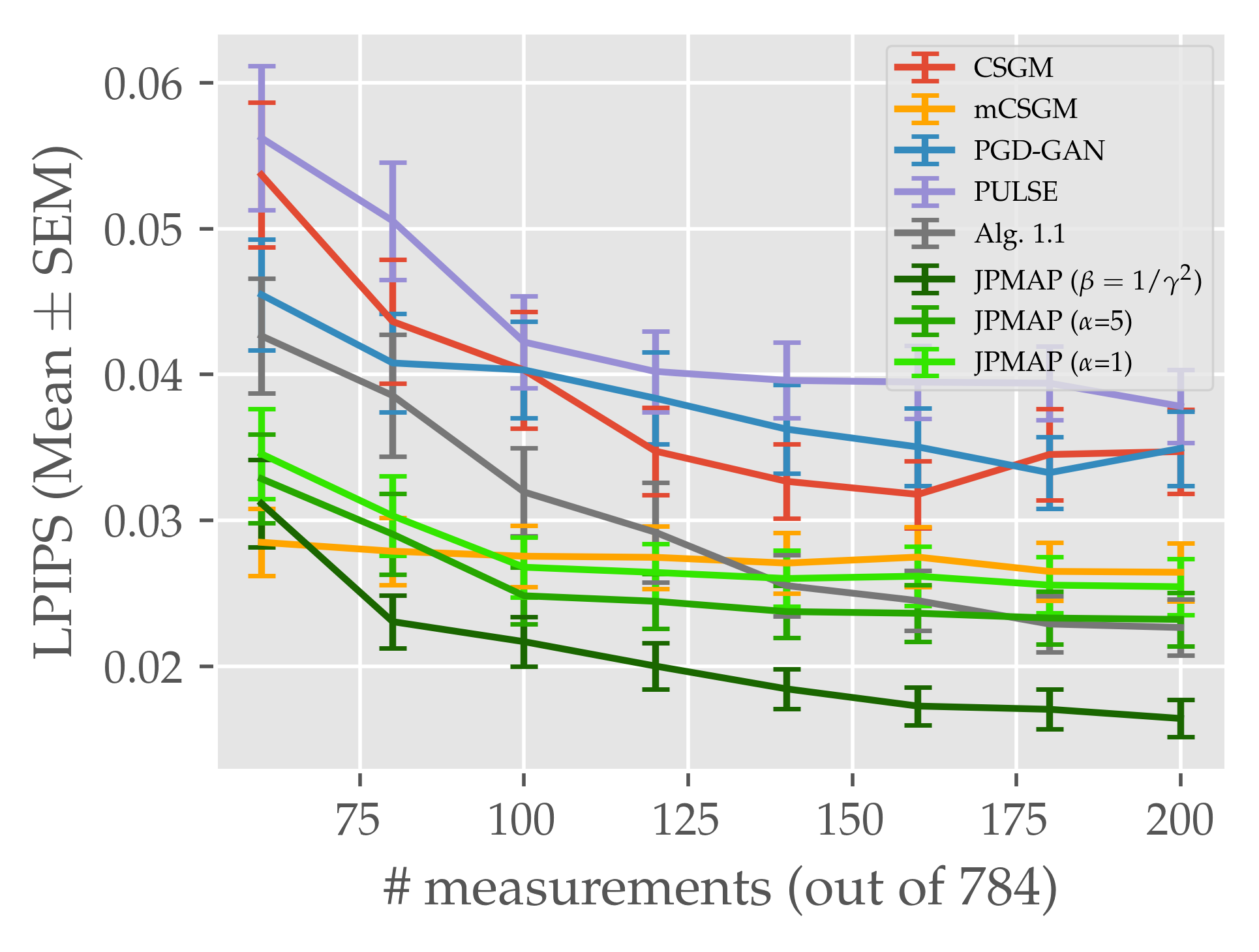}
  }\\
  \subfigure[Interpolation (PSNR)]{
  \includegraphics[width=0.45\textwidth]{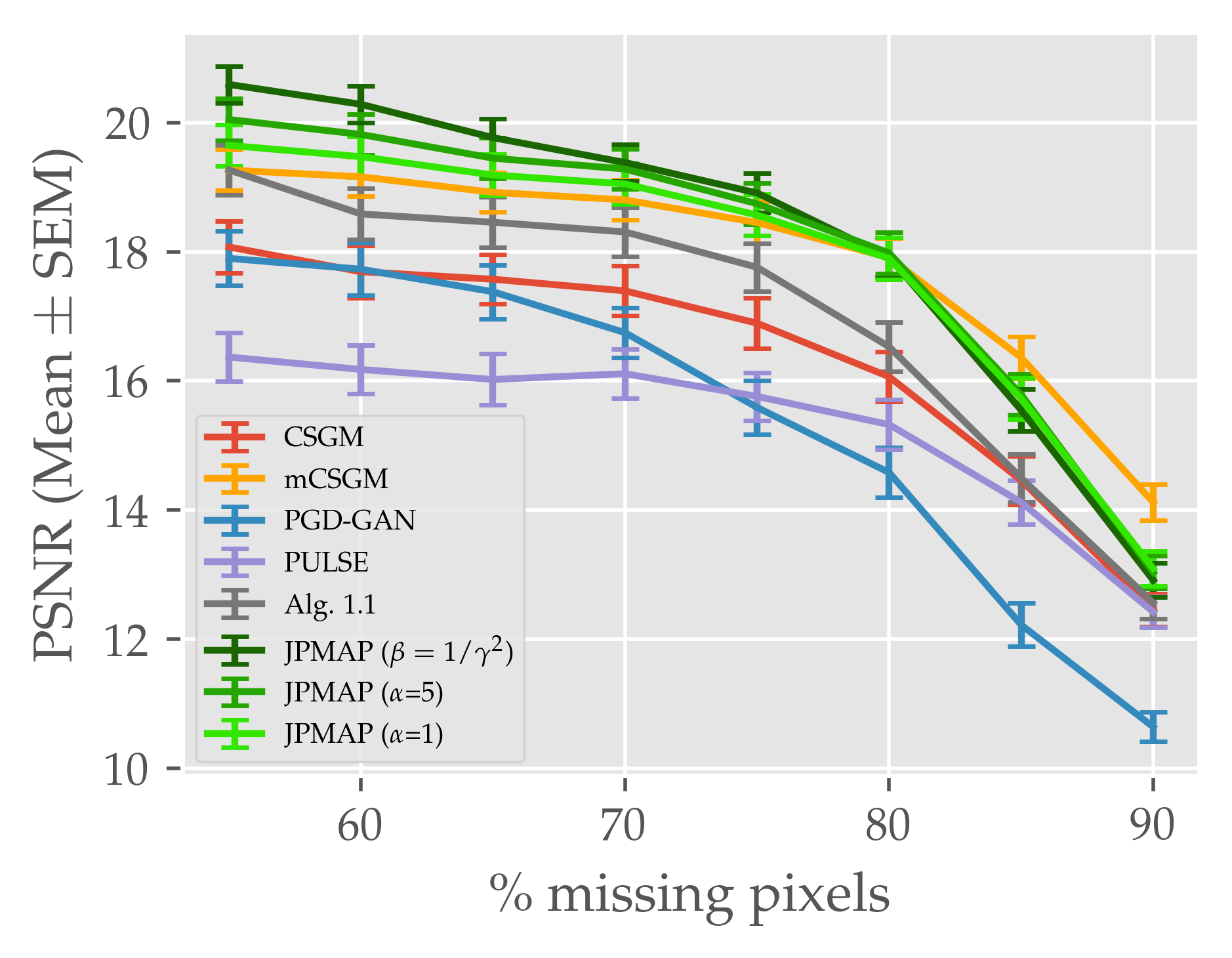}
  }
  \subfigure[Interpolation (LPIPS)]{
  \includegraphics[width=0.45\textwidth]{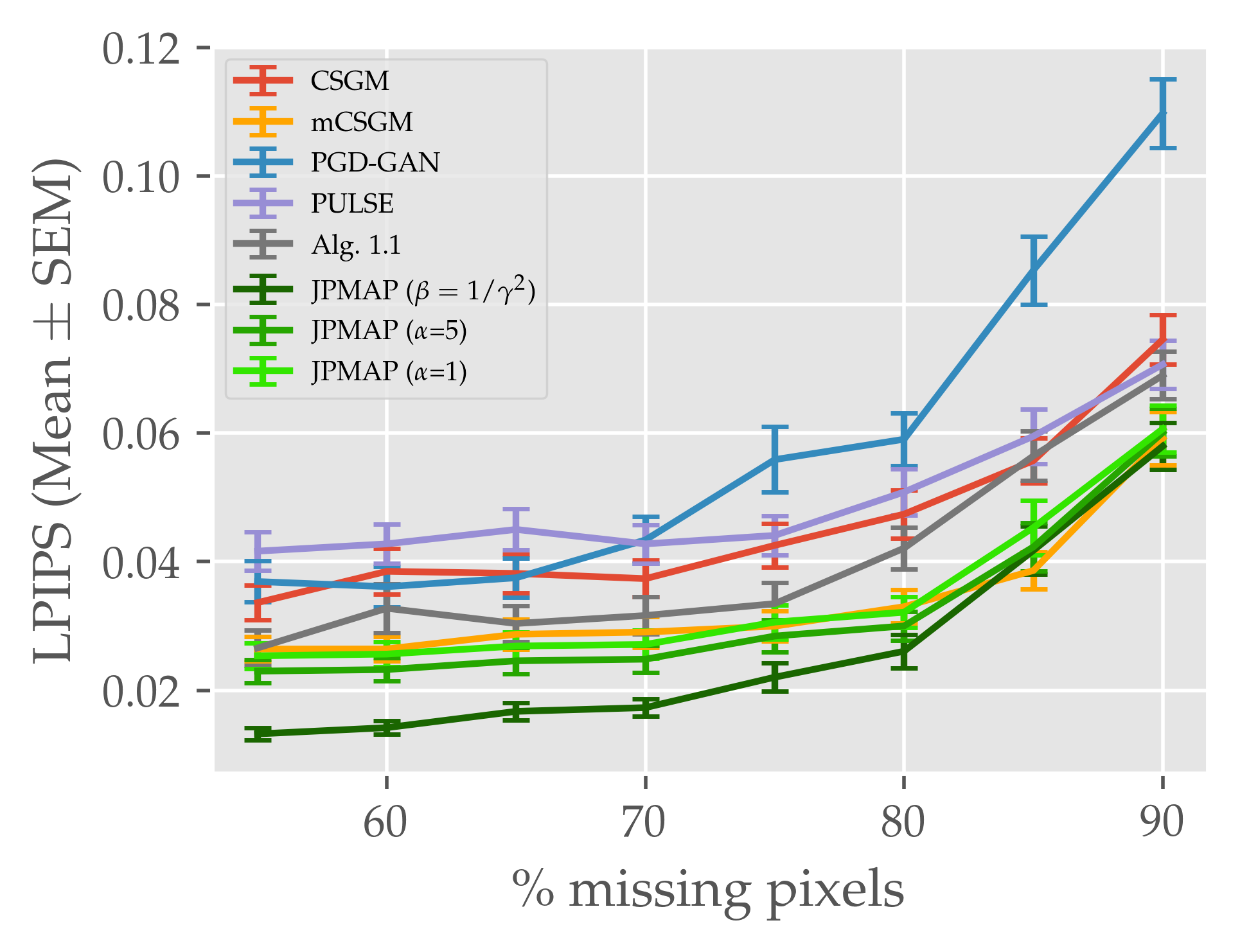}
  }
\caption{\emph{Denoising, Compressed Sensing and Interpolation}: Evaluating the effectiveness of Algorithm \ref{alg:JPMAP3new} (fixed $\beta$) and Algorithm  \ref{alg:MAPzCS} for different values of $\epsilon=\left(\frac{\alpha}{255}\right)^2 n$, with $\sigmaDVAE = 15$ (metrics were computed on a batch of 100 test images). For PSNR, higher is better and for LPIPS, lower is better. For comparison we provide the results of the baselines introduced in Section~\ref{sec:baseline_algorithms} (namely, Algorithm~\ref{alg:MAPz-splitting}, CSGM \cite{bora2017compressed}, mCSGM (CSGM with restarts), PGD-GAN \cite{shah2018solving} and PULSE \citep{Menon2020}.)}
\label{fig:DVAE-opt-restoration1}
\end{center}
\end{figure*}

\begin{figure*}[htbp]
\begin{center}
  \subfigure[Deblurring (PSNR)]{
  \includegraphics[width=0.45\textwidth]{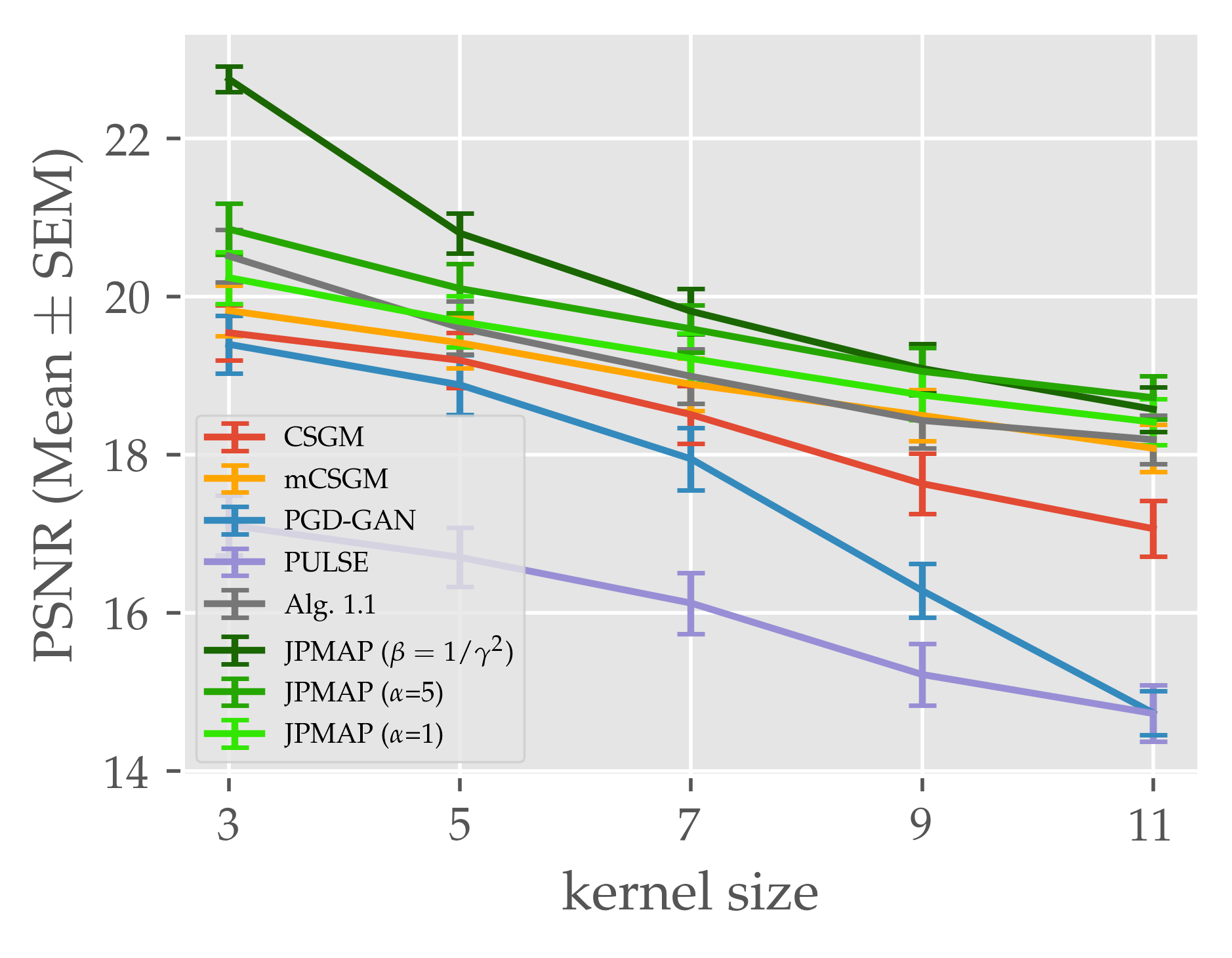}
  }
  \subfigure[Deblurring (LPIPS)]{
  \includegraphics[width=0.45\textwidth]{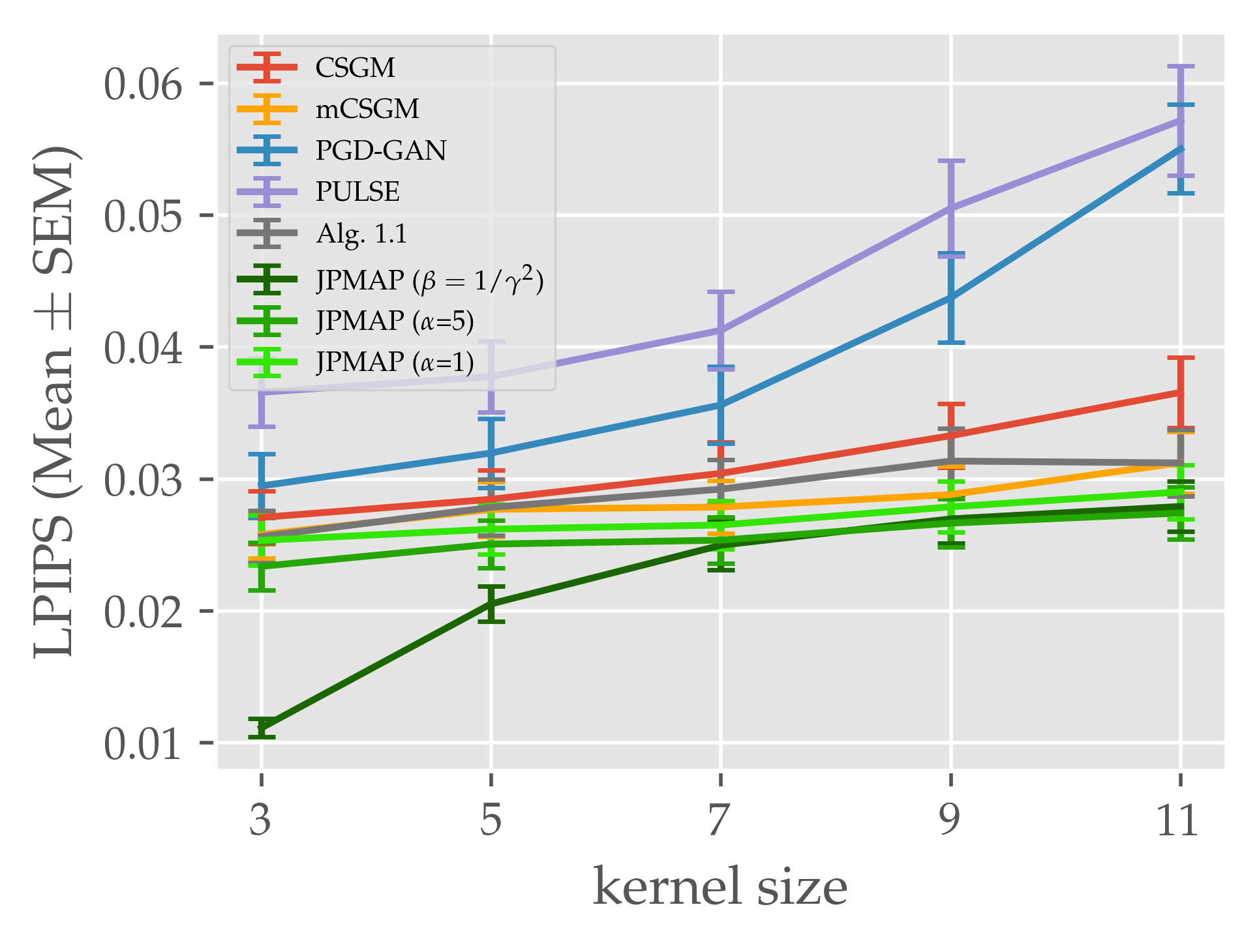}
  }
  \subfigure[Super-resolution (PSNR)]{
  \includegraphics[width=0.45\textwidth]{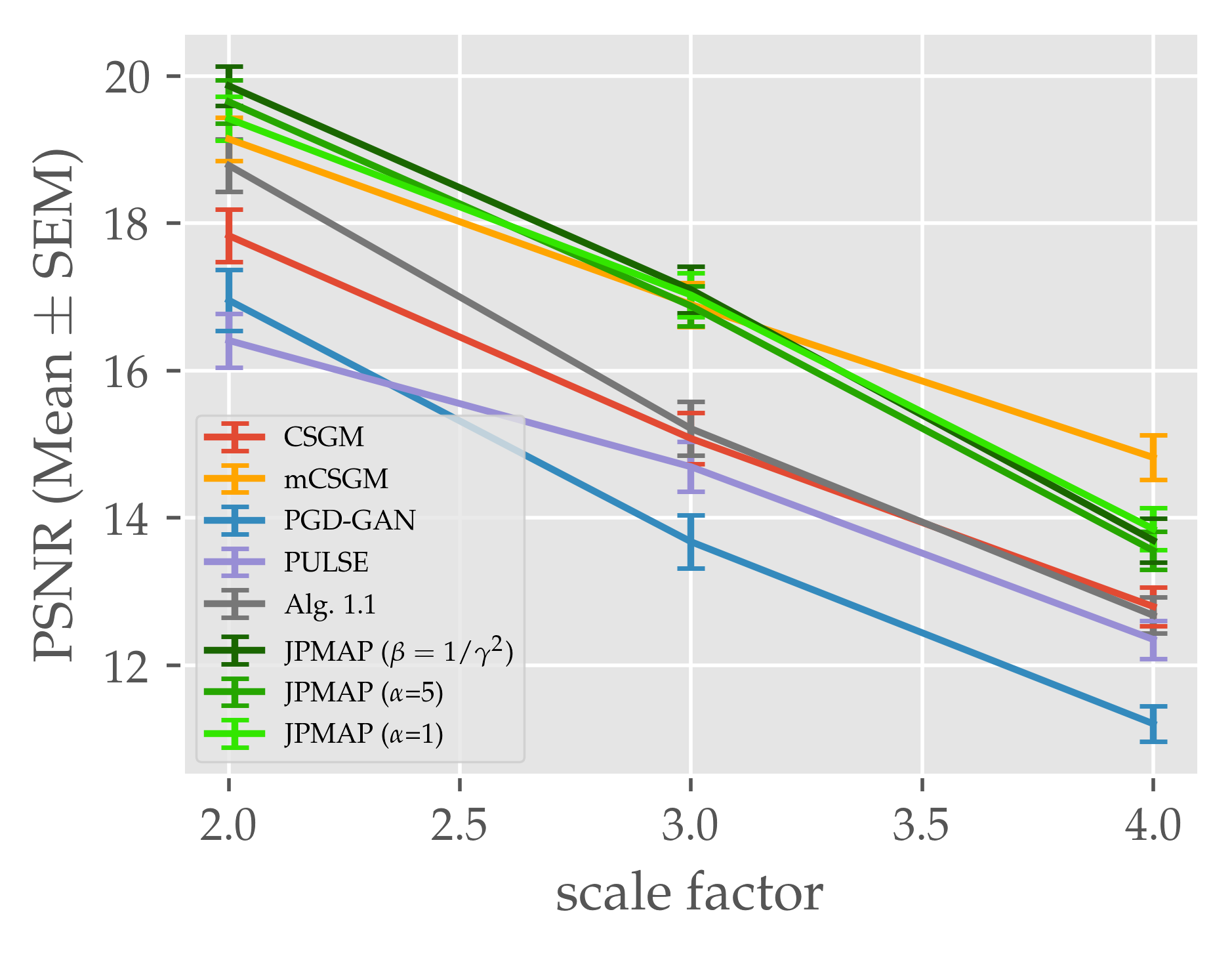}
  }
  \subfigure[Super-resolution (LPIPS)]{
  \includegraphics[width=0.45\textwidth]{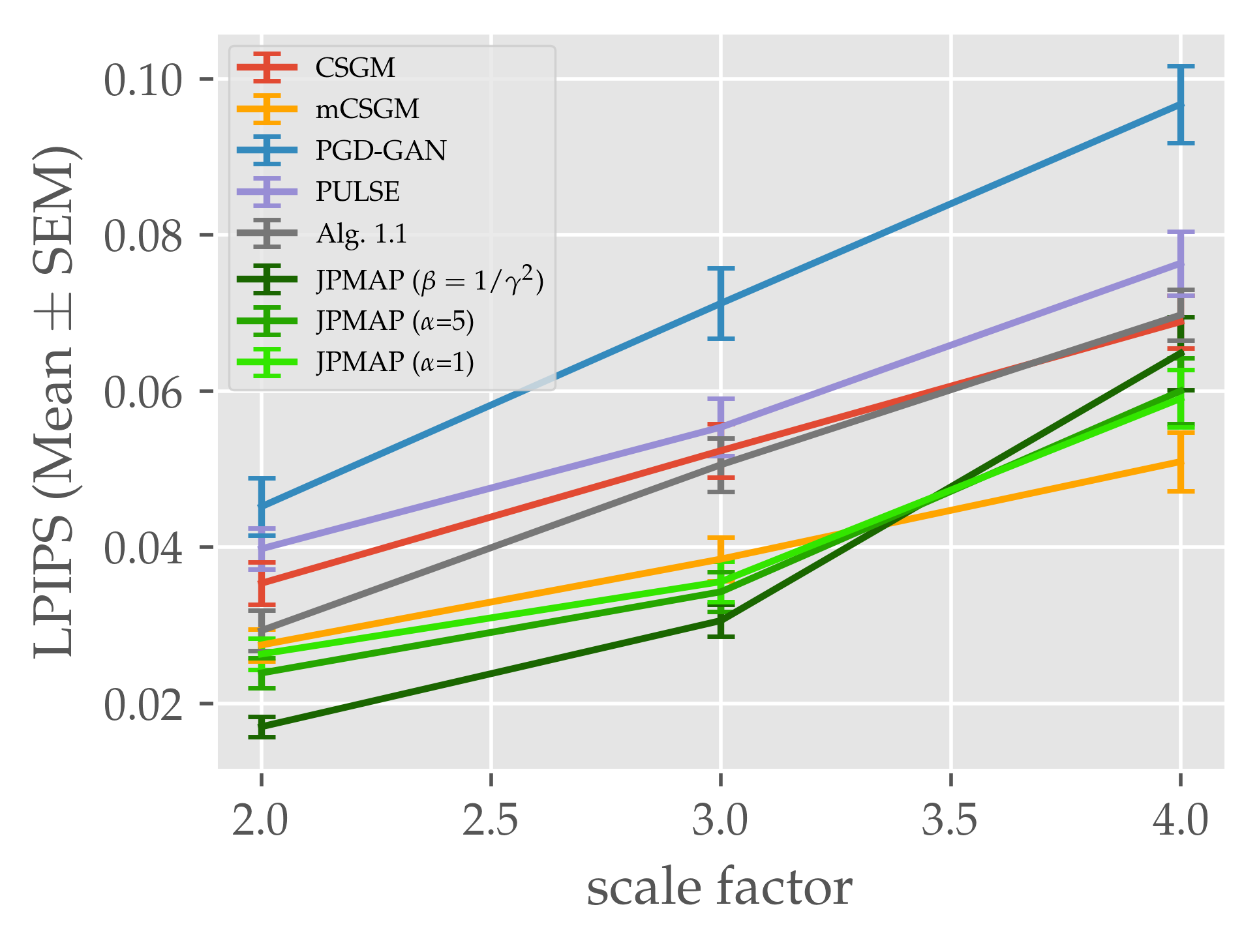}
  }
\caption{\emph{Deblurring and Super-resolution}: Evaluating the effectiveness of Algorithm \ref{alg:JPMAP3new} (fixed $\beta$) and Algorithm  \ref{alg:MAPzCS} for different values of $\epsilon=\left(\frac{\alpha}{255}\right)^2 n$, with $\sigmaDVAE = 15$ (metrics were computed on a batch of 100 test images). For PSNR, higher is better and for LPIPS, lower is better. For comparison we provide the results of the baselines introduced in Section~\ref{sec:baseline_algorithms} (namely, Algorithm~\ref{alg:MAPz-splitting}, CSGM \cite{bora2017compressed}, mCSGM (CSGM with restarts), PGD-GAN \cite{shah2018solving} and PULSE \citep{Menon2020}.)}
\label{fig:DVAE-opt-restoration2}
\end{center}
\end{figure*}

\begin{figure*}[htbp]
\begin{center}
  \subfigure[Interpolation ($p=75$)]{
  \includegraphics[width=0.45\textwidth]{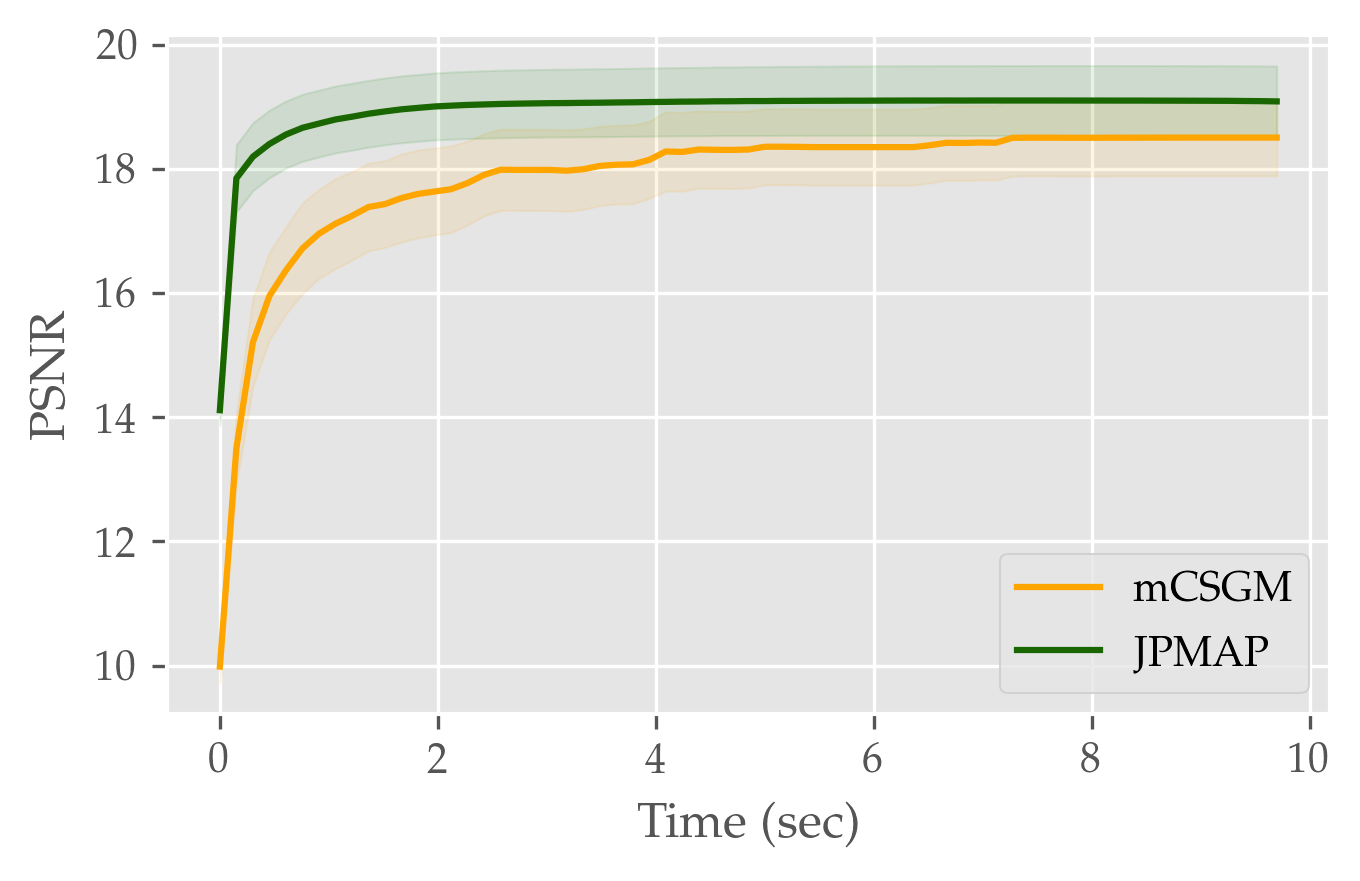}
  }
  \subfigure[Example run]{
  \includegraphics[width=0.45\textwidth]{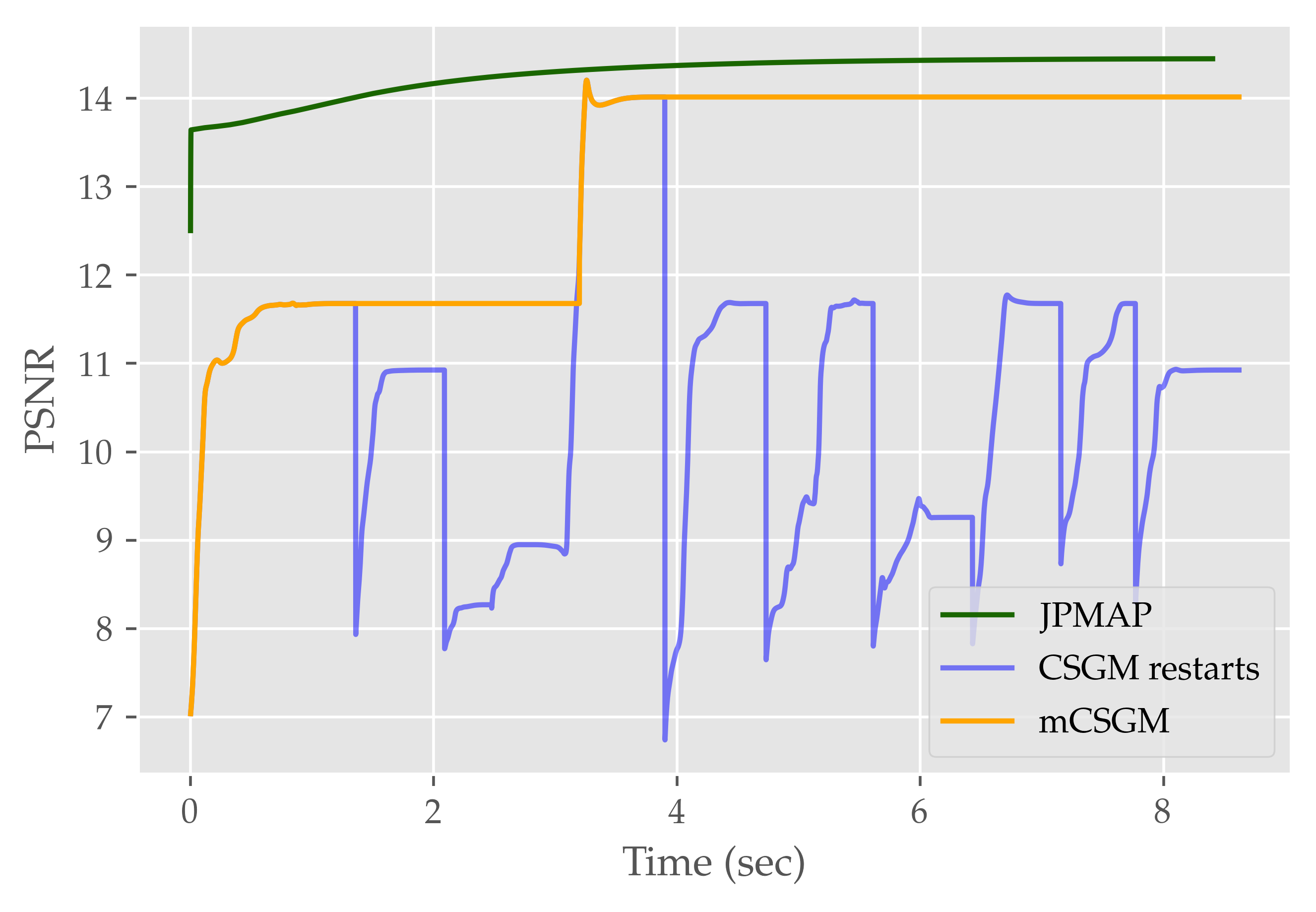}
  }
  \subfigure[Interpolation ($p=90$)]{
  \includegraphics[width=0.45\textwidth]{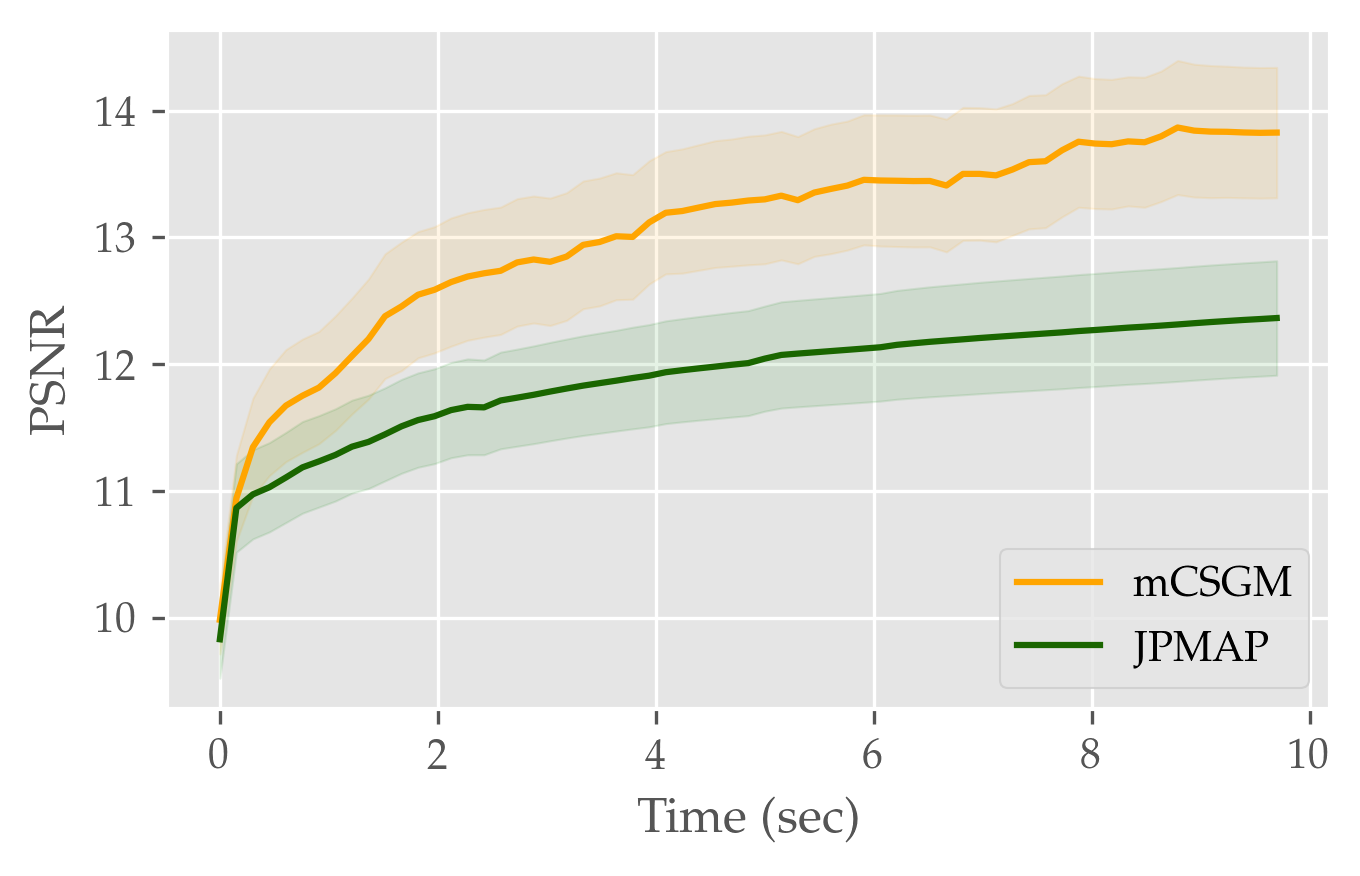}
  }
  \subfigure[Example run]{
  \includegraphics[width=0.45\textwidth]{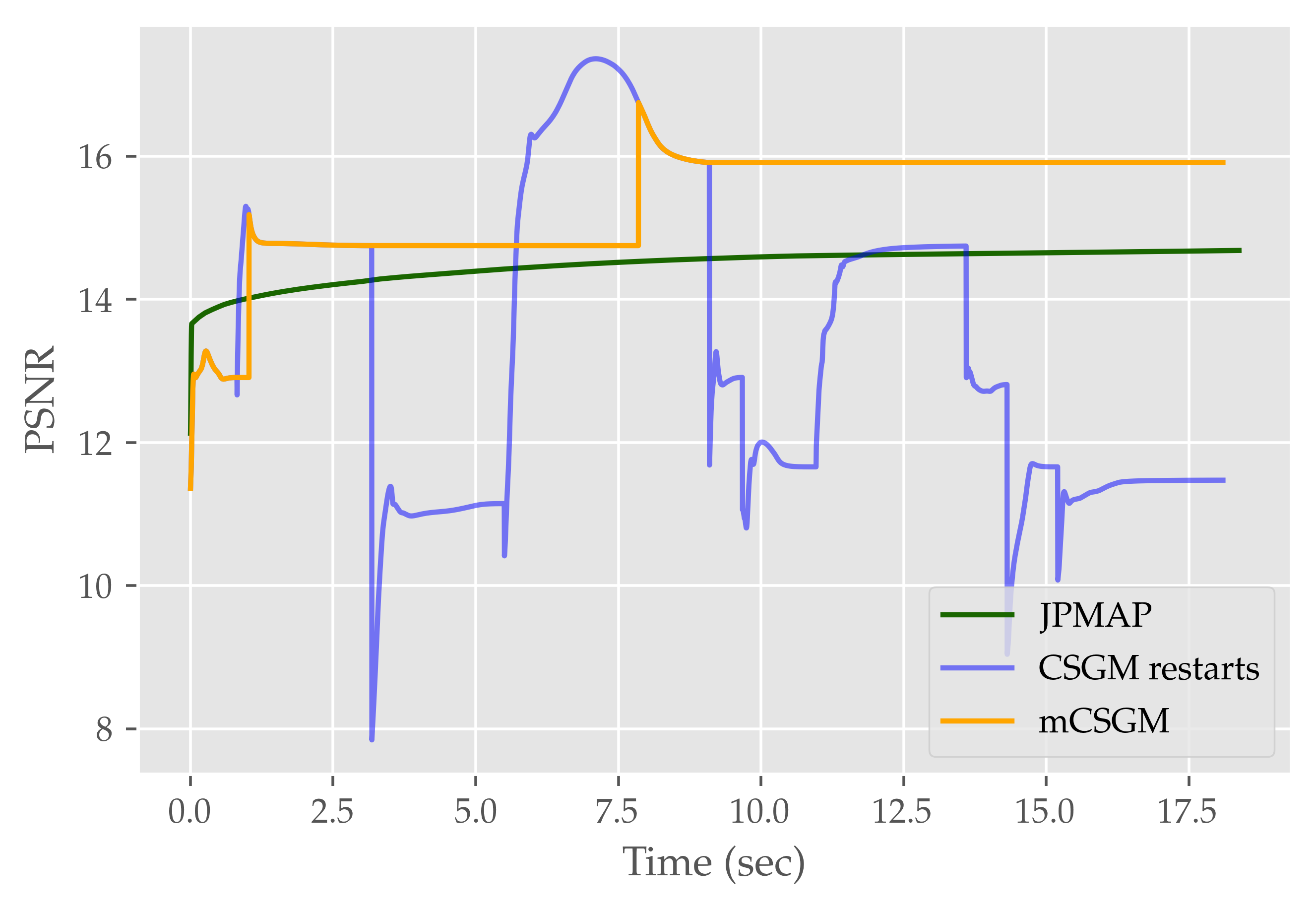}
  }
\caption{\emph{Time/PSNR comparison between mCSGM and JPMAP.} \emph{Left}: Confidence intervals (for a batch of 100 random experiments) for PSNR vs computing time for both algorithms on the interpolation problem with $p\%$ of missing pixels with noise std $\sigma=10/255$. \emph{Right}: Detailed view of one of the 100 random experiments on the left. Blue lines represent $m=10$ random restarts of CSGM and the orange line is the PSNR of the best $\vz_k$ at iteration $k$ of mCSGM as measured by~\eqref{eq:MAPz}.}
\label{fig:timePSNR_interpolation}
\end{center}
\end{figure*}

\begin{figure*}[htbp]
\centering
  \subfigure[Results on interpolation]{
  \label{fig:results_interpolation}
  \includegraphics[width=0.75\textwidth]{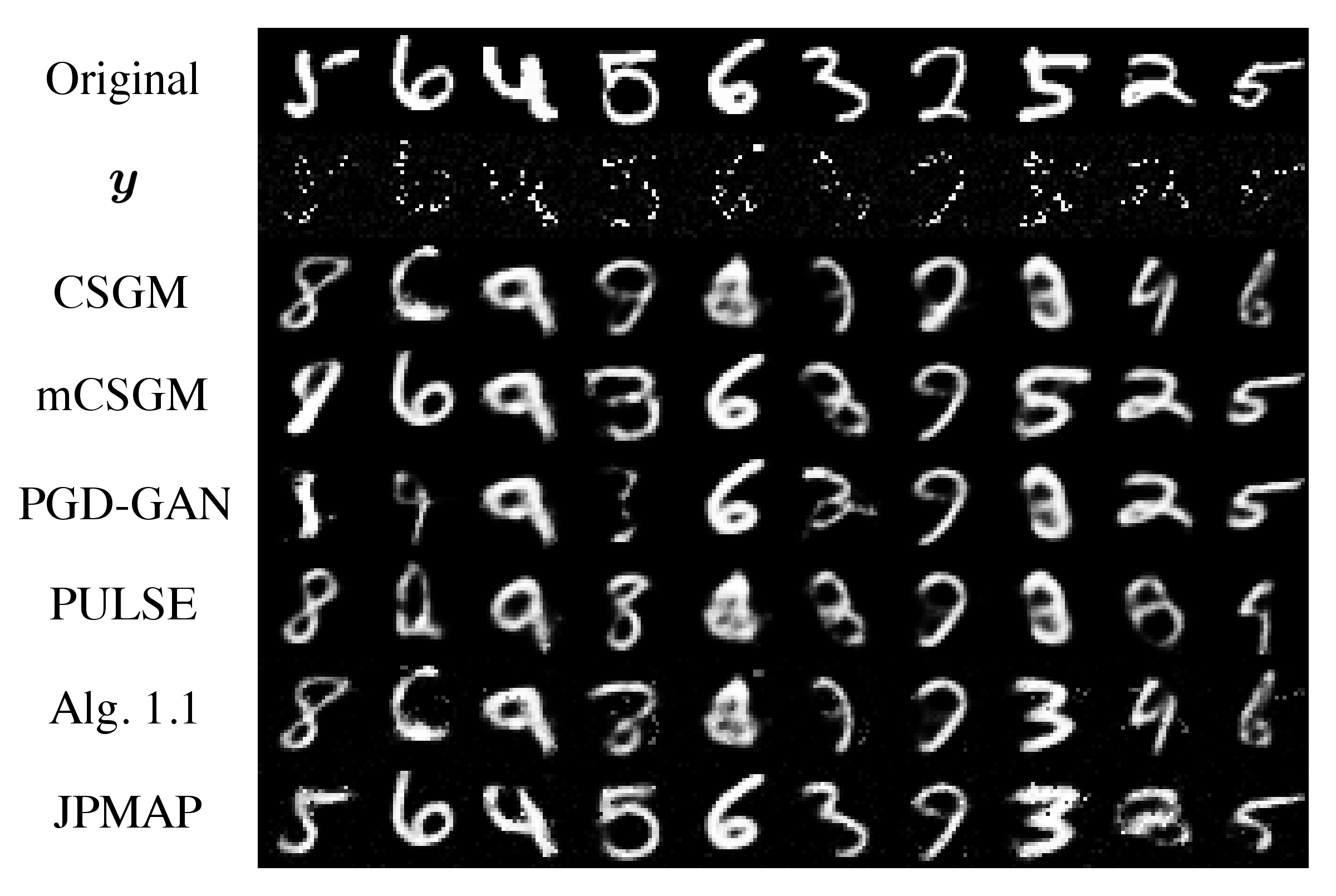}
  }
  \subfigure[Results on deblurring.]{
  \label{fig:results_deblurring}
    \includegraphics[width=0.75\textwidth]{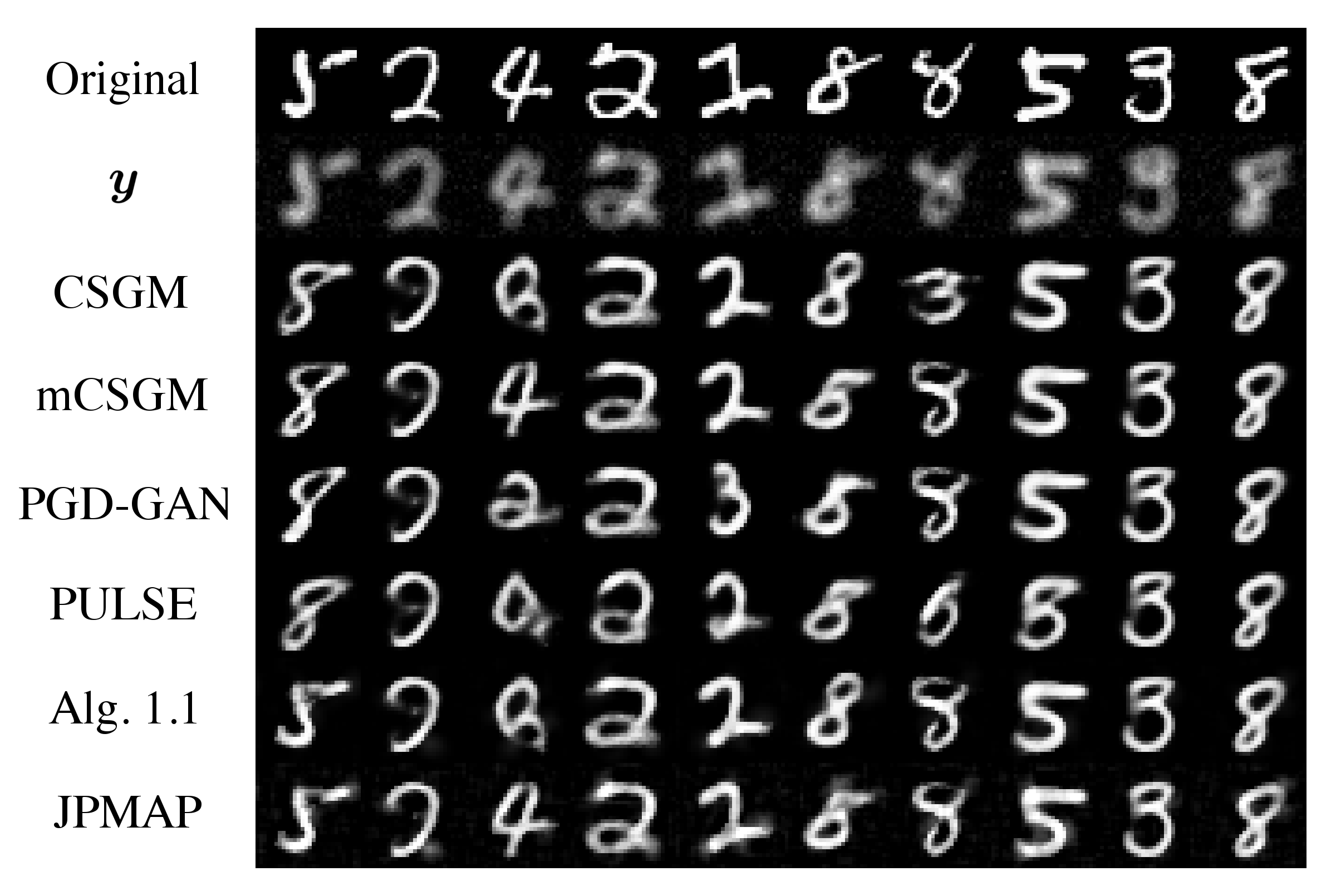}
  }\\
  \caption{\emph{Experimental results on MNIST.} Comparison of JPMAP with the baseline algorithms described in Section~\ref{sec:baseline_algorithms}.
  \emph{(a)} Some selected results from the interpolation experiment with 80\% of missing pixels and Gaussian noise with $\sigma=10/255$.
  From top to bottom: original image $\vx^*$, corrupted image $\vy$, and the results computed by CSGM, mCSGM, PGD-GAN, PULSE, Algorithm~\ref{alg:MAPz-splitting} and JPMAP.
  \emph{(b)} Same as (a) from the deblurring experiment with kernel size $3\times 3$ and Gaussian noise with $\sigma=10/255$.
  \emph{Conclusion:} Our algorithm performs generally better than the baseline algorithms, although in some cases it falls behind mCSGM. %
  }
  \label{fig:MNIST_results_comparison}
 \end{figure*}

 \begin{figure*}[htbp]
\centering
  \subfigure[Results of interpolation on CelebA.]{
  \label{fig:celeba_inpainting}
    \includegraphics[width=0.45\textwidth]{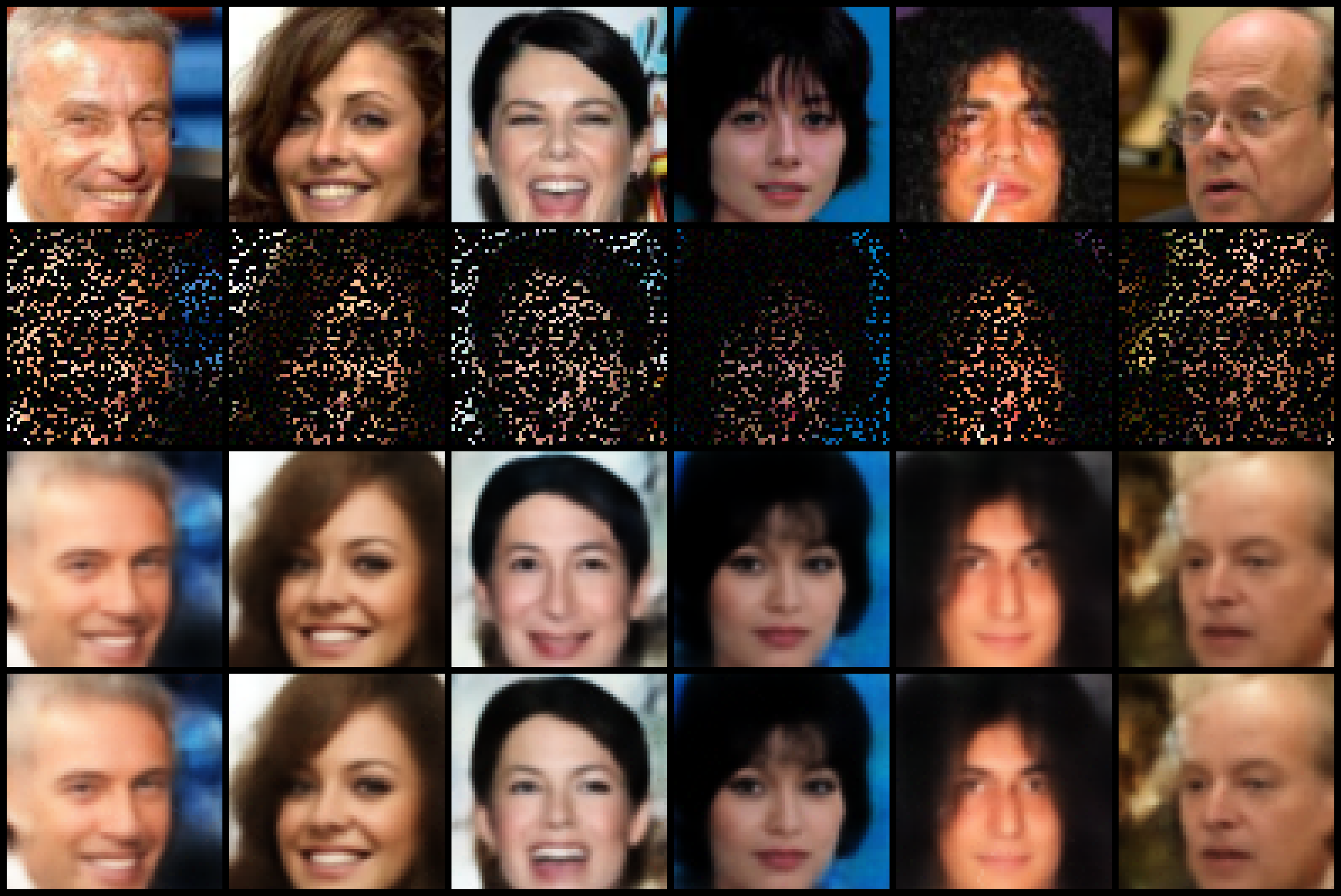}
  }
  \subfigure[Reconstructions $\muDecoder(\muEncoder(\x))$ .]{
  \label{fig:celeba_reconstructions}
    \includegraphics[width=0.45\textwidth]{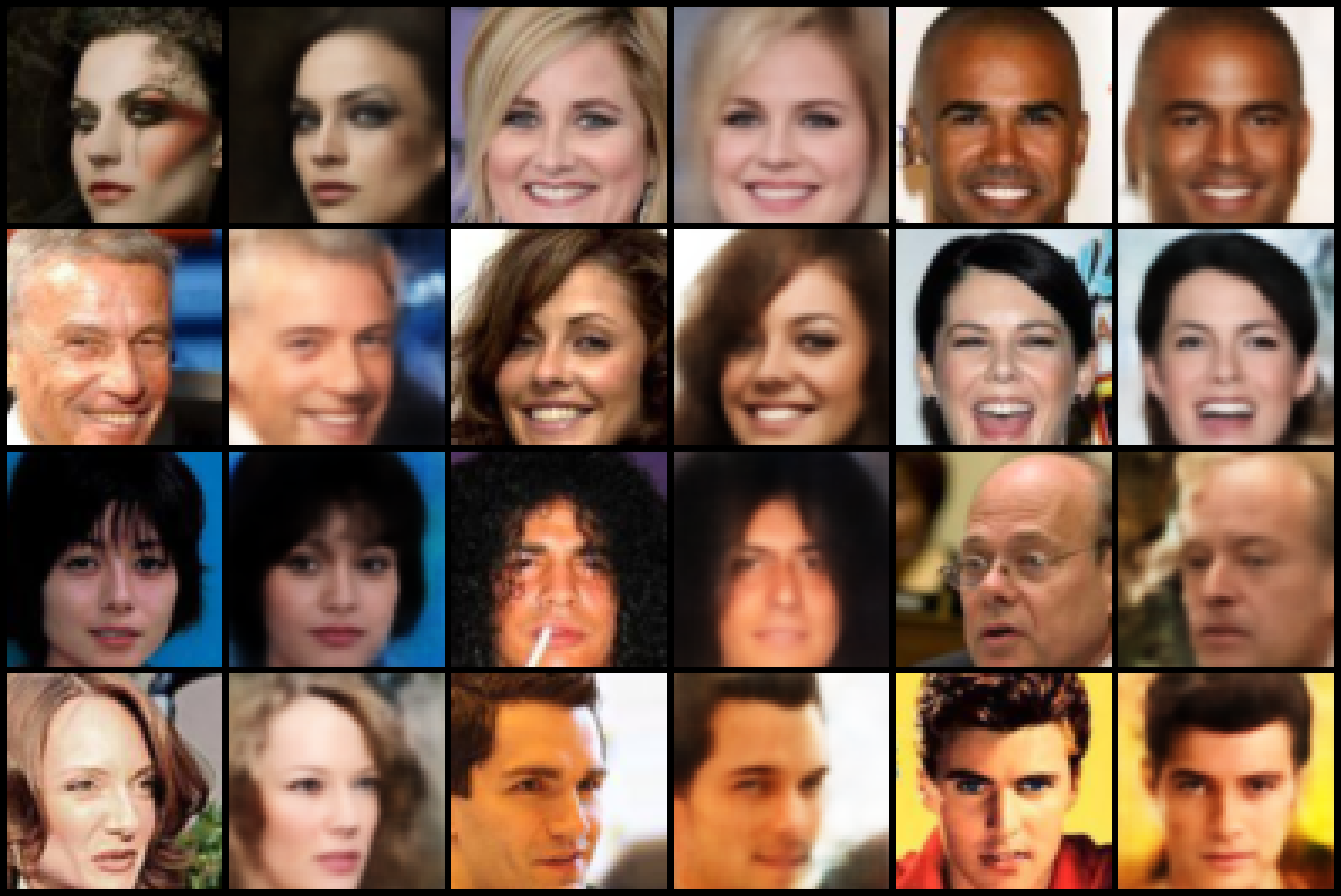}
  }
  \caption{\emph{(a)} Some preliminary results on CelebA: 80\% of missing pixels, noise std $\sigma=10/255$. From top to bottom: original image $\vx^*$, corrupted image $\tilde \vx$, restored by CSGM \cite{bora2017compressed}, restored image $\hat \vx$ by our framework. \emph{(b)} Reconstructions $\muDecoder(\muEncoder(\vx))$ (even columns) for some test samples $\vx$ (odd columns), showing the over-regularization of data manifold imposed by the trained vanilla VAE.
  As a consequence, $-\log \ConditionalPDF{\Z}{\Y}{\z}{\y}$  does not have as many local minima and then a simple gradient descent yields almost the same result as JPMAP (except on third column of (a)).
  }
  \label{fig:CelebA_results_comparison}
 \end{figure*}

\subsection{Effectiveness of the encoder as a fast approximate minimizer}
\label{sec:qualitative-JPMAP2}

Proposition~\ref{thm:convergence-approx} shows that the proposed alternate minimization scheme in Algorithm~\ref{alg:JPMAP3new} converges to a stationary point of $J_1$. And so does the gradient descent scheme in~\citep{bora2017compressed}. 
Since both algorithms have to deal with non-convex energies, they both risk converging to spurious local minima. Also both algorithms solve essentially the same model when the variance $\gamma$ of the coupling term tends to zero.\\

If our algorithm shows better performance (see next subsection), it is mainly because it relies on a previously trained VAE in two fundamental ways:
\emph{(i)} to avoid getting trapped in spurious local minima and 
\emph{(ii)} to accelerate performance during the initial iterations ($n<n_{\min}$). These two features are only possible if the autoencoder approximation is good enough and if the encoder is able to provide good initializations for the non-convex $\vz$- optimization subproblem in line 13 of Algorithm~\ref{alg:JPMAP3new}.\\

Figures~\ref{fig:trained_VAE}~and~\ref{fig:encoder-init} illustrate these two properties of our VAE.
We do so by selecting a random $\vx_0$ from MNIST test set and by computing $\vz^*(\vz_0) := \gd_\vz J_1(\vx_0,\vz)$ with different initial values $\vz_0$. These experiments were performed using the ADAM minimization algorithm with learning rate equal to $0.01$. Figure~\ref{fig:assumption2A} shows that $\vz^*(\vz_0)$ reaches the global optimum for most (but not all) initializations $\vz_0$.
Indeed from 200 random initializations $\vz_0 \sim \Normal(0,I)$, 195 reach the same global minimum, whereas 5 get stuck at a higher energy value. However these 5 initial values have energy values $J_1(\vx_0,\vz_0) \gg J_1(\vx_0,\vz^1)$ far larger than those of the encoder initialization $\vz^1 = \muEncoder(\vx_0)$, and are thus never chosen by Algorithm~\ref{alg:JPMAP3new}.
The encoder initialization $\vz^1$ on the other hand provides much faster convergence to the global optimum.\\

In addition, this experiment shows that we cannot assume $\vz$-convexity: The presence of plateaux in the trajectories of many random initializations as well as the fact that a few initializations do not lead to the global minimum indicates that $J_1$ may not be everywhere convex with respect to $\vz$. However, in contrast to classical works on alternate convex search, our approach adopts weaker assumptions and does not require %
convexity on
$\vz$ to prove convergence in Proposition~\ref{thm:convergence-approx}.\\

In Figure~\ref{fig:assumption2B} we display the distances of each trajectory to the global optimum $\vz^*$ (taken as the median over all initializations $\vz_0$ of the final iterates $\vz^*(\vz_0)$); note that this optimum is always reached, which suggests that $\vz \mapsto J_1(\vx_0,\vz)$ has a unique global minimizer in this case. Finally, Figure~\ref{fig:trained_VAE} shows that the encoder approximation is quite good both in the latent space (Figure~\ref{fig:encoder_approximation}) and in image space (Figures~\ref{fig:decoded_exact_optimum}~and~\ref{fig:decoded_approximate_optimum}). It also shows that the true posterior $p_\theta(\vz|\vx)$ is pretty close to log-concave near the maximum of $q_\phi(\vz|\vx)$.\\

\subsection{Image restoration experiments}
\emph{Choice of $\vx_0$:} In the previous section, our validation experiments used a random $\vx_0$ from the data set as initialization. When dealing with an image restoration problem, Algorithms~\ref{alg:JPMAP2new}~and~\ref{alg:JPMAP3new} require an initial value of $\vx_0$ to be chosen. In all experiments we choose this initial value as $\mA^T\y$.

\emph{Choice of $n_1$ and $n_2$:} 
After a few runs of Algorithm~\ref{alg:JPMAP2new} we find that in most cases, during the first 10 or 20 iterations $\vz^1$ decreases the energy with respect to the previous iteration, and this value depends on the inverse problems (for example, for denoising is smaller than for compressed sensing). But after at most 150 iterations the autoencoder approximation is no longer good enough and we need to perform gradient descent on $\vz_n$ in order to further decrease the energy. Based on these findings we set $n_1 = 25$ and $n_2 = 150$ in Algorithm~\ref{alg:JPMAP3new} for all experiments.
Note that we could also choose $n_1=n_2=n_{\max}$, since in all our experiments we observed that the algorithm auto-regulates itself, evolving from $i^*=1$ in the first few dozen iterations to $i^*=3$ when it is close to convergence. Choosing a finite value for $n_1$ and $n_2$ is only needed to ensure that $i^*=3$ when $n\to\infty$, which is a necessary condition to prove statement 3 of Proposition~\ref{thm:convergence-approx}.

Figure~\ref{fig:evolution_JPMAP} shows the evolution of $\vx_k$ and $\generator(\vz_k)$ from Algorithm~\ref{alg:MAPzCS} in an interpolation example. Here we can see how the exponential multiplier method in Equation~\eqref{eq:lagrangian_emm} updates the values of $\beta_k$ to ensure $\|\generator(\z_k)-\x_k\|^2 \leq \varepsilon$.\\

Figures~\ref{fig:DVAE-opt-restoration1} and \ref{fig:DVAE-opt-restoration2} show the results of denoising, interpolation, compressed sensing, deblurring and super-resolution experiments on MNIST for different degradation levels using the proposed algorithm (JPMAP) and the baseline algorithms introduced in section \ref{sec:baseline_algorithms}. The metrics used are PSNR and LPIPS\footnote{MNIST images were zero-padded to $32\times32$ because LPIPS does not accept $28\times28$ images.}~\citep{zhang2018unreasonable} mean $\pm$ its standard error computed over 100 random experiments for each problem.
Figure~\ref{fig:MNIST_results_comparison} displays the images of 10 representative interpolation and deblurring experiments from the hundreds of experiments summarized in figures \ref{fig:DVAE-opt-restoration1} and \ref{fig:DVAE-opt-restoration2}.

These results show that JPMAP outperforms all other baseline algorithms in terms of PSNR and LPIPS when random restarts are not allowed.
When 10 random restarts are allowed for CSGM, but not for JPMAP, then both algorithms (JPMAP and mCSGM) show a similar global performance: JPMAP tends to provide a slightly better result than mCSGM except for the most extremely ill-posed interpolation, super-resolution and compressed sensing experiments (when available measurements are less than 10\% the number of pixels). In that case mCSGM outperforms JPMAP by an equally small margin.
The latter case can be explained by the fact that the encoder (which is used by JPMAP but not by CSGM) struggles to generalize to images $\x$ which are very far away from $\mathcal{M}$ (the range of the generator). Indeed, in Section~\ref{sec:denoising-criterion} we trained the VAE's encoder to generalize to $\mathcal{M}+n$ where $n\sim\mathcal{N}(0,\sigmaDVAE^2I)$ and $\sigmaDVAE=15/255$. This value is optimal for moderately ill-posed problems, but more extreme problems may require larger values of $\sigmaDVAE$ or a coarse to fine scheme, where a coarse VAE (with large $\sigmaDVAE$) is used during the first few iterations and a finer VAE (with smaller $\sigmaDVAE$) is used later until convergence.
Finally one may consider using random restarts for both JPMAP and CSGM for a more fair comparison.

Figure~\ref{fig:timePSNR_interpolation} performs a more detailed comparison between JPMAP and mCSGM, which also considers running times of both algorithms. For the stopping criteria used in our experiments, one run of JPMAP requires roughly as much time as mCSGM (with $m=10$ restarts).
In addition for moderately ill-posed problems (like interpolation of 75\% missing pixels see subfigures (a) and (b)) where JPMAP's performance beats mCSGM, we can observe that JPMAP also converges much faster to that solution.
For more extremely ill-posed problems (like interpolation of 90\% missing pixels, see subfigures (c) and (d)) the opposite is true.

In the case of CelebA, we did not observe as much difference between JPMAP and CSGM as on MNIST. In Figure \ref{fig:celeba_inpainting} the restorations on an interpolation problem (80\% of missing pixels) are very similar to each other, but blurry. Also, although this problem is very ill-posed, both algorithms impressively find a solution $\vz^*$ very close to the code $\muEncoder(\vx)$ of the ground truth image $\vx$, except for the third column where CSGM converges to a local minimum.\\

We hypothesize that, as CelebA is a substantially more complex dataset than MNIST, a simple model like vanilla VAE is over-regularizing the manifold of samples (underfitting problem). In particular, because of the spectral bias \citep{rahaman2019spectral} the learned manifold perhaps only contains low-frequency approximations of the true images as we can see in the reconstructions $\muDecoder(\muEncoder(\vx))$ of test samples (see Figure \ref{fig:celeba_reconstructions}). This may cause  the posterior $\ConditionalPDF{\Z}{\Y}{\z}{\y}$ to have fewer local minima. With more realistic generative models such as VDVAE~\citep{child2020vdvae} or NVAE~\citep{vahdat2020nvae}, which better represent the true data manifold, we expect the objective function $-\log \ConditionalPDF{\Z}{\Y}{\z}{\y}$ to exhibit a much larger number of local minima, thus making it more difficult to optimize by a simple gradient descent scheme. In that situation the proposed JPMAP method would more clearly show its advantages.

\section{Conclusions and Future work}\label{sect:future_work}

In this work we presented a new framework to solve inverse problems with a convex data-fitting term and a non-convex regularizer learned in the latent space via variational autoencoders. Unlike similar approaches like CSGM \cite{bora2017compressed}, PULSE \cite{Menon2020} and PGD-GAN \cite{shah2018solving} which learn the prior based on generative models, our approach is based on a generalization of alternate convex search to quasi-bi-convex functionals. This quasi-bi-convexity is the result of considering the joint posterior distribution of latent and image spaces. As a result, the proposed approach provides convergence guarantees that extend to a larger family of inverse problems. 
Experiments on denoising, interpolation, deconvolution, super-resolution and compressed sensing confirm this, since our approach gets stuck much less often in spurious local minima than CSGM, PGD-GAN or PULSE, which are simply based on gradient descent of a highly non-convex functional. This leads to restored images which are significantly better in terms of PSNR and LPIPS.

\paragraph{JPMAP vs related \emph{Plug \& Play} approaches}
When compared to other decoupled \emph{plug \& play} approaches that solve inverse problems using NN-based priors, our approach is constrained in different ways:

\emph{(a)} In a certain sense our approach is \emph{less constrained} than existing decoupled approaches since we do not require to retrain the NN-based denoiser to enforce any particular property to ensure convergence: \citet{ryu2019plug} requires the denoiser's residual operator to be non-expansive, and \citet{gupta2018cnn,shah2018solving} require the denoiser to act as a projector. The effect of these modifications to the denoiser on the quality of the underlying image prior has never been studied in detail and chances are that such constraints degrade it. Our method only requires a variational autoencoder without any further constraints, and the quality and expressiveness of this prior can be easily checked by sampling and reconstruction experiments. Checking the quality of the prior is a much more difficult task for \citet{ryu2019plug,gupta2018cnn,shah2018solving} which rely on an implicit prior, and do not provide a generative model.

\emph{(b)} Unlike \citep{ryu2019plug} which requires the data-fitting term $\Fdata(\vx)$ to be \emph{strongly convex} to ensure convergence, our method admits weakly convex and ill-posed data-fitting terms like missing pixels, compressed sensing and non-invertible blurring for instance.

\emph{(c)} On the other hand our method is \emph{more constrained} in the sense that it relies on a generative model of a \emph{fixed size}. Even if the generator and encoder are both convolutional neural networks, training and testing the same model on images of different sizes is a priori not possible because the latent space has a fixed dimension and a fixed distribution. As a future work we plan to explore different ways to address this limitation. The most straightforward way is to use our model to learn a prior of image patches of a fixed size and stitch this model via aggregation schemes like in EPLL~\cite{Zoran2011} to obtain a global prior model for images of any size. Alternatively we can use hierarchical generative models like in \cite{Karras2018, vahdat2020nvae} or resizable ones like in \cite{Bergmann2017,Whang2020}, and adapt our framework accordingly.

\paragraph{\MAP-x or \MAP-z or joint \MAP-x-z}

In this work we explored and clarified the tight relationships between joint \MAP-\x-\z\ estimation, splitting and continuation schemes and the more common \MAP-\z\ estimator in the context of inverse problems with a generative prior.
On the other hand \MAP-\x\ estimators (which are otherwise standard in bayesian imaging) remained largely unexplored in the context of generative priors, due to the optimization challenges they impose, until the recent work of \citet{Helminger2020,Whang2020} showed that a normalizing flow-based generative model allows to overcome those challenges and deems this problem tractable. Similarly \citet{Oberlin2021} use Glow (an invertible normalizing flow) to compare synthesis-based and analysis-based reconstructions. 
Yet an extensive comparison of the advantages and weaknesses of these three families of estimators under the same prior model is still missing, and so is the link between these MAP estimators and the analysis/synthesis-based estimators in \citep{Oberlin2021}. This will be the subject of future work.

\paragraph{Extension to higher dimensional problems} 

The present paper provides a first proof of concept of our framework, on a very simple dataset (MNIST) with a very simple VAE. More experiments are needed to verify that the framework preserves its qualitative advantages on more high-dimensional datasets (like CelebA, FFHQ, etc.), and a larger selection of inverse problems.

Generalizing our proposed method to much higher dimensional problems implies training much more complex generative models which can match the finer details and higher complexity of such data.
We can still use over-simplified generative models in those cases, but our preliminary experiments suggest that in that situation, not only do we obtain relatively poor reconstructions, but the objective function associated to the \MAP-\z\ problem presents less spurious local-minima: as a consequence our proposed joint \MAP-\x-\z\ is overkill in that configuration, and does not present such a great competitive advantage. 

The big challenge of generalizing our proposed method to much higher dimensional problems is then to train sufficiently detailed and complex generative models. And in this area VAEs traditionally lagged behind GANs in terms of quality of the generated samples, the former producing in general more blurred samples. Nevertheless some studies \cite{Sajjadi2018a} %
show that VAEs and Normalizing Flows produce more accurate representations of the probability distribution. In the medium term our work should be able to benefit from recent advances in VAE architectures \citep{child2020vdvae,vahdat2020nvae,TwoStageVAE}, and adversarial training for VAEs \citep{Pu2017,Pu2017a,PGA-Zhang2019} that reach GAN-quality samples with the additional benefits of VAEs. These extensions are however non-trivial, since these VAEs have a huge number of parameters  and they need to be retrained or fine-tuned using a denoising criterion (see section~\ref{sec:denoising-criterion} and \cite{Im2017}) for our method to work properly. In addition, the latent space of the most competitive VAEs is much larger than the image space, which may reduce its regularization capabilities. 

As an alternative, GAN-based generative models can be augmented with a denoising encoder network \cite{Donahue2019}, and Normalizing Flows can also act as projectors or denoising VAEs if we split the latent space to separate the data manifold from its complement, as suggested in \cite{Brehmer2020,Liu2021}. In combination with relaxation techniques, such augmented GANs or specially tailored Flows may provide SOTA priors that fit our quasi-bi-convex optimization framework.

\paragraph{Towards stronger convergence guarantees under weaker conditions.}
The proposed Algorithm~\ref{alg:MAPzCS} bears strong similarities with ADMM with non-linear constraints as introduced by Valkonen \emph{et al.} \cite{Valkonen2014,Benning2016} and analyzed by Latorre-Gómez \emph{et al.} \cite{latorre2019fastADMM}. Latorre-Gómez result provides very strong convergence guarantees (linear convergence rates to a global optimum), but requires the data fitting term to be strongly convex or to satisfy a restricted strong convexity property. Our result, on the other hand, provides much weaker convergence guarantees (convergence to a stationary point), but does not require strong convexity. Further exploring these connections might hopefully lead to something closer to the best of both worlds.  

\section*{Acknowledgments}
We would like to sincerely thank Mauricio Delbracio, José Lezama and Pablo Musé for their help, their insightful comments, and their continuous support throughout this project.

\bibliographystyle{plainnat} %
\bibliography{references}

\begin{thebibliography}{80}
\providecommand{\natexlab}[1]{#1}
\providecommand{\url}[1]{\texttt{#1}}
\expandafter\ifx\csname urlstyle\endcsname\relax
  \providecommand{\doi}[1]{doi: #1}\else
  \providecommand{\doi}{doi: \begingroup \urlstyle{rm}\Url}\fi

\bibitem[Aguerrebere et~al.(2017)Aguerrebere, Almansa, Delon, Gousseau, and
  Muse]{Aguerrebere2014b}
Cecilia Aguerrebere, Andres Almansa, Julie Delon, Yann Gousseau, and Pablo
  Muse.
\newblock {A Bayesian Hyperprior Approach for Joint Image Denoising and
  Interpolation, With an Application to HDR Imaging}.
\newblock \emph{IEEE Transactions on Computational Imaging}, 3\penalty0
  (4):\penalty0 633--646, dec 2017.
\newblock ISSN 2333-9403.
\newblock \doi{10.1109/TCI.2017.2704439}.
\newblock URL \url{https://nounsse.github.io/HBE_project/}.

\bibitem[Benning et~al.(2016)Benning, Knoll, Sch{\"{o}}nlieb, and
  Valkonen]{Benning2016}
Martin Benning, Florian Knoll, Carola~Bibiane Sch{\"{o}}nlieb, and Tuomo
  Valkonen.
\newblock {Preconditioned ADMM with nonlinear operator constraint}.
\newblock \emph{IFIP Advances in Information and Communication Technology},
  494:\penalty0 117--126, 2016.
\newblock ISSN 18684238.
\newblock \doi{10.1007/978-3-319-55795-3_10}.

\bibitem[Bergmann et~al.(2017)Bergmann, Jetchev, and Vollgraf]{Bergmann2017}
Urs Bergmann, Nikolay Jetchev, and Roland Vollgraf.
\newblock {Learning Texture Manifolds with the Periodic Spatial GAN}.
\newblock \emph{(ICML) International Conference on Machine Learning},
  1:\penalty0 722--730, may 2017.

\bibitem[Bigdeli and Zwicker(2017)]{Bigdeli2017a}
Siavash~Arjomand Bigdeli and Matthias Zwicker.
\newblock {Image Restoration using Autoencoding Priors}.
\newblock Technical report, 2017.

\bibitem[Bigdeli et~al.(2017)Bigdeli, Jin, Favaro, and Zwicker]{Bigdeli2017}
Siavash~Arjomand Bigdeli, Meiguang Jin, Paolo Favaro, and Matthias Zwicker.
\newblock {Deep Mean-Shift Priors for Image Restoration}.
\newblock In \emph{(NIPS) Advances in Neural Information Processing Systems
  30}, pages 763--772, sep 2017.
\newblock URL
  \url{http://papers.nips.cc/paper/6678-deep-mean-shift-priors-for-image-restoration}.

\bibitem[Bora et~al.(2017)Bora, Jalal, Price, and Dimakis]{bora2017compressed}
Ashish Bora, Ajil Jalal, Eric Price, and Alexandros~G Dimakis.
\newblock Compressed sensing using generative models.
\newblock In \emph{(ICML) International Conference on Machine Learning},
  volume~2, pages 537--546. JMLR. org, 2017.
\newblock ISBN 9781510855144.

\bibitem[Bredies et~al.(2010)Bredies, Kunisch, and Pock]{Bredies2010}
Kristian Bredies, Karl Kunisch, and Thomas Pock.
\newblock {Total generalized variation}.
\newblock \emph{SIAM Journal on Imaging Sciences}, 3\penalty0 (3):\penalty0
  492--526, 2010.
\newblock ISSN 19364954.
\newblock \doi{10.1137/090769521}.

\bibitem[Brehmer and Cranmer(2020)]{Brehmer2020}
Johann Brehmer and Kyle Cranmer.
\newblock {Flows for simultaneous manifold learning and density estimation}.
\newblock mar 2020.
\newblock URL \url{http://arxiv.org/abs/2003.13913}.

\bibitem[Buzzard et~al.(2018)Buzzard, Chan, Sreehari, and
  Bouman]{buzzard2018plug}
Gregery~T Buzzard, Stanley~H Chan, Suhas Sreehari, and Charles~A Bouman.
\newblock Plug-and-play unplugged: Optimization-free reconstruction using
  consensus equilibrium.
\newblock \emph{SIAM Journal on Imaging Sciences}, 11\penalty0 (3):\penalty0
  2001--2020, 2018.

\bibitem[Chambolle(2004)]{Chambolle04}
A~Chambolle.
\newblock {An algorithm for total variation minimization and applications}.
\newblock \emph{Journal of Mathematical Imaging and Vision}, 20:\penalty0
  89--97, 2004.
\newblock \doi{10.1023/B:JMIV.0000011325.36760.1e}.

\bibitem[{Chan} et~al.(2017){Chan}, {Wang}, and {Elgendy}]{chan2017plug}
S.~H. {Chan}, X.~{Wang}, and O.~A. {Elgendy}.
\newblock Plug-and-play admm for image restoration: Fixed-point convergence and
  applications.
\newblock \emph{IEEE Transactions on Computational Imaging}, 3\penalty0
  (1):\penalty0 84--98, March 2017.
\newblock ISSN 2333-9403.
\newblock \doi{10.1109/TCI.2016.2629286}.

\bibitem[Chen and Pock(2017)]{Chen2017}
Yunjin Chen and Thomas Pock.
\newblock {Trainable Nonlinear Reaction Diffusion: A Flexible Framework for
  Fast and Effective Image Restoration}.
\newblock \emph{IEEE Transactions on Pattern Analysis and Machine
  Intelligence}, 39\penalty0 (6):\penalty0 1256--1272, 2017.
\newblock ISSN 01628828.
\newblock \doi{10.1109/TPAMI.2016.2596743}.

\bibitem[Child(2020)]{child2020vdvae}
Rewon Child.
\newblock {Very Deep VAEs Generalize Autoregressive Models and Can Outperform
  Them on Images}.
\newblock In \emph{(ICLR) International Conference on Learning
  Representations}, pages 1--17, 2020.
\newblock URL \url{https://openreview.net/forum?id=RLRXCV6DbEJ}.

\bibitem[Clevert et~al.(2016)Clevert, Unterthiner, and Hochreiter]{Clevert2016}
Djork-Arn{\'{e}} Clevert, Thomas Unterthiner, and Sepp Hochreiter.
\newblock {Fast and Accurate Deep Network Learning by Exponential Linear Units
  (ELUs)}.
\newblock In \emph{(ICLR) International Conference on Learning
  Representations}, nov 2016.

\bibitem[Cohen et~al.(2020)Cohen, Elad, and Milanfar]{Cohen2020}
Regev Cohen, Michael Elad, and Peyman Milanfar.
\newblock {Regularization by Denoising via Fixed-Point Projection (RED-PRO)}.
\newblock aug 2020.
\newblock ISSN 23318422.
\newblock URL \url{http://arxiv.org/abs/2008.00226}.

\bibitem[Dai and Wipf(2019)]{TwoStageVAE}
Bin Dai and David Wipf.
\newblock {Diagnosing and Enhancing VAE Models}.
\newblock \emph{ICLR}, pages 1--42, 2019.
\newblock URL \url{https://openreview.net/forum?id=B1e0X3C9tQ}.

\bibitem[{Dal Maso}(1993)]{DalMaso1993}
Gianni {Dal Maso}.
\newblock \emph{{An Introduction to $\Gamma$-Convergence}}.
\newblock Birkh{\"{a}}user Boston, Boston, MA, 1993.
\newblock ISBN 978-1-4612-6709-6.
\newblock \doi{10.1007/978-1-4612-0327-8}.
\newblock URL \url{http://link.springer.com/10.1007/978-1-4612-0327-8}.

\bibitem[Diamond et~al.(2017)Diamond, Sitzmann, Heide, and
  Wetzstein]{diamond2017unrolled}
Steven Diamond, Vincent Sitzmann, Felix Heide, and Gordon Wetzstein.
\newblock Unrolled optimization with deep priors.
\newblock 2017.

\bibitem[Donahue and Simonyan(2019)]{Donahue2019}
Jeff Donahue and Karen Simonyan.
\newblock {Large Scale Adversarial Representation Learning}.
\newblock 2019.
\newblock URL \url{http://arxiv.org/abs/1907.02544}.

\bibitem[Dong et~al.(2014)Dong, Loy, He, and Tang]{dong2014learning}
Chao Dong, Chen~Change Loy, Kaiming He, and Xiaoou Tang.
\newblock Learning a deep convolutional network for image super-resolution.
\newblock In \emph{European conference on computer vision}, pages 184--199.
  Springer, 2014.

\bibitem[Donoho(1995)]{Donoho1995}
D.L. Donoho.
\newblock {De-noising by soft-thresholding}.
\newblock \emph{IEEE Transactions on Information Theory}, 41\penalty0
  (3):\penalty0 613--627, may 1995.
\newblock ISSN 00189448.
\newblock \doi{10.1109/18.382009}.

\bibitem[Elad(2010)]{Elad2010}
Michael Elad.
\newblock \emph{{Sparse and Redundant Representations: From Theory to
  Applications in Signal and Image Processing}}.
\newblock Springer New York, New York, NY, 2010.
\newblock ISBN 978-1-4419-7011-4.
\newblock \doi{10.1007/978-1-4419-7011-4}.

\bibitem[Gao et~al.(2019)Gao, Tao, Shen, and Jia]{gao2019dynamic}
Hongyun Gao, Xin Tao, Xiaoyong Shen, and Jiaya Jia.
\newblock Dynamic scene deblurring with parameter selective sharing and nested
  skip connections.
\newblock In \emph{Proceedings of the IEEE Conference on Computer Vision and
  Pattern Recognition}, pages 3848--3856, 2019.

\bibitem[Gharbi et~al.(2016)Gharbi, Chaurasia, Paris, and
  Durand]{gharbi2016deep}
Micha{\"e}l Gharbi, Gaurav Chaurasia, Sylvain Paris, and Fr{\'e}do Durand.
\newblock Deep joint demosaicking and denoising.
\newblock \emph{ACM Transactions on Graphics (TOG)}, 35\penalty0 (6):\penalty0
  191, 2016.

\bibitem[Gilboa and Osher(2008)]{Gilboa2008}
Guy Gilboa and Stanley Osher.
\newblock {Nonlocal operators with applications to image processing}.
\newblock \emph{Multiscale Modeling and Simulation}, 7\penalty0 (3):\penalty0
  1005--1028, 2008.
\newblock ISSN 15403459.
\newblock \doi{10.1137/070698592}.

\bibitem[Gilton et~al.(2019)Gilton, Ongie, and Willett]{gilton2019neumann}
Davis Gilton, Greg Ongie, and Rebecca Willett.
\newblock Neumann networks for inverse problems in imaging.
\newblock 2019.

\bibitem[Gorski et~al.(2007)Gorski, Pfeuffer, and Klamroth]{Gorski2007}
Jochen Gorski, Frank Pfeuffer, and Kathrin Klamroth.
\newblock {Biconvex sets and optimization with biconvex functions: a survey and
  extensions}.
\newblock \emph{Mathematical Methods of Operations Research}, 66\penalty0
  (3):\penalty0 373--407, nov 2007.
\newblock ISSN 1432-2994.
\newblock \doi{10.1007/s00186-007-0161-1}.

\bibitem[Gregor and LeCun(2010)]{gregor2010learning}
Karol Gregor and Yann LeCun.
\newblock Learning fast approximations of sparse coding.
\newblock In \emph{Proceedings of the 27th International Conference on
  International Conference on Machine Learning}, pages 399--406. Omnipress,
  2010.

\bibitem[Gupta et~al.(2018)Gupta, Jin, Nguyen, McCann, and Unser]{gupta2018cnn}
Harshit Gupta, Kyong~Hwan Jin, Ha~Q Nguyen, Michael~T McCann, and Michael
  Unser.
\newblock Cnn-based projected gradient descent for consistent ct image
  reconstruction.
\newblock \emph{IEEE transactions on medical imaging}, 37\penalty0
  (6):\penalty0 1440--1453, 2018.
\newblock \doi{10.1109/TMI.2018.2832656}.

\bibitem[Hand and Voroninski(2020)]{Hand2020}
Paul Hand and Vladislav Voroninski.
\newblock {Global Guarantees for Enforcing Deep Generative Priors by Empirical
  Risk}.
\newblock \emph{IEEE Transactions on Information Theory}, 66\penalty0
  (1):\penalty0 401--418, 2020.
\newblock ISSN 15579654.
\newblock \doi{10.1109/TIT.2019.2935447}.

\bibitem[Helminger et~al.(2020)Helminger, Bernasconi, Djelouah, Gross, and
  Schroers]{Helminger2020}
Leonhard Helminger, Michael Bernasconi, Abdelaziz Djelouah, Markus Gross, and
  Christopher Schroers.
\newblock {Blind Image Restoration with Flow Based Priors}.
\newblock Technical report, sep 2020.
\newblock URL \url{http://arxiv.org/abs/2009.04583}.

\bibitem[Huang et~al.(2018)Huang, Hand, Heckel, and Voroninski]{Huang2018}
Wen Huang, Paul Hand, Reinhard Heckel, and Vladislav Voroninski.
\newblock {A Provably Convergent Scheme for Compressive Sensing under Random
  Generative Priors}.
\newblock dec 2018.
\newblock URL \url{http://arxiv.org/abs/1812.04176}.

\bibitem[Im et~al.(2017)Im, Ahn, Memisevic, and Bengio]{Im2017}
Daniel~Jiwoong Im, Sungjin Ahn, Roland Memisevic, and Yoshua Bengio.
\newblock {Denoising criterion for variational auto-encoding framework}.
\newblock In \emph{31st AAAI Conference on Artificial Intelligence, AAAI 2017},
  pages 2059--2065. AAAI press, nov 2017.

\bibitem[Kamilov et~al.(2017)Kamilov, Mansour, and Wohlberg]{kamilov2017plug}
Ulugbek~S Kamilov, Hassan Mansour, and Brendt Wohlberg.
\newblock A plug-and-play priors approach for solving nonlinear imaging inverse
  problems.
\newblock \emph{IEEE Signal Processing Letters}, 24\penalty0 (12):\penalty0
  1872--1876, 2017.

\bibitem[Karras et~al.(2017)Karras, Aila, Laine, and Lehtinen]{Karras2018}
Tero Karras, Timo Aila, Samuli Laine, and Jaakko Lehtinen.
\newblock {Progressive Growing of GANs for Improved Quality, Stability, and
  Variation}.
\newblock \emph{(ICLR) International Conference on Learning Representations},
  10\penalty0 (2):\penalty0 327--331, oct 2017.
\newblock URL \url{https://openreview.net/forum?id=Hk99zCeAb}.

\bibitem[Kingma and Welling(2013)]{Kingma2014}
Diederik~P Kingma and Max Welling.
\newblock {Auto-Encoding Variational Bayes}.
\newblock In \emph{(ICLR) International Conference on Learning
  Representations}, number~Ml, pages 1--14, dec 2013.
\newblock ISBN 1312.6114v10.
\newblock \doi{10.1051/0004-6361/201527329}.

\bibitem[Krizhevsky et~al.(2012)Krizhevsky, Sutskever, and
  Hinton.]{Krizhevsky2012}
Alex Krizhevsky, Ilya Sutskever, and Geoffrey~E. Hinton.
\newblock {Imagenet classification with deep convolutional neural networks}.
\newblock \emph{(NIPS) Advances in neural information processing systems},
  pages 1097--1105, 2012.
\newblock ISSN 10495258.

\bibitem[Latorre et~al.(2019)Latorre, Eftekhari, Cevher, G{\'{o}}mez,
  Eftekhari, and Cevher]{latorre2019fastADMM}
Fabian Latorre, Armin Eftekhari, Volkan Cevher, Fabian~Latorre G{\'{o}}mez,
  Armin Eftekhari, and Volkan Cevher.
\newblock {Fast and Provable ADMM for Learning with Generative Priors}.
\newblock In H~Wallach, H~Larochelle, A~Beygelzimer,
  F~d$\backslash$textquotesingle Alch{\'{e}}-Buc, E~Fox, and R~Garnett,
  editors, \emph{(NeurIPS) Advances in Neural Information Processing Systems},
  volume~32. Curran Associates, Inc., 2019.
\newblock URL
  \url{https://papers.nips.cc/paper/2019/hash/4559912e7a94a9c32b09d894f2bc3c82-Abstract.html}.

\bibitem[Laumont et~al.(2021)Laumont, {De Bortoli}, Almansa, Delon, Durmus, and
  Pereyra]{laumont2021pnpsgd}
R{\'{e}}mi Laumont, Valentin {De Bortoli}, Andr{\'{e}}s Almansa, Julie Delon,
  Alain Durmus, and Marcelo Pereyra.
\newblock {On Maximum-a-Posteriori estimation with Plug {\&} Play priors and
  stochastic gradient descent}.
\newblock sep 2021.
\newblock URL \url{https://hal.archives-ouvertes.fr/hal-03348735}.

\bibitem[Lecun et~al.(1998)Lecun, Bottou, Bengio, and Haffner]{MNIST}
Yann Lecun, Leon Bottou, Yoshua Bengio, and P.~Haffner.
\newblock {Gradient-based learning applied to document recognition}.
\newblock \emph{Proceedings of the IEEE}, 86\penalty0 (11):\penalty0
  2278--2324, 1998.
\newblock ISSN 00189219.
\newblock \doi{10.1109/5.726791}.

\bibitem[Liu et~al.(2021)Liu, Anwar, Qin, Ji, Caldwell, and Gedeon]{Liu2021}
Yang Liu, Saeed Anwar, Zhenyue Qin, Pan Ji, Sabrina Caldwell, and Tom Gedeon.
\newblock {Disentangling Noise from Images: A Flow-Based Image Denoising Neural
  Network}.
\newblock may 2021.
\newblock URL \url{https://arxiv.org/abs/2105.04746v1}.

\bibitem[Liu et~al.(2015)Liu, Luo, Wang, and Tang]{liu2015faceattributes}
Ziwei Liu, Ping Luo, Xiaogang Wang, and Xiaoou Tang.
\newblock Deep learning face attributes in the wild.
\newblock In \emph{Proceedings of International Conference on Computer Vision
  (ICCV)}, December 2015.

\bibitem[Louchet and Moisan(2013)]{Louchet2013}
C{\'{e}}cile Louchet and Lionel Moisan.
\newblock {Posterior expectation of the total variation model: Properties and
  experiments}.
\newblock \emph{SIAM Journal on Imaging Sciences}, 6\penalty0 (4):\penalty0
  2640--2684, dec 2013.
\newblock ISSN 19364954.
\newblock \doi{10.1137/120902276}.

\bibitem[Lucas et~al.(2019)Lucas, Tucker, Grosse, and Norouzi]{Lucas2019}
James Lucas, George Tucker, Roger Grosse, and Mohammad Norouzi.
\newblock {Don't blame the ELBO! A linear VAE perspective on posterior
  collapse}.
\newblock In \emph{Advances in Neural Information Processing Systems},
  volume~32, nov 2019.
\newblock URL \url{https://arxiv.org/abs/1911.02469}.

\bibitem[Meinhardt et~al.(2017)Meinhardt, Moller, Hazirbas, and
  Cremers]{meinhardt2017learning}
Tim Meinhardt, Michael Moller, Caner Hazirbas, and Daniel Cremers.
\newblock Learning proximal operators: Using denoising networks for
  regularizing inverse imaging problems.
\newblock In \emph{(ICCV) International Conference on Computer Vision}, pages
  1781--1790, 2017.
\newblock \doi{10.1109/ICCV.2017.198}.
\newblock URL
  \url{http://openaccess.thecvf.com/content_iccv_2017/html/Meinhardt_Learning_Proximal_Operators_ICCV_2017_paper.html}.

\bibitem[Menon et~al.(2020)Menon, Damian, Hu, Ravi, and Rudin]{Menon2020}
Sachit Menon, Alexandru Damian, Shijia Hu, Nikhil Ravi, and Cynthia Rudin.
\newblock {PULSE: Self-Supervised Photo Upsampling via Latent Space Exploration
  of Generative Models}.
\newblock \emph{Proceedings of the IEEE Computer Society Conference on Computer
  Vision and Pattern Recognition}, pages 2434--2442, 2020.
\newblock ISSN 10636919.
\newblock \doi{10.1109/CVPR42600.2020.00251}.

\bibitem[Oberlin and Verm(2021)]{Oberlin2021}
Thomas Oberlin and Mathieu Verm.
\newblock {Regularization via deep generative models: an analysis point of
  view}.
\newblock jan 2021.
\newblock URL \url{http://arxiv.org/abs/2101.08661}.

\bibitem[Papamakarios et~al.(2019)Papamakarios, Nalisnick, Rezende, Mohamed,
  and Lakshminarayanan]{Papamakarios2019}
George Papamakarios, Eric Nalisnick, Danilo~Jimenez Rezende, Shakir Mohamed,
  and Balaji Lakshminarayanan.
\newblock {Normalizing Flows for Probabilistic Modeling and Inference}.
\newblock 2019.

\bibitem[Paszke et~al.(2017)Paszke, Gross, Chintala, Chanan, Yang, DeVito, Lin,
  Desmaison, Antiga, and Lerer]{paszke2017automatic}
Adam Paszke, Sam Gross, Soumith Chintala, Gregory Chanan, Edward Yang, Zachary
  DeVito, Zeming Lin, Alban Desmaison, Luca Antiga, and Adam Lerer.
\newblock Automatic differentiation in pytorch.
\newblock 2017.

\bibitem[Pereyra(2016)]{Pereyra2016}
Marcelo Pereyra.
\newblock {Proximal Markov chain Monte Carlo algorithms}.
\newblock \emph{Statistics and Computing}, 26\penalty0 (4):\penalty0 745--760,
  jul 2016.
\newblock ISSN 0960-3174.
\newblock \doi{10.1007/s11222-015-9567-4}.
\newblock URL \url{http://dx.doi.org/10.1007/s11222-015-9567-4}.

\bibitem[Pesquet et~al.(2020)Pesquet, Repetti, Terris, and Wiaux]{Pesquet2020}
Jean-Christophe Pesquet, Audrey Repetti, Matthieu Terris, and Yves Wiaux.
\newblock {Learning Maximally Monotone Operators for Image Recovery}.
\newblock 2020.
\newblock URL \url{http://arxiv.org/abs/2012.13247}.

\bibitem[Pu et~al.(2017{\natexlab{a}})Pu, Wang, Henao, Chen, Gan, Li, and
  Carin]{Pu2017}
Yunchen Pu, Weiyao Wang, Ricardo Henao, Liqun Chen, Zhe Gan, Chunyuan Li, and
  Lawrence Carin.
\newblock {Adversarial symmetric variational autoencoder}.
\newblock In \emph{(NIPS) Advances in Neural Information Processing Systems},
  pages 4331--4340, 2017{\natexlab{a}}.

\bibitem[Pu et~al.(2017{\natexlab{b}})Pu, Wang, Henao, Chen, Gan, Li, and
  Carin]{Pu2017a}
Yunchen Pu, Weiyao Wang, Ricardo Henao, Liqun Chen, Zhe Gan, Chunyuan Li, and
  Lawrence Carin.
\newblock {Adversarial symmetric variational autoencoder}.
\newblock In \emph{(NIPS) Advances in Neural Information Processing Systems},
  volume 2017-Decem, pages 4331--4340, 2017{\natexlab{b}}.

\bibitem[Radford et~al.(2015)Radford, Metz, and
  Chintala]{radford2015unsupervised}
Alec Radford, Luke Metz, and Soumith Chintala.
\newblock Unsupervised representation learning with deep convolutional
  generative adversarial networks.
\newblock \emph{arXiv preprint arXiv:1511.06434}, 2015.

\bibitem[Rahaman et~al.(2019)Rahaman, Baratin, Arpit, Draxler, Lin, Hamprecht,
  Bengio, and Courville]{rahaman2019spectral}
Nasim Rahaman, Aristide Baratin, Devansh Arpit, Felix Draxler, Min Lin, Fred
  Hamprecht, Yoshua Bengio, and Aaron Courville.
\newblock On the spectral bias of neural networks.
\newblock In \emph{International Conference on Machine Learning}, pages
  5301--5310. PMLR, 2019.

\bibitem[Raj et~al.(2019)Raj, Li, and Bresler]{Raj2019}
Ankit Raj, Yuqi Li, and Yoram Bresler.
\newblock {GAN-Based Projector for Faster Recovery With Convergence Guarantees
  in Linear Inverse Problems}.
\newblock In \emph{(ICCV) International Conference on Computer Vision}, pages
  5601--5610. IEEE, oct 2019.
\newblock ISBN 978-1-7281-4803-8.
\newblock \doi{10.1109/ICCV.2019.00570}.

\bibitem[Reehorst and Schniter(2018)]{reehorst2018regularization}
Edward~T Reehorst and Philip Schniter.
\newblock Regularization by denoising: Clarifications and new interpretations.
\newblock \emph{IEEE Transactions on Computational Imaging}, 5\penalty0
  (1):\penalty0 52--67, 2018.
\newblock \doi{10.1109/TCI.2018.2880326}.

\bibitem[Romano et~al.(2017)Romano, Elad, and Milanfar]{romano2017little}
Yaniv Romano, Michael Elad, and Peyman Milanfar.
\newblock The little engine that could: Regularization by denoising (red).
\newblock \emph{SIAM Journal on Imaging Sciences}, 10\penalty0 (4):\penalty0
  1804--1844, 2017.

\bibitem[Rudin et~al.(1992)Rudin, Osher, and Fatemi]{Rudin1992}
Leonid~I. Rudin, Stanley Osher, and Emad Fatemi.
\newblock {Nonlinear total variation based noise removal algorithms}.
\newblock \emph{Physica D: Nonlinear Phenomena}, 60\penalty0 (1-4):\penalty0
  259--268, 1992.
\newblock ISSN 01672789.
\newblock \doi{10.1016/0167-2789(92)90242-F}.

\bibitem[Ryu et~al.(2019)Ryu, Liu, Wang, Chen, Wang, and Yin]{ryu2019plug}
Ernest~K. Ryu, Jialin Liu, Sicheng Wang, Xiaohan Chen, Zhangyang Wang, and
  Wotao Yin.
\newblock Plug-and-play methods provably converge with properly trained
  denoisers.
\newblock In \emph{Proceedings of the 36th International Conference on Machine
  Learning, {ICML} 2019, 9-15 June 2019, Long Beach, California, {USA}}, pages
  5546--5557, 2019.
\newblock URL \url{http://proceedings.mlr.press/v97/ryu19a.html}.

\bibitem[Sajjadi et~al.(2018)Sajjadi, Bachem, Lucic, Bousquet, and
  Gelly]{Sajjadi2018a}
Mehdi S.~M. Sajjadi, Olivier Bachem, Mario Lucic, Olivier Bousquet, and Sylvain
  Gelly.
\newblock {Assessing Generative Models via Precision and Recall}.
\newblock In \emph{(NeurIPS) Neural Information Processing Systems}, may 2018.

\bibitem[Schwartz et~al.(2018)Schwartz, Giryes, and
  Bronstein]{schwartz2018deepisp}
Eli Schwartz, Raja Giryes, and Alex~M Bronstein.
\newblock Deepisp: Toward learning an end-to-end image processing pipeline.
\newblock \emph{IEEE Transactions on Image Processing}, 28\penalty0
  (2):\penalty0 912--923, 2018.

\bibitem[Shah and Hegde(2018)]{shah2018solving}
Viraj Shah and Chinmay Hegde.
\newblock Solving linear inverse problems using gan priors: An algorithm with
  provable guarantees.
\newblock In \emph{2018 IEEE International Conference on Acoustics, Speech and
  Signal Processing (ICASSP)}, pages 4609--4613. IEEE, 2018.

\bibitem[Sreehari et~al.(2016)Sreehari, Venkatakrishnan, Wohlberg, Buzzard,
  Drummy, Simmons, and Bouman]{Sreehari2015}
Suhas Sreehari, Singanallur~V. Venkatakrishnan, Brendt Wohlberg, Gregery~T.
  Buzzard, Lawrence~F. Drummy, Jeffrey~P. Simmons, and Charles~A. Bouman.
\newblock {Plug-and-Play Priors for Bright Field Electron Tomography and Sparse
  Interpolation}.
\newblock \emph{IEEE Transactions on Computational Imaging}, 2\penalty0
  (4):\penalty0 1--1, 2016.
\newblock ISSN 2333-9403.
\newblock \doi{10.1109/TCI.2016.2599778}.

\bibitem[Teodoro et~al.(2018)Teodoro, Bioucas-Dias, and
  Figueiredo]{Teodoro2018scene}
Afonso~M. Teodoro, Jos{\'{e}}~M. Bioucas-Dias, and M{\'{a}}rio A.~T.
  Figueiredo.
\newblock {Scene-Adapted Plug-and-Play Algorithm with Guaranteed Convergence:
  Applications to Data Fusion in Imaging}, jan 2018.

\bibitem[Terris et~al.(2020)Terris, Repetti, Pesquet, and Wiaux]{Terris2020}
Matthieu Terris, Audrey Repetti, Jean~Christophe Pesquet, and Yves Wiaux.
\newblock {Building firmly nonexpansive convolutional neural networks}.
\newblock \emph{ICASSP, IEEE International Conference on Acoustics, Speech and
  Signal Processing - Proceedings}, 2020-May:\penalty0 8658--8662, 2020.
\newblock ISSN 15206149.
\newblock \doi{10.1109/ICASSP40776.2020.9054731}.

\bibitem[Tikhonov(1943)]{Tikhonov1943}
A.~N. Tikhonov.
\newblock {On the regularization of ill-posed problems}.
\newblock \emph{Dokl. Akad. Nauk SSSR}, 39\penalty0 (1):\penalty0 195--198,
  1943.

\bibitem[Tseng and Bertsekas(1993)]{Tseng1993}
Paul Tseng and Dimitri~P. Bertsekas.
\newblock {On the convergence of the exponential multiplier method for convex
  programming}.
\newblock \emph{Mathematical Programming}, 60\penalty0 (1-3):\penalty0 1--19,
  jun 1993.
\newblock ISSN 00255610.
\newblock \doi{10.1007/BF01580598}.
\newblock URL \url{https://link.springer.com/article/10.1007/BF01580598}.

\bibitem[Vahdat and Kautz(2020)]{vahdat2020nvae}
Arash Vahdat and Jan Kautz.
\newblock Nvae: A deep hierarchical variational autoencoder.
\newblock \emph{Advances in Neural Information Processing Systems}, 33, 2020.

\bibitem[Valkonen(2014)]{Valkonen2014}
Tuomo Valkonen.
\newblock {A primal-dual hybrid gradient method for nonlinear operators with
  applications to MRI}.
\newblock \emph{Inverse Problems}, 30\penalty0 (5):\penalty0 1--42, 2014.
\newblock ISSN 13616420.
\newblock \doi{10.1088/0266-5611/30/5/055012}.

\bibitem[Venkatakrishnan et~al.(2013)Venkatakrishnan, Bouman, and
  Wohlberg]{venkatakrishnan2013plug}
Singanallur~V. Venkatakrishnan, Charles~A. Bouman, and Brendt Wohlberg.
\newblock {Plug-and-Play priors for model based reconstruction}.
\newblock \emph{2013 IEEE Global Conference on Signal and Information
  Processing, GlobalSIP 2013 - Proceedings}, pages 945--948, 2013.
\newblock \doi{10.1109/GlobalSIP.2013.6737048}.

\bibitem[Whang et~al.(2020)Whang, Lei, and Dimakis]{Whang2020}
Jay Whang, Qi~Lei, and Alexandros~G. Dimakis.
\newblock {Compressed Sensing with Invertible Generative Models and Dependent
  Noise}.
\newblock In \emph{NeurIPS deep-inverse workshop}, 2020.

\bibitem[Xu et~al.(2020)Xu, Sun, Liu, Wohlberg, and Kamilov]{Xu2020}
Xiaojian Xu, Yu~Sun, Jiaming Liu, Brendt Wohlberg, and Ulugbek~S. Kamilov.
\newblock {Provable Convergence of Plug-and-Play Priors with MMSE denoisers}.
\newblock \penalty0 (4):\penalty0 1--10, 2020.
\newblock URL \url{http://arxiv.org/abs/2005.07685}.

\bibitem[Yu et~al.(2011)Yu, Sapiro, and Mallat]{yu2011solving}
Guoshen Yu, Guillermo Sapiro, and St{\'e}phane Mallat.
\newblock Solving inverse problems with piecewise linear estimators: From
  gaussian mixture models to structured sparsity.
\newblock \emph{IEEE Transactions on Image Processing}, 21\penalty0
  (5):\penalty0 2481--2499, 2011.

\bibitem[Zhang et~al.(2017{\natexlab{a}})Zhang, Zuo, Chen, Meng, and
  Zhang]{zhang2017beyond}
Kai Zhang, Wangmeng Zuo, Yunjin Chen, Deyu Meng, and Lei Zhang.
\newblock Beyond a gaussian denoiser: Residual learning of deep cnn for image
  denoising.
\newblock \emph{IEEE Transactions on Image Processing}, 26\penalty0
  (7):\penalty0 3142--3155, 2017{\natexlab{a}}.

\bibitem[Zhang et~al.(2017{\natexlab{b}})Zhang, Zuo, Gu, and Zhang]{Zhang2017}
Kai Zhang, Wangmeng Zuo, Shuhang Gu, and Lei Zhang.
\newblock {Learning Deep CNN Denoiser Prior for Image Restoration}.
\newblock In \emph{(CVPR) IEEE Conference on Computer Vision and Pattern
  Recognition}, pages 2808--2817. IEEE, apr 2017{\natexlab{b}}.
\newblock ISBN 978-1-5386-0457-1.
\newblock \doi{10.1109/CVPR.2017.300}.
\newblock URL
  \url{http://openaccess.thecvf.com/content_cvpr_2017/html/Zhang_Learning_Deep_CNN_CVPR_2017_paper.html}.

\bibitem[Zhang et~al.(2018{\natexlab{a}})Zhang, Zuo, and
  Zhang]{zhang2018ffdnet}
Kai Zhang, Wangmeng Zuo, and Lei Zhang.
\newblock Ffdnet: Toward a fast and flexible solution for cnn-based image
  denoising.
\newblock \emph{IEEE Transactions on Image Processing}, 27\penalty0
  (9):\penalty0 4608--4622, 2018{\natexlab{a}}.

\bibitem[Zhang et~al.(2018{\natexlab{b}})Zhang, Isola, Efros, Shechtman, and
  Wang]{zhang2018unreasonable}
Richard Zhang, Phillip Isola, Alexei~A Efros, Eli Shechtman, and Oliver Wang.
\newblock The unreasonable effectiveness of deep features as a perceptual
  metric.
\newblock In \emph{Proceedings of the IEEE conference on computer vision and
  pattern recognition}, pages 586--595, 2018{\natexlab{b}}.

\bibitem[Zhang et~al.(2019)Zhang, Zhang, Li, Bengio, and Paull]{PGA-Zhang2019}
Zijun Zhang, Ruixiang Zhang, Zongpeng Li, Yoshua Bengio, and Liam Paull.
\newblock {Perceptual Generative Autoencoders}.
\newblock In \emph{(ICLR) International Conference on Learning
  Representations}, pages 1--7, jun 2019.
\newblock URL \url{https://github.com/zj10/PGA}.

\bibitem[Zoran and Weiss(2011)]{Zoran2011}
Daniel Zoran and Yair Weiss.
\newblock {From learning models of natural image patches to whole image
  restoration}.
\newblock In \emph{2011 International Conference on Computer Vision}, pages
  479--486. IEEE, nov 2011.
\newblock ISBN 978-1-4577-1102-2.
\newblock \doi{10.1109/ICCV.2011.6126278}.
\newblock URL
  \url{http://people.csail.mit.edu/danielzoran/EPLLICCVCameraReady.pdf}.

\end{thebibliography}

\clearpage
\appendix
\section{Properties of $J_1$}\label{sec:J1prop}
In this section, we establish that the objective function $J_1$ fulfills the assumptions required to prove the convergence of Algorithm 3, namely 
\begin{itemize}
    \item $J_1(\cdot,\vz)$ is convex for any $\vz$;
    \item $J_1(\cdot,\vz)$ has a unique minimizer for any $\vz$;
    \item $J_1$ is coercive;
    \item $J_1$ is continuously differentiable;
\end{itemize}
We recall that
\begin{align*}
J_1(\vx,\vz) = & \underbrace{\frac{1}{2\,\sigma^2}\,\|\mA\,\vx-\vy\|^2}_{\displaystyle F(\vx,\vy)} \\
& +\underbrace{ \frac{1}{2}\Big (Z_\theta(\vz) 
     +\|\SigmaDecoder^{-1/2}(\vz)(\vx-\muDecoder(\vz))\|^2 \Big ) }_{\displaystyle H_{\theta}(\vx,\vz)} \\
& + \frac12\,\|\vz\|^2
\end{align*}
where
$$ Z_\theta(\vz) = \xdim\log(2\pi) + \log\det\SigmaDecoder(\vz) $$
     Thus, it is the sum of three non-negative terms.
\subsection{Convexity and unicity of the minimizer of $J_1(\cdot,z)$}
Let $\vz$ be fixed. Then there exists a constant $C\in\RR$ such that $\forall\,\vx$
\[%
J_1(\vx,\vz) = \frac{1}{2\,\sigma^2}\,\|\mA\,\vx-\vy\|^2  +\|\SigmaDecoder^{-1/2}(\vz)(\vx-\muDecoder(\vz))\|^2 + C\]

Being the sum of two quadratic forms, $J_1(\cdot,\vz)$ is obviously twice differentiable. Its gradient is given by
\begin{align*}
    \frac{\partial J_1}{\partial \vx}(\vx,\vz) = & \frac{1}{\sigma^2}\,\mA^T(\mA\,\vx-\vy) \\
    & + 2\,(\SigmaDecoder^{-1/2}(\vz))^T\big(\SigmaDecoder^{-1/2}(\vz)(\vx-\muDecoder(\vz))\big))
\end{align*}

and its Hessian is
\[
\text{Hess}_{\vx}J_1(\vx,\vz) = 
\frac{1}{\sigma^2}\,\mA^T\mA +2\,(\SigmaDecoder^{-1/2}(\vz))^T \SigmaDecoder^{-1/2}(\vz)
\]
Since $\SigmaDecoder(\vz)=\gamma^2 I$ the Hessian is positive definite (without the need to assume that $A$ is full rank), and we have that 
\begin{lemma}\label{x-convexity}
$J_1(\cdot,\vz)$ is strictly convex for any $\vz$.
\end{lemma}
An immediate consequence is the unicity of the minimizer of the partial function $J_1(\cdot,\vz)$.

\subsection{Coercivity of $J_1$}
\begin{lemma}\label{coercivity}
$J_1$ is coercive.
\end{lemma}
\begin{proof}
First, let us note that $J_1$ is the sum of three non-negative terms. 
If it was not coercive, then we could find a sequence $(\vx_{k},\vz_{k}) \to \infty$ such that $J_1(\vx_{k},\vz_{k})$ is bounded.
As a consequence all three terms are bounded. In particular the last term $\|\vz_{k}\|$ is bounded, which means that $\vx_{k} \to \infty$.
From Property~\ref{lipschitz-nn}, $\{\muDecoder(\vz_k)\}$ and $\{\SigmaDecoder(\vz_k)\}$ are bounded for bounded $\{\vz_k\}$.
Now, from the definition of the second term of $J_1$, we get that, $\{\muDecoder(\vz_k)\}$ and $\{\SigmaDecoder(\vz_k)\}$ being bounded and $\vx_{k}$ going to $\infty$ yield that $H_{\theta}(x_k,z_k)$ goes to infinity, while being bounded. This leads to a contradiction and thus proves that $J_1$ is coercive.

\end{proof}

\subsection{Regularity of $J_1$}
In the sequel we adopt the common assumption that all neural networks used in this work are composed of a finite number $d$ of layers, each layer being composed of: 
\emph{(a)} a linear operator (\emph{e.g.} convolutional or fully connected layer), followed by
\emph{(b)} a non-linear $L$-Lipschitz component-wise activation function with $0<L<\infty$.

Therefore we have the following property:

\begin{property}\label{lipschitz-nn}
For any neural network $f_\theta$ with parameters $\theta$ having the structure described above: \\
There exists a constant $C_\theta$ such that $\forall \vu$,
$$ \|f_\theta(\vu) \|_2 \leq C_\theta \| \vu \|_2.$$
\end{property}

Concerning activation functions we use two kinds:
\begin{itemize}
    \item continously differentiable activations like ELU, or
    \item continuous but non-differentiable activations like ReLU
\end{itemize}

Hence, by composition, we have that 
\begin{lemma}For continuously differentiable activation functions, $J_1$ is continuously differentiable.
\end{lemma}

\section{MAP-x and MAP-z for deterministic generative models}

Assume that the stochastic $\gamma$-generative model is 
$$ \ConditionalPDF{\X_{\gamma}}{\Z_{\gamma}}{\vx}{\vz} = \Normal(\generator(\vz),\gamma^2 I) $$
meaning that when $\gamma\to 0$
$$ \ConditionalPDF{\X}{\Z}{\vx}{\vz} = \delta(\vx-\generator(\vz)) $$

We now analyze the \MAP-\z\ and \MAP-\x\ estimators for the limit case when $\gamma=0$. This is what we call a deterministic generative model, and it includes GANs for instance.

\subsection{MAP-z}\label{sec:mapZ}

By definition the \MAP-\z\ estimator is obtained by maximising the posterior with respect to \z:

\begin{equation}
\begin{split}
\hat{\vz}_{\MAP-\vz}
& = \argmax_\vz \left \{\ConditionalPDF{\Z}{\Y}{\vz}{\vy} \right \}\\
& = \argmax_\vz \left \{ \ConditionalPDF{\Y}{\Z}{\vy}{\vz} \PDF{\Z}{\vz} \right \}.
\end{split}
\end{equation}
In the last line we used Bayes rule to rewrite this posterior in more simple terms.
However, this expression still involves the unknown conditional  $\ConditionalPDF{\Y}{\Z}{\vy}{\vz}$.\\
Let us  express this maximization in terms of $\ConditionalPDF{\Y}{\X}{\vy}{\vx}$.

To do so we recall the relation between the conditionals and the joint:

\begin{equation}\label{eq:mapZ-jointYZ}
\ConditionalPDF{\Y}{\Z}{\vy}{\vz} \PDF{\Z}{\vz} = \PDF{\Y,\Z}{\vy,\vz} = \ConditionalPDF{\Z}{\Y}{\vz}{\vy} \PDF{\Y}{\vy}
\end{equation}

We can also compute the joint distribution $ \PDF{\Y,\Z}{\vy,\vz} $ by marginalization on a third random variable $\X$:

\begin{equation}\label{eq:mapZ-jointYZ-conditionalYX}
\begin{split}
\PDF{\Y,\Z}{\vy,\vz}  & = \int  \PDF{\X,\Y,\Z}{\vx,\vy,\vz} d\vx \\
        & = \int \ConditionalPDF{\Y}{\X,\Z}{\vy}{\vx,\vz} \ConditionalPDF{\X}{\Z}{\vx}{\vz} \PDF{\Z}{\vz} d\vx \\
        & = \int \ConditionalPDF{\Y}{\X}{\vy}{\vx}  \delta(\vx-\generator(\vz))  \PDF{\Z}{\vz} d\vx \\
        & = \ConditionalPDF{\Y}{\X}{\vy}{\generator(\vz)}  \PDF{\Z}{\vz}
\end{split}
\end{equation}

The third line follows from our graphical model $\Z \rightarrow \X \rightarrow \Y$
which implies that once we know $\X=\vx$, then $\Z$ provides no additional information, therefore
$$\ConditionalPDF{\Y}{\X,\Z}{\vy}{\vx,\vz} =  \ConditionalPDF{\Y}{\X}{\vy}{\vx}.$$

The last line follows simply from the integration on $\vx$ of a delta function.

From equations~\eqref{eq:mapZ-jointYZ} and \eqref{eq:mapZ-jointYZ-conditionalYX} we can derive an expression of $ \ConditionalPDF{\Z}{\Y}{\vz}{\vy} $ in terms of  $\ConditionalPDF{\Y}{\X}{\cdot}{\cdot}$ and the generator $\generator$ namely:

$$  \ConditionalPDF{\Z}{\Y}{\vz}{\vy} = \frac{1}{\PDF{\Y}{\y}}  \ConditionalPDF{\Y}{\X}{\vy}{\generator(\vz)}  \PDF{\Z}{\vz} $$

This proves the main result of this section:

\begin{proposition}[\MAP-\z\  estimator for deterministic generative models]\label{thm:mapZ}
Assume we have
\begin{itemize}
\item a deterministic generative model where $X = \generator(\Z)$ and %
\item an inverse problem characterised by the log conditional distribution $ \log \ConditionalPDF{\Y}{\X}{\vy}{\vx} = - \Fdata(\vx,\vy)$.
\end{itemize}

Then the \MAP-\z\  estimator is computed as $\hat{\vx}_{\MAP-\z} = \generator(\hat{\vz}_{\MAP-\z})$ where

\begin{equation}
\begin{split}
\hat{\vz}_{\MAP-\z}
& = \argmax_\vz \left \{\ConditionalPDF{\Y}{\X}{\vy}{\generator(\vz)}\PDF{\Z}{\vz} \right \} \\
& = \argmin_\vz \left \{\Fdata(\generator(\vz),\vy) - \log \PDF{\Z}{\vz} \right \}.
\end{split}
\end{equation}
\end{proposition}

\subsection{MAP-x} \label{sec:mapX}

The \MAP-\x\ estimator is obtained by maximizing the posterior with respect to \x.
The generative model induces a prior on $X$ via the push-forward measure $p_\X = \generator \sharp p_\Z$, which following \citep[section 5]{Papamakarios2019} can be developed as
$$
\PDF{\X}{\x} = 
\frac{\PDF{\Z}{\generator^{-1}(\x)}}%
{\sqrt{\det S(\generator^{-1}(\x))}}
\delta_\mathcal{M}(\x)
$$
where 
$S = \left(\frac{\partial \generator}{\partial \vz}\right)^T\left(\frac{\partial \generator}{\partial \vz}\right)$ is the squared Jacobian and 
the manifold $\mathcal{M} = \lbrace \vx \,:\, \exists \vz,\, \vx = \generator(\vz)\rbrace$ represents the image of the generator \generator.

With such a prior $p_\X$, the \x-optimization \eqref{eq:MAP} required to obtain $\hat{\x}_\MAP$ becomes intractable (in general), for various reasons:
\begin{itemize}
    \item the computation of $\det S$,
    \item the inversion of $\generator$, and
    \item the hard constraint $\x\in\mathcal{M}$.
\end{itemize}
These operations are are all memory and/or computationally intensive, except when they are partially addressed by the use of a normalizing flow like in \citep{Helminger2020,Whang2020}.\\

\section{Joint MAP-x-z, Continuation Scheme and convergence to MAP-z}\label{sec:continuation}

The functional $J_{1,\beta}$ introduced in Equation~\ref{eq:MAPz-splitting} can be seen from two different perspectives.

From a machine learning perspective it corresponds to the joint log-posterior $J_1$ in the case where $\SigmaDecoder(\vz) = \frac{1}{\beta} I$ and $\muDecoder=\generator$, namely:

\begin{align*}
J_{1,\beta}(\vx,\vz) = & 
\underbrace{\frac{1}{2\,\sigma^2}\,\|\mA\,\vx-\vy\|^2}_{%
    \displaystyle F(\vx,\vy)} \\
& + \underbrace{ \frac{\beta}{2}%
        \|\vx-\muDecoder(\vz))\|^2
     }_{%
     \displaystyle H_{\theta}(\vx,\vz) = \varphi_\beta(\x,\z)} \\
& + \frac12\,\|\vz\|^2 + C_\beta
\end{align*}

From an optimization standpoint it can be considered as an inexact penalisation procedure: We want to solve the constrained problem
\[\min_{(\vx,\vz)\in \mathcal C} \underbrace{F(\vx,\vy) + \frac12\,\|\vz\|^2}_{=J_{1,0}(\vx,\vz)}\]
with $\mathcal C = \{(\vx,\vz)\mid \vx=\muDecoder(\vz)\}$
whose solution provides the \MAP-\z\ estimator
\begin{equation}\label{eq:MAPz_const}
    (\x^*,\z^*) \in \argmin_{(\vx,\vz)\in \mathcal C} J_{1,0}(\vx,\vz).
\end{equation}

To do so, we introduced the family of unconstrained problems
\[\min_{\vx,\vz} J_{1,\beta}(\x,\z)\]
and their corresponding minimizers 
$$ (\hat{\x}_\beta, \hat{\z}_\beta) \in \argmin_{\x,\z} J_{1,\beta}(\x,\z) $$
which for $\beta=\frac{1}{\gamma^2}$ provide the \MAP-\x-\z\ estimator.\\

We can show that the \MAP-\x-\z\ estimator converges to the \MAP-\z\ estimator when $\beta\to\infty$ (or equivalently $\gamma\to 0$).

\begin{proposition}
The unconstrained functional tends to the constrained functional plus the constraint:
\begin{equation}
J_{1,\beta}(\vx,\vz) \xrightarrow{\beta \to \infty} J_{1,\infty}(\vx,\vz) = %
F(\vx,\vz)
+ \iota_{\vx = \muDecoder(\vz)}(\vx,\vz)
+ \frac12\,\|\vz\|^2
\end{equation}

and the unconstrained minimizers tend to the constrained minimizer as $\beta\to\infty$:

\begin{equation}
 \lim_{\beta\to\infty} (\hat{\x}_\beta, \hat{\z}_\beta) \in \argmin_{(\vx,\vz)\in \mathcal C} J_{1,0}(\vx,\vz) = \argmin_{\vx,\vz} J_{1,\infty}(\vx,\vz).
\end{equation}

\end{proposition}

\begin{proof}
The pointwise convergence of $\varphi_{\beta}$ to $\iota_{\vx = \generator(\vz)}$ as $\beta$ goes to $\infty$ is straightforward.

Let us first prove that for any sequence $(\beta_n)_n$ that goes to $\infty$, the quantity $\|\hat\vx_{\beta_n}-\generator(\hat\vz_{\beta_n})\|$ goes to zero. Otherwise, for any $\varepsilon>0$, there exists a subsequence $(\beta_{n_j})_j$ such that $\|\hat\vx_{\beta_n}-\generator(\hat\vz_{\beta_n})\|>\varepsilon$. In this case, for any $\vz$, one has by optimality
\[J_{1,0}(\generator(\vz),\vz) = J_{1,\beta_{n_j}}(\generator(\vz),\vz)
\geq J_{1,\beta_{n_j}}(\hat\vx_{\beta_{n_j}},\hat\vz_{\beta_{n_j}})
> J_{1,0}(\hat\vx_{\beta_{n_j}},\hat\vz_{\beta_{n_j}})+\frac{\beta_{n_j}}{2}\,\varepsilon^2\]
As a result, the nonnegative quantity $J_{1,0}(\hat\vx_{\beta_{n_j}},\hat\vz_{\beta_{n_j}})$ goes to $-\infty$, which leads to a contradiction. Thus, one has $\hat\vx_{\infty}=\generator(\hat\vz_{\infty})$ for any limit point $(\hat\vx_{\infty},\hat\vz_{\infty})$ of $(\hat\vx_{\beta_{n}},\hat\vz_{\beta_{n}})$.
Assume that $J_{1,0}(\hat\vx_{\infty},\hat\vz_{\infty}) > J_{1,0}(\vx^*,\vz^*)$. Since
\[
J_{1,\beta_{n}}(\hat\vx_{\beta_{n}},\hat\vz_{\beta_{n}})
\leq J_{1,\beta_{n}}(\vx^*,\vz^*)
=J_{1,0}(\vx^*,\vz^*)
<J_{1,0}(\hat\vx_{\infty},\hat\vz_{\infty})
\] this leads to another contradiction.

\end{proof}

The previous result motivates Algorithm \ref{alg:MAPz-splitting}.

Consider Algorithm \ref{alg:MAPz-splitting} in the ideal case (maxiter=$\infty$) where the internal loop converges.

\begin{proposition}[Convergence of Algorithm \ref{alg:MAPz-splitting}]
Let $(\vx^k_{\infty},\vz^k_{\infty})_k$ be a sequence generated by Algorithm \ref{alg:MAPz-splitting} when maxiter=$\infty$. If $(\z^k_{\infty})_k$ is bounded, then any limit point of $(\vx^k_{\infty},\vz^k_{\infty})_k$ is in $\mathcal C$. Moreover, any limit point of $(\vz^k_{\infty})_k$ is a stationary point of
\begin{equation}\label{eq:MAPz_func}
     f(\vz) = F(\generator(\vz),\vy) + \frac12\,\|\vz\|^2
\end{equation}
\end{proposition}

\begin{proof}
Note that, for any $k$, $(\vx^k_{\infty},\vz^k_{\infty})$ is a limit point of the sequence generated by the $k$-th subloop in Algorithm \ref{alg:MAPz-splitting} if it does not converge.
Let $(\beta_k)_k$ a sequence that converges to $\infty$. Let $k\in\NN$.
We consider the sequence $(\vx^k_n,\vz^k_n)_n$ generated by
\[
\forall\,n\in\NN,\qquad
\vz^k_{n+1} \in\arg\min_{\vz} J_{1,\beta_k}(\vx^k_n,\vz)
\text{ and }
\vx^k_{n+1} =\arg\min_{\vx} J_{1,\beta_k}(\vx,\vz^k_{n+1})
\]
with $\vx^k_0=\vx^{k-1}_{\infty}$.
Since $J_{1,\beta_k}$ corresponds to a particular instance of $J_1$, and since Algorithm \ref{alg:MAPz-splitting} can be seen asymptotically as a particular instance of Algorithm \ref{alg:JPMAP3new}, one can use all the results established in Proposition \ref{thm:convergence-approx}.
In particular,  the sequence $(\vx^k_n,\vz^k_n)_n$ admits a limit point $(\vx^k_{\infty},\vz^k_{\infty})$ and we have
\[
\frac{\partial J_{1,\beta_k}}{\partial\vx}(\vx^k_{\infty},\vz^k_{\infty})=\frac{\partial J_{1,0}}{\partial\vx}(\vx^k_{\infty},\vz^k_{\infty})+\beta_k(\vx^k_{\infty}-\generator(\vz^k_{\infty}))=0
\]and\[
\frac{\partial J_{1,\beta_k}}{\partial\vz}(\vx^k_{\infty},\vz^k_{\infty})
=\frac{\partial J_{1,0}}{\partial\vz}(\vx^k_{\infty},\vz^k_{\infty})+\beta_k(D\generator(\vz^k_{\infty}))^*(\generator(\vz^k_{\infty})-\vx^k_{\infty})=0
\]
By convexity, $\vx^k_{\infty}$ is the (unique) minimizer of $J_{1,\beta_k}(\cdot,\vz^k_{\infty})$.

\textbf{Assume that the sequence $(\vz^k_{\infty})_k$ is bounded.} By optimality, one has
\[\min J_{1,0}\leq 
J_{1,0}(\vx^k_{\infty},\vz^k_{\infty})
\leq J_{1,\beta_k}(\vx^k_{\infty},\vz^k_{\infty})
\leq J_{1,\beta_k}(\generator(\vz^k_{\infty}),\vz^k_{\infty}) = J_{1,0}(\generator(\vz^k_{\infty}),\vz^k_{\infty})
\]
Since $(J_{1,0}(\generator(\vz^k_{\infty}),\vz^k_{\infty}))_k$ is bounded, so is $(J_{1,0}(\vx^k_{\infty},\vz^k_{\infty}))_k$. By coercivity, the sequence $(\vx^k_{\infty},\vz^k_{\infty})_k$ is also bounded.
Then it admits a limit point denoted $(\hat{\vx}_{\infty},\hat{\vz}_{\infty})$. Let $(\vx^{k_j}_{\infty},\vz^{k_j}_{\infty})_j$ be a convergent subsequence of limit $(\hat{\vx}_{\infty},\hat{\vz}_{\infty})$. Let us assume that $\hat{\vx}_{\infty}\neq \generator(\hat{\vz}_{\infty})$. Then, there exists $a>0$ and $j_0\in\NN$ such that
\[
\forall\,j\geq j_0,\qquad
\|\vx^{k_j}_{\infty}-\generator(\vz^{k_j}_{\infty})\|^2>a
\]
Hence, one has
\[
J_{1,\beta_{k_j}}(\vx^{k_j}_{\infty},\vz^{k_j}_{\infty})
\geq J_{1,0}(\vx^{k_j}_{\infty},\vz^{k_j}_{\infty})+\beta_{k_j}\,a 
\geq \min J_{1,0}+\beta_{k_j}\,a \underset{j\to+\infty}{\longrightarrow}\infty
\]
which leads to a contradiction.
This proves that $\hat{\vx}_{\infty} = \generator(\hat{\vz}_{\infty})$. Otherwise said, $(\vx^k_{\infty}-\generator(\vz^k_{\infty}))_k$ goes to zero.

Since we have for any $k$
\[
\frac{\partial J_{1,0}}{\partial\vx}(\vx^k_{\infty},\vz^k_{\infty})+\beta_k(\vx^k_{\infty}-\generator(\vz^k_{\infty}))=0
\]
the continuity of $\frac{\partial J_{1,0}}{\partial\vx}$ ensures that
$\left(\frac{\partial J_{1,0}}{\partial\vx}(\vx^{k_j}_{\infty},\vz^{k_j}_{\infty})\right)_j$ converges; thus, so is $(\beta_{k_j}(\vx^{k_j}_{\infty}-\generator(\vz^{k_j}_{\infty})))_j$.
Then there exists $\lambda^*\in \R^\xdim$ such that
\[
\frac{\partial J_{1,0}}{\partial\vx}(\vx^{k_j}_{\infty},\vz^{k_j}_{\infty})=-\beta_{k_j}(\vx^{k_j}_{\infty}-\generator(\vz^{k_j}_{\infty}))\underset{j\to+\infty}{\longrightarrow} \lambda^*
= \frac{\partial J_{1,0}}{\partial\vx}(\hat \vx_{\infty},\hat \vz_{\infty})
\]
and
\[
\frac{\partial J_{1,0}}{\partial\vz}(\vx^{k_j}_{\infty},\vz^{k_j}_{\infty})=-\beta_{k_j}(D\generator(\vz^{k_j}_{\infty}))^*(\generator(\vz^{k_j}_{\infty})-\vx^{k_j}_{\infty})\underset{j\to+\infty}{\longrightarrow} -(D\generator(\hat \vz_{\infty}))^*(\lambda^*)
\]
Note that $f(\vz)=J_{1,0}(\generator(\vz),\vz)$. One can check that $f$ is differentiable and that
\[
\nabla f(\vz) = (D\generator(\vz))^*\left(\frac{\partial J_{1,0}}{\partial \vx}(\generator(\vz),\vz)\right)
+\frac{\partial J_{1,0}}{\partial \vz}(\generator(\vz),\vz)
\]
Hence,
we have proved that \[
\nabla f(\hat\vz_{\infty})= 0
\]
Conclusion: If $(\z^k_{\infty})_k$ is bounded, any limit point of $(\z^k_{\infty})_k$ is a stationary point of (\ref{eq:MAPz_func}).
\end{proof}

In general, we can only prove that the limit points of the sequences generated by Algorithm \ref{alg:MAPz-splitting} are stationary points of \ref{eq:MAPz_func}. However, if the growth of $\beta$ is sufficiently slow, then we obtain the optimality of the limit points.
Indeed, given that, in Algorithm \ref{alg:MAPz-splitting}, each subloop is an exact BCD scheme, one has for any $n$ and any $j$
\[\forall\,\vz,\qquad J_{1,\beta_{k_j}}(\vx_{n-1},\vz_{n})
\leq J_{1,\beta_{k_j}}(\vx_{n-1},\vz)\]
By considering the subsequence $(\vx_{n_{\ell}},\vz_{n_{\ell}})$, which converges to $\vx^{k_j}_{\infty},\vz^{k_j}_{\infty}$ (we recall that $\vx_{n_{\ell}}$ and $\vx_{n_{\ell}-1}$ have same limit), we can prove that
\[\forall\,\vz,\qquad J_{1,\beta_{k_j}}(\vx^{k_j}_{\infty},\vz^{k_j}_{\infty})
\leq J_{1,\beta_{k_j}}(\vx^{k_j}_{\infty},\vz)\]
that is, $\vz^{k_j}_{\infty}$ is a minimizer of $J_{1,\beta_{k_j}}(\vx^{k_j}_{\infty},\cdot)$. Hence, we have
\[\forall\,\vz,\qquad J_{1,0}(\vx^{k_j}_{\infty},\vz^{k_j}_{\infty})+\frac{\beta_{k_j}}{2}\,\|\vx^{k_j}_{\infty}-\generator(\vz^{k_j}_{\infty})\|^2
\leq J_{1,0}(\vx^{k_j}_{\infty},\vz)+\frac{\beta_{k_j}}{2}\,\|\vx^{k_j}_{\infty}-\generator(\vz)\|^2\]
\textbf{Assume that $\beta_{k_j}\|\vx^{k_j}_{\infty}-\hat\vx_{\infty}\|^2\underset{j\to+\infty}{\longrightarrow}0$.}
By letting $j$ to $\infty$, we get that, for any $\vz$ such that $\generator(\vz) = \vx^{k_j}_{\infty}$,
\[
J_{1,0}(\hat \vx_{\infty},\hat \vz_{\infty})
\leq J_{1,0}(\hat \vx_{\infty},\vz)+\lim_{j\to+\infty}\frac{\beta_{k_j}}{2}\,\|\vx^{k_j}_{\infty}-\vx^{k_j}_{\infty}\|^2 = J_{1,0}(\hat \vx_{\infty},\vz)\]
that is, $\hat\vz_{\infty}$ is a minimizer of $J_{1,0}(\hat\vx_{\infty},\cdot)+\chi_{\mathcal C}(\hat\vx_{\infty},\cdot)$. By definition of $f$, this also means that  $\hat\vz_{\infty}$ is a minimizer of $f$.
However, one has to note that the growth control for $\beta$ depends on the convergence speed of $\vx_{\infty}^k$, which cannot be estimated.

Algorithm \ref{alg:MAPzCS} is a particular (truncated) case of Algorithm \ref{alg:MAPz-splitting} with an adaptive choice of $\beta$ that does not need to go to $\infty$.

\begin{proposition}[Convergence of Algorithm \ref{alg:MAPzCS}]
\end{proposition}

\begin{proof}
Let us write the Lagrangian of the problem solved in Algorithm \ref{alg:MAPzCS}:
\[\forall\,\lambda\geq 0,\qquad \mathcal L(\vx,\vz;\lambda) = J_{1,0}(\vx,\vz) + \lambda\,(\|\vx-\generator(\vz)\|^2-\varepsilon)\]
KKT conditions ensure that any solution $(\vx^*,\vz^*)$ of the constrained problem is associated to at least one Lagrange multiplier $\lambda^*\geq 0$ such that
\[
\frac{\partial\mathcal L}{\partial(\vx,\vz)}(\vx^*,\vz^*;\lambda^*) = 0 = \begin{pmatrix}
\frac{\partial J_{1,0}}{\partial \vx}(\vx^*,\vz^*) + 2\lambda^*\,(\vx^*-\generator(\vz^*))\\
\frac{\partial J_{1,0}}{\partial \vz}(\vx^*,\vz^*) + 2\lambda^*\,(D\generator(\vz^*))^*(\vx^*-\generator(\vz^*))
\end{pmatrix}
\]
According to the calculus above, this proves that $(\vx^*,\vz^*)$ is a stationary point of $J_{1,2\lambda^*}$. Note that, if $\lambda^*=0$, then $(\vx^*,\vz^*)$ is a minimizer of $J_{1,0}$. Otherwise, one has $\|\vx^* - \generator(\vz^*)\|^2=\varepsilon$.\\

Hence, if we consider Algorithm \ref{alg:MAPz-splitting} with the update rule for $\beta_k$ as in Algorithm \ref{alg:MAPzCS} and a stopping rule saying that the iterations stop as soon as, for any given $k$, 
\[
\|\vx^k_{\infty}-\generator(\vz^k_{\infty})\|^2\leq \varepsilon
\]
there are two possible cases:
\begin{enumerate}
    \item \textbf{case $\lambda^*=0$:} then $(\vx^k_{\infty},\vz^k_{\infty})$ is a solution of the constraint problem iff $\nabla J_{1,0}(\vx^k_{\infty},\vz^k_{\infty})=0$ (that is, $\vx^k_{\infty}=\generator(\vz^k_{\infty}))$;
    \item \textbf{case $\lambda^*>0$:} unless $\|\vx^k_{\infty}-\generator(\vz^k_{\infty})\|^2$ exactly equals $ \varepsilon$, $(\vx^k_{\infty},\vz^k_{\infty})$ is \textbf{not} a solution of the constraint problem

\end{enumerate}
However, in general, $(\vx^k_{\infty},\vz^k_{\infty})$ is a solution of the following constraint problem
\[\min_{\|\vx^k_{\infty}-\generator(\vz^k_{\infty})\|^2\leq\tilde\varepsilon} J_{1,0}(\vx,\vz)\]
with $\tilde\varepsilon=\|\vx^k_{\infty}-\generator(\vz^k_{\infty})\|^2\leq \varepsilon$.
Hence, if we stop the iterations when $\|\vx^k_{\infty}-\generator(\vz^k_{\infty})\|^2\leq \varepsilon$, we will get a solution of
\[\min_{\|\vx^k_{\infty}-\generator(\vz^k_{\infty})\|^2\leq\tilde\varepsilon} J_{1,0}(\vx,\vz),\qquad \tilde\varepsilon\leq  \varepsilon\]
which provides an error control as well.
\end{proof}

\end{document}